\documentclass[12pt]{article}
\usepackage[T1]{fontenc}
\usepackage[utf8]{inputenc}
\usepackage[english]{babel}
\usepackage{url}
\usepackage{bm}
\usepackage{amsmath}
\usepackage{amssymb}
\usepackage{amsthm}
\usepackage{mathtools}
\usepackage{amsmath}
\usepackage{amssymb}
\usepackage{stmaryrd}
\usepackage{babel}
\usepackage{comment,xspace}
\usepackage{tablefootnote,subfig,subfloat}
\usepackage{booktabs,multirow,float}
\usepackage{graphicx}
\usepackage{diagbox}
\usepackage{natbib}

\usepackage{wasysym}
\usepackage{adjustbox}
\usepackage[pagewise]{lineno}\linenumbers
\modulolinenumbers[5]
\nolinenumbers
\usepackage{soul,amsthm}

\providecommand{\tabularnewline}{\\}

\providecommand{\theoremname}{Theorem}
\theoremstyle{plain}
\newtheorem{thm}{\protect\theoremname}
\providecommand{\lemmaname}{Lemma}
\theoremstyle{plain}
\newtheorem{lem}[thm]{\protect\lemmaname}
\theoremstyle{plain}

\usepackage{enumitem}
\usepackage[section]{placeins}
\usepackage{breakcites}
\usepackage{booktabs}
\usepackage{amsthm}

\newtheorem{prop}{Proposition}
\newtheorem{defn*}{Definition}
\usepackage{thmtools}
\usepackage{colortbl}

\usepackage{url}
\usepackage{color}
\usepackage[ruled,linesnumbered,resetcount]{algorithm2e}
\usepackage{algcompatible}
\newcounter{algoline}
\usepackage{multibib}

\newcites{SM}{SM References}

\newcommand{\blind}{1}

\usepackage{geometry}
\usepackage[unicode=true,
 bookmarks=false,
 breaklinks=true,pdfborder={0 0 1},backref=false,colorlinks=false]
 {hyperref}

\providecommand{\tabularnewline}{\\}

\usepackage{xcolor}

\newlength\myindent
\providecommand{\theoremname}{Theorem}

\usepackage{setspace}
\begin{document}

\if1\blind
{
  \title{\bf Efficient Decision Trees for Tensor Regressions} 
  \author{Hengrui Luo\\
Department of Statistics, Rice University;\\ Computational Research Division, Lawrence Berkeley National Laboratory\\
and \\
    Akira Horiguchi \\
    Department of Statistics, University of California Davis\\
and \\
    Li Ma\\
    Department of Statistics, University of Chicago}
\date{}
\maketitle
} 
\fi

\if0\blind
{
  \bigskip
  \bigskip
  \bigskip
  \begin{center}
    {\LARGE\bf Efficient Decision Trees for Tensor Regressions}
\end{center}
  \medskip
} \fi

\begin{abstract}
 We proposed the tensor-input tree (TT) method for  scalar-on-tensor and tensor-on-tensor regression
problems. %
We first
address scalar-on-tensor problem by proposing scalar-output regression
tree models whose input variables are tensors (i.e., multi-way arrays).
We devised and implemented fast randomized and deterministic algorithms for efficient fitting of scalar-on-tensor trees, making TT competitive against tensor-input GP models \citep{sun2023tensor,yu2018tensor}. %
Based on scalar-on-tensor tree models,
we extend our method to tensor-on-tensor problems using additive tree
ensemble approaches. Theoretical justification and extensive experiments, {including testing robustness to entrywise input tensor noise}, are provided on real and synthetic datasets to illustrate the performance of TT. 

Our implementation is provided at \url{http://www.github.com/hrluo}. 
\end{abstract}
\textit{Keywords: }Decision tree regressions, scalar-on-tensor regressions, tensor-on-tensor regressions, ensemble methods.  
\section{Introduction}

\subsection{Problem and Settings}

In recent years, the intersection of tensor data analysis and non-parametric
modeling methods \citep{guhaniyogi2017bayesian,papadogeorgou2021soft,wang2024bayesian}
has garnered considerable interest among mathematicians and statisticians.
Non-parametric tensor models have the potential to handle complex
multi-dimensional data \citep{bi2021tensors} and represent spatial
correlation between entries of data. This paper addresses both scalar-on-tensor
(i.e., to predict a scalar response based on a tensor input) and tensor-on-tensor 
(i.e., both the input and output are tensors)
non-linear regression problems using
recursive partitioning methods, often referred to as tree(-based)
models.

Supervised learning on tensor data, such as tensor regression, has
significant relevance due to the proliferation of multi-dimensional
data in modern applications. Tensor data naturally arises in various
fields such as imaging \citep{wang2024bayesian}, neuroscience \citep{li2018tucker},
and computer vision \citep{LLM2023}, where observations often take
the form of multi-way arrays. Traditional regression models typically
handle vector inputs and outputs, and thus can fail to
capture the structural information embedded within tensor data.

Tree-based methods \citep{breiman1984classification,Breiman01,friedman2004discussion,hastie2009elements}, on the other
hand, offer a flexible and interpretable framework for regression.
They can capture non-linear relationships and interactions
between features, making them particularly well-suited for the intricate
nature of tensor data. The ability to apply tree models directly to
tensor inputs provides a powerful tool for researchers and practitioners
dealing with high-dimensional and multi-way data.

%
Existing tree regression methods like CART, random forest, and boosting do not handle tensor data's multi-array characteristics. We develop regression tree models for tensor inputs and outputs to enable non-linear and possibly non-parametric modeling that captures complex tensor interactions. We introduce scalar-on-tensor regression with new scalar-output tree algorithms featuring strategies and loss functions designed for tensor data, integrating low-rank tensor approximation methods.

We then address the tensor-on-tensor problem by constructing an additive
tree ensemble model analogous to that for tree boosting in regression~\citep{chipman1998bayesian,denison1998bayesian,Friedman01,friedman2004discussion}. Our tensor tree boosting uses our scalar-on-tensor tree models as ``weak'' learners
to achieve competitive predictive performance. We show that the tree
ensemble approach is particularly effective for complex outcome spaces
such as in tensor-response regression.

In addition to developing methodology, we also address the
computational scalability and basic theoretical guarantees of our algorithms. This is particularly
important in the ensemble approach, as a large number of single trees
are trained. 
Our algorithms are competitive with existing tensor regression approaches such as tensor
Gaussian processes \citep{yu2018tensor,sun2023tensor}, and alternative non-parametric models. %
Non-parametric tree-based models provide appropriate data partitions in conjunction with ensemble methods to extend these models
into the domain of tensor inputs and outputs, integrating methodological ideas and theories from
tensor decompositions and low-rank approximations in parametric tensor
regression models \citep{zhou2013tensor,li2017parsimonious,li2018tucker}. Our proposed non-parametric regression method of \emph{tensor-input trees} (TT) allows us to capture multi-way dependence expressed in tensor covariates, and presents a scalable model for potentially heterogeneous tensors. In contrast to smoothing methods, TT is particularly suitable when there are non-smooth or change-of-pattern behavior in the tensor. 

In connection with generalized tree-based models with inputs in Fr\'{e}chet spaces \citep{capitaine2024frechet} (a.k.a., Fr\'{e}chet trees in \citet{qiu2024random}), the tensor space can be viewed as a Fr\'{e}chet space with additional structure that admits unique moments \citep{LLM2023} and low-rank decompositions \citep{kolda2009tensor}. This allows us to introduce novel tensor-specific splitting rules (i.e., LAE and LRE in Sec \ref{subsec:Splitting-Criterion}), whereas existing works \citep{capitaine2024frechet,qiu2024random} only consider a variance-based splitting criteria using clustering following the spirit of honest forest \citep{athey2019generalized}. In light of this, \citet{krawczyk2021tensor} developed a clustering-based splitting tensor-input tree model. To reduce the computational burden introduced by solving Fr\'{e}chet moments in splitting, \citet{bulte2024medoid} propose to solve the optimization among a smaller fraction of solution space called Fr\'{e}chet Medoid, which is very similar to our subsampling method, but different from our branch-and-bound strategy. In addition to single tensor-input trees, we also leverage boosting ensembles \citep{Friedman01} in order to reduce bias in tensor-output cases; in contrast, a random forest approach leaves the bias unchanged and instead aims to reduce variance. 

\subsection{Regression Trees Revisited}

We start by reviewing a scalar-on-vector regression problem. Consider
the regression setup with $n$ data pairs $(X_{i},y_{i})$ where the input
and response variables are $X_{i}\in\mathbb{R}^{d},y_{i}\in\mathbb{R}$
\begin{align}
y_{i}=f(X_{i})+\epsilon_{i},i=1,\ldots,n\label{eq:vector-input regression}
\end{align}
where $f\colon \mathbb{R}^{d}\rightarrow\mathbb{R}$ is a real-valued function
and the $\epsilon_{i}$s are independent mean zero noises.

A single regression tree assumes a vector input in $\mathbb{R}^{d}$
and is built by recursively partitioning \citep{breiman1984classification,hastie2009elements}
the input space into disjoint regions $R_{1},R_{2},\ldots,R_{J}\subset\mathbb{R}^{d},R_{j}\cap R_{k}=\emptyset$
when $j\neq k$, before fitting a regression
mean model $m_{j}$ in each region. 
The regression tree model can be written as 
\begin{equation}
y_{i}=g(X_{i})+\epsilon_{i}, \qquad
g(X;\mathcal{T},\mathcal{M})=\sum_{j=1}^{J}m_{j}(X)\cdot I(X\in R_{j}),\label{eq:single_tree}
\end{equation}
where $I(X\in R_{j})$ is an indicator function, $\mathcal{T}$
is the tree structure that dictates the partition of the input space, and $\mathcal{M}$ determines the tree model's prediction values. 
The partition is created by minimizing a splitting
criterion such as the sum of squared residuals (or sum of variances) 
of responses within each region: 
\begin{equation}
\text{SSE}(\mathcal{T};y_{1},\cdots,y_{n})=\sum_{j=1}^{J}\frac{1}{N_j}\sum_{X_{i}\in R_{j}}(y_{i}-\hat{y}_{R_{j}})^{2}.\label{eq:SSE}
\end{equation}
Typically $J=2$ so that each tree split creates left and right children nodes which correspond to two (sub)regions defined by bisecting a given region along the $j_1$th axis at the observed $\bm{X}[j_0,j_{1}]$, i.e., at the $(j_0,j_{1})$-th entry in the design matrix $\bm{X}\in\mathbb{R}^{n\times d}$ created by stacking the vectors $X_{1},\cdots,X_{n}$ as rows.
We adopt a top-down greedy algorithm using criterion \eqref{eq:SSE} (see
Algorithm~\ref{alg:Vector-input-decision}) to create the partition. 
%
%
%
%
%
%
%
The model then assigns a prediction value to each region; 
this value is often the sample mean of the responses of the training data in the region, i.e., $m_{j}(X)=\hat{y}_{R_{j}}$
where $\hat{y}_{R_{j}}=\frac{1}{N_{j}}\sum_{i=1}^{n}y_{i}\cdot I(X_{i}\in R_{j})$ and $N_{j}=\sum_{i=1}^{n}I(X_{i}\in R_{j})$.


While there exists Bayesian regression tree methods \citep{chipman1998bayesian,denison1998bayesian,chipman2010bart,Chipman12}, this paper will focus
on non-Bayesian approaches of ensemble construction, but 
generalizing our model to Bayesian models is natural in most scenarios
\citep{wang2024bayesian}.

\subsection{Tensor Input Linear Regression}

Instead of adopting a constant baseline,
which ignores any further dependence between the outcome and predictors, 
our base model will use existing linear regression models tailored for tensor inputs. 
The more flexible baseline model will not only improve predictive performance, but it will also 
generally lead to more parsimonious trees and therefore improve computational
efficiency. 

Next we briefly review the relevant tensor-input linear
models considered later as base models.
When the regression function $f$ in \eqref{eq:vector-input regression} is assumed to
be   linear, %
we can fit a \emph{vector-input linear regression model }for
estimation and prediction of the output $y$. In matrix product notation,
this model is $\bm{y}=\bm{X}\bm{\beta}+\bm{\epsilon}$
with $\bm{X}\in\mathbb{R}^{n\times d}$ and $\bm{y}\in\mathbb{R}^{n\times1}$.
We can estimate the model by solving the loss function (i.e., $L_{2}$
residual) $\|\bm{y}-\bm{X}\bm{\beta}\|_{2}$ for the coefficient $\bm{\beta}\in\mathbb{R}^{d\times1}$.
This can be generalized to \emph{tensor-input (multi-)linear regression} 
with a tensor input and
a scalar response \citep{zhou2013tensor,LLM2023}. %
Without loss of generality, 
we consider data pairs consisting of a 3-way input tensor $\bm{X}_{i}\in\mathbb{R}^{1\times d_{1}\times d_{2}}$, where
the first dimension is an observation index as in vector-input models, and a scalar response $y_{i}\in\mathbb{R}^{1}$. 
This setup can directly handle tensor
inputs via the tensor product $\circ$ in the tensor regression model: 
\begin{align}
y_{i} =\bm{X}_{i}\circ\bm{B}+E_{i}, \quad i=1,\cdots,n,\quad
\label{eq:tensor_linear_model_bm}
\end{align}
where the matrix $\bm{y}\in\mathbb{R}^{n\times1}$ stacks the responses $y_1, \ldots, y_n$, 
the tensor $\bm{B}\in\mathbb{R}^{d_{1}\times d_{2}}$ collects the regression coefficients, and the tensor $\bm{E}\in\mathbb{R}^{n\times1}$ stacks the i.i.d.\ Gaussian noise $E_{i}\in\mathbb{R}^{1}$.
{The tensor products are defined as $\bm{X}\circ\bm{B}=\sum_i^n\sum_{j_1}^{d_1}\sum_{j_2}^{d_2}\bm{X}[i,j_1,j_2]\bm{B}[j_1,j_2]$ in} \citet{zhou2013tensor,guo2011tensor} {(Eq. (5)) but the assumed model} \eqref{eq:tensor_linear_model_bm} {is compatible with other tensor products as well.
}%
As in the vector-input case, the solution to \eqref{eq:tensor_linear_model_bm} corresponds to
the least-squares problem $\min_{\bm{B}^{(n)}}\|\bm{y}^{(n)}-\bm{X}^{(n)}\circ\bm{B}^{(n)}\|_{2}$. In this paper, the superscript $(n)$ will denote an object created by stacking $n$ quantities (corresponding to $n$ observations) along the first dimension. {For $\bm{X}^{(n)}$ we stack all $n$ samples; for $\bm{B}^{(n)}$ the shared coefficient $\bm{B}$ is simply repeated $n$ times.} If it is obvious that the object is stacked, we will omit the superscript.

As \citet{zhou2013tensor} and \citet{li2018tucker} noted, scalar-on-tensor linear regression can extend to $D$-mode tensor inputs $\bm{X}\in\mathbb{R}^{n\times d_{1}\times d_{2}\times\cdots\times d_{D}}$ with more complex notations. We focus on $D=3$ and $D=4$ (i.e., $\bm{X}\in\mathbb{R}^{n\times d_{1}\times d_{2}}$ and $\bm{X}\in\mathbb{R}^{n\times d_{1}\times d_{2}\times d_{3}}$), which are common in applications, with notation centered on $D=3$. The model \eqref{eq:tensor_linear_model_bm} has high-dimensional coefficients $\bm{B}$  complicating model fitting and optimization. Current state-of-the-art implementations like CatBoost \citep{prokhorenkova2018catboost} and XGBoost \citep{Chen16} support up to $D=2$ and overlook spatial correlation in $\bm{X}\circ\bm{B}$. We use the state-of-the-art $\mathtt{tensorly}$
\citep{JMLR:v20:18-277} implementation for scalar-on-tensor CP/Tucker regressions for leaf models 
in our tree regression method.

Tensor decomposition \citep{kolda2009tensor,johndrow2017tensor} is
a popular approach to reduce dimensionality and hence model
complexity by aiming to capture multi-way interactions in the
data using lower-dimensional representations. 
In particular, CP and Tucker decompositions
extend the idea of low-rank approximation (e.g., SVD) from matrices
to tensors, offering unique ways to explore and model the possible
sparse structure in tensors. \emph{CP decomposition} represents
a tensor $\bm{T}\in\mathbb{R}^{t_{1}\times t_{2}\times t_{3}}$
as a linear combination of $R$ rank-$1$ tensors $\mathbf{a}_{r}\times\mathbf{b}_{r}\times\mathbf{c}_{r}$:
\begin{equation}
\bm{T}=\sum_{r=1}^{R}\lambda_{r}\mathbf{a}_{r}\times\mathbf{b}_{r}\times\mathbf{c}_{r}\label{eq:CP}
\end{equation}
where $\times$ denotes the outer product (i.e., Kronecker product)
between vectors, and $\lambda_{r}$ is the scalar weight of the $r$-th
component for \emph{factor vectors} $\mathbf{a}_{r}\in\mathbb{R}^{t_{1}},\mathbf{b}_{r}\in\mathbb{R}^{t_{2}},\mathbf{c}_{r}\in\mathbb{R}^{t_{3}}$,
reducing the number of coefficients in $\bm{B}$ from $nd_{1}d_{2}$
to $R(1+t_{1}+t_{2}+t_{3})\leq R(1+n+d_{1}+d_{2})$. \emph{Tucker
decomposition} represents the tensor $\bm{T}$ 
using a (usually dense)
core tensor $\mathbf{G}\in\mathbb{R}^{R_{1}\times R_{2}\times R_{3}}$,
and factor matrices $\mathbf{A}_{1}\in\mathbb{R}^{R_{1}\times t_{1}},\mathbf{A}_{2}\in\mathbb{R}^{R_{2}\times t_{2}},\mathbf{A}_{3}\in\mathbb{R}^{R_{3}\times t_{3}}$
where $R_{1}\leq t_{1},R_{2}\leq t_{2},R_{3}\leq t_{3}$: 
\begin{equation}
\bm{T}=\mathbf{G}\times_{1}\mathbf{A}_{1}\times_{2}\mathbf{A}_{2}\times_{3}\mathbf{A}_{3}=\sum_{r_{1}=1}^{R_{1}}\sum_{r_{2}=1}^{R_{2}}\sum_{r_{3}=1}^{R_{3}}\lambda_{r_{1},r_{2},r_{3}}\mathbf{a}_{r_{1}}\times\mathbf{b}_{r_{2}}\times\mathbf{c}_{r_{3}}.\label{eq:Tucker}
\end{equation}
Here $\times_{q}$ denotes the \emph{product
along mode-$q$} ($q=1,2,\cdots,D$), which is the multiplication of a tensor by another tensor
along a specific mode by permuting the mode-$q$ in front and permuting
it back. That is, we pull one dimension to the front,
flatten it, transform it by a matrix, then put everything
back the way it was, except that the numbers along one dimension are
mixed and changed \citep{kolda2009tensor}. The core tensor $\mathbf{G}$
has rank $(R_{1},R_{2},R_{3})$ along each mode but is typically
much smaller in size than $\bm{T}$. We use the second equivalent
summation as in (3) of \citet{li2018tucker} for $\lambda_{r_{1},r_{2},r_{3}}\in\mathbb{R},\mathbf{a}_{r_{1}}\in\mathbb{R}^{t_{1}},\mathbf{b}_{r_{2}}\in\mathbb{R}^{t_{2}},\mathbf{c}_{r_{3}}\in\mathbb{R}^{t_{3}}$,
reducing the number of coefficients parameters from $nd_{1}d_{2}$
to $(1+t_{1}R_{1}+t_{2}R_{2}+t_{3}R_{3})\leq(1+nR+d_{1}R+d_{2}R)$
when we use the same rank $R$ along each mode.

For tensor regression \eqref{eq:tensor_linear_model_bm}, we adopt
a low-rank decomposition for the coefficient $\bm{B}$ 
using CP or Tucker decompositions and use the factor matrices through
the linear model \eqref{eq:tensor_linear_model_bm}, following the state-of-the-art practice in \citet{li2018tucker,zhou2013tensor}.

\paragraph{Organization.}
The rest of the paper is organized as follows: Section~\ref{sec:Trees-for-Tensor}
introduces the ingredients for fitting a single tree model with tensor
inputs, including a discussion of splitting criteria and complexity-based pruning. 
Section~\ref{sec:Ensemble-of-Trees}
reviews ensemble techniques for improving single trees before introducing entry-wise
and low-rank methods for handling tensor output using ensembles (Section~\ref{subsec:Multi-way-Output-Tensor}). Section~\ref{sec:Theoretical-Guarantees}
provides two groups of theoretical results including the consistency
of leaf models and the oracle bounds for predictions. Section
\ref{sec:Data-Experiments} investigates the effect of using novel
splitting criteria (Section~\ref{subsec:Effect-of-different}), compares
to other tensor models in terms of prediction and efficiency (Section~\ref{subsec:Comparison-with-Other}), and concludes with a tensor-on-tensor
application example (Section~\ref{subsec:Tensor-on-tensor:-Image-Recovery}).
Section~\ref{sec:Conclusion} provides discussions and future works.
{The Supplementary Material contains additional simulation experiments}, including for robustness to entrywise input noise, and for comparing to existing random-forest methods.

\section{\label{sec:Trees-for-Tensor}Trees for Tensor Inputs}
We now propose and implement fast algorithms for the scalar-on-tensor
regression problem %
\begin{align}\label{eq:tensor-input regression}
y_{i}=g^*(\bm{X}_{i})+E_{i},i=1,\ldots,n
\end{align}
for inputs $\bm{X}_{i}\in\mathbb{R}^{d_1 \times d_2}$ and responses $y_{i}\in\mathbb{R}$.
Here $g^*\colon \mathbb{R}^{d_1 \times d_2}\rightarrow\mathbb{R}$ is a real-valued function
we wish to estimate and the $E_{i}$s are independent mean zero noises.
Extending the tree model \eqref{eq:single_tree}
and creating ensembles for tensor input regression require ingredients
from low-rank tensor approximations (i.e., \eqref{eq:CP} and \eqref{eq:Tucker})
and associated low-rank tensor input regressions. 
\subsection{\label{subsec:Splitting-Criterion}Splitting Criterion}

For continuous vector inputs, the split rule $\bm{X}[:,j_{1}]>c$ in \eqref{eq:single_tree}
and Algorithm \ref{alg:Vector-input-decision} can be written as the rule $\bm{X}^\top\bm{e}_{j_1}>c$
using the design matrix $\bm{X}$ and unit vector
$\bm{e}_{j_1}$ along splitting coordinate $j_1$. 
For continuous \textit{tensor} inputs,
we can consider the split rule $\bm{X}\circ\bm{e}_{j_{1},j_{2}}>c$ and the following splitting criteria: 

\noindent \textbf{Variance criterion (SSE).} Under this criterion, which generalizes the well-accepted criterion \eqref{eq:SSE} in the tree regression literature
\citep{breiman1984classification} to scalar-on-tensor regressions \eqref{eq:tensor-input regression}, we choose a dimension pair $(j_{1},j_{2})$ and observed value $\bm{X}[j_{0},j_{1},j_{2}]$ so that the induced pair of children
\begin{align}
R_{1}\coloneqq\left\{ \bm{X}\mid\bm{X}[:,j_{1},j_{2}]\leq\bm{X}[j_{0},j_{1},j_{2}]\right\} \quad
R_{2}\coloneqq\left\{ \bm{X}\mid\bm{X}[:,j_{1},j_{2}]>\bm{X}[j_{0},j_{1},j_{2}]\right\}\label{eq:tensor_leftright}
\end{align}
minimizes the sum of variances (here $I_{j}\coloneqq\{i\mid\bm{X}[i,j_{1},j_{2}]\in R_{j}\}$ and $\bm{y}[I_j,:]=\{y_{i}\mid i\in I_j\}$) 
\begin{align}
\text{SSE}(j_{0},j_{1},j_{2}) = \sum_{j=1}^{2}\frac{1}{N_j}\sum_{\bm{X}_{i}\in R_{j}}(y_{i}-\hat{y}_{R_{j}})^{2} 
=\sum_{j=1}^{2}\frac{1}{N_j} \left\Vert \bm{y}[I_{j},:]-\frac{1}{|I_{j}|}\bm{1}_{|I_{j}|}^\top\bm{y}[I_{j},:]\cdot\bm{1}_{|I_{j}|}\right\Vert _{2}^{2}\label{eq:sse-tensor}
\end{align}
over all such possible pairs of children. 

\noindent \textbf{Low-rank approximation error (LAE).} 
This criterion aims to partition the input space by leveraging the potential low-rank structure in the tensor input $\bm{X}$. 
Rather than use the predictors, we
can split based on how closely the children tensor inputs match their low-rank CP or Tucker approximations.
Namely, we choose the child pair \eqref{eq:tensor_leftright}
whose tensor inputs $\bm{X}[I_{j},:,:]\in\mathbb{R}^{|I_{j}|\times d_{1}\times d_{2}}$ (for $j=1,2$) and corresponding tensor low-rank approximations $\widetilde{\bm{X}}[I_{j},:,:]$ minimize the error (again, here $I_{j}\coloneqq\{i\mid\bm{X}[i,j_{1},j_{2}]\in R_{j}\}$)
\begin{align}
\text{LAE}(j_{0},j_{1},j_{2})=\sum_{j=1}^{2}\sum_{i\in I_{j}}\left\Vert \widetilde{\bm{X}}[i,:,:]-\bm{X}[i,:,:]\right\Vert _{F}^{2}\label{eq:low-rank LAE}
\end{align}
among all possible such child pairs.
This criterion does not use the predictors $\bm{y}$ and is more expensive to compute than
\eqref{eq:sse-tensor} as described by Lemma \ref{lem:(CP-ALS-per-iteration}
and \ref{lem:(Tucker-ALS-per-iteration} below.

\noindent \textbf{Low-rank regression error (LRE).} Alternatively, we can find
a split that minimizes the low-rank \textit{regression} error in each child: 
\begin{align}
\text{LRE}(j_{0},j_{1},j_{2})=\sum_{j=1}^{2}\sum_{i\in I_{j}}\left\Vert \widetilde{\bm{y}}\left(\bm{X}[i,:,:]\right)-y_{i}\right\Vert _{F}^{2},\label{eq:low-rank-reg LRE}
\end{align}
where $\widetilde{\bm{y}}\left(\bm{X}[I_{j},:,:]\right)$ is the
chosen CP \eqref{eq:CP} or Tucker \eqref{eq:Tucker} low-rank regression
model (of a given target rank) of $\bm{X}[I_{j},:,:]$ against $\bm{y}[I_{j}]$.
Intuitively, this minimizes the error of predicting $\bm{y}$ in each child, which is analogous to CART minimizing
the sum of variances of $\bm{y}$ in each node. Because we fit two models for two children nodes, \eqref{eq:low-rank-reg LRE}
is slightly more expensive than \eqref{eq:low-rank LAE}. 
Any given split with $D=3$ requires a split-coordinate pair $(j_{1},j_{2})$
and index $j_{0}$ for the split value $c=\bm{X}[j_{0},j_{1},j_{2}]$.
The $j_{0},j_{1},j_{2}$ triplet is usually obtained by solving the following
mixed integer problem of dimensionality $nd_{1}d_{2}$: 
\begin{equation}
\min_{(j_{0},j_{1},j_{2})\in\{1,\cdots,n\}\times\{1,\cdots,d_{1}\}\times\{1,\cdots,d_{2}\}}\mathcal{L}(j_{0},j_{1},j_{2})\label{eq:generic_opt}
\end{equation}
where the criterion function $\mathcal{L}$ is chosen as one of
\eqref{eq:sse-tensor}, \eqref{eq:low-rank LAE} or \eqref{eq:low-rank-reg LRE}.
This is a major computational overhead for generating a tree structure
via (binary) splitting.

The splitting criteria SSE and LRE involve
both $\bm{X}$ and $\bm{y}$, whereas LAE is specifically designed for
low-rank tensors $\bm{X}$ (or mat-vec that essentially has low-rank
structures). Because LAE does not involve the response $\bm{y}$,
the learned tree based on LAE will well approximate
the mean response only when the underlying latent tree structure
in $\bm{X}$ corresponds to that in $\bm{y}$. %
On the other hand, SSE and LRE more directly target the prediction
task and more easily produce a good fit to the data, but also make
the algorithm more prone to overfitting, though this issue can be
easily addressed with proper regularization techniques such as pruning and out-of-sample
validation.

\begin{figure}[t]
\centering

\includegraphics[height=2.5cm]{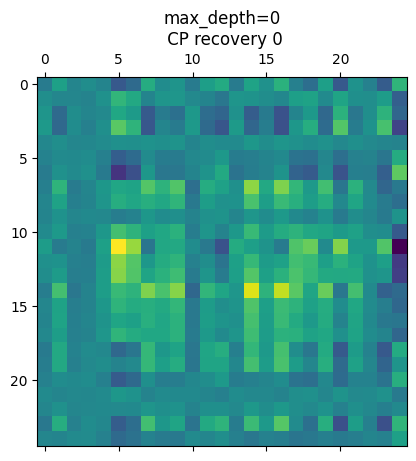}\\
 \includegraphics[height=2.5cm]{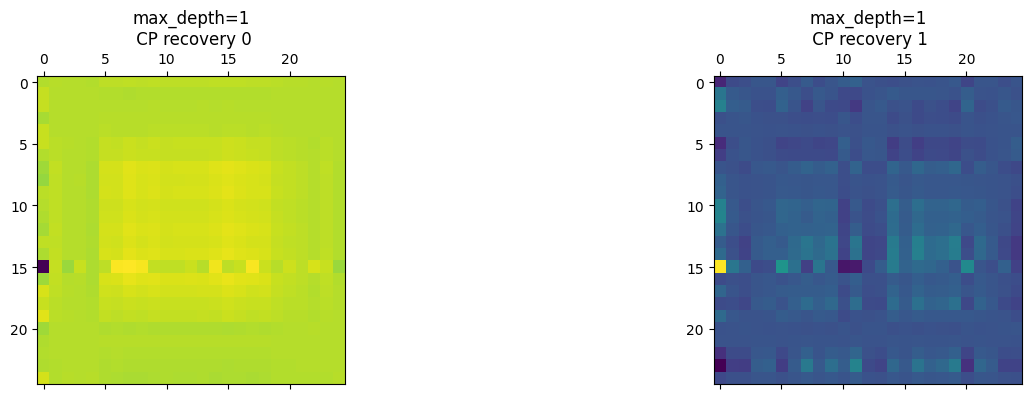}\\
 \includegraphics[height=2.5cm]{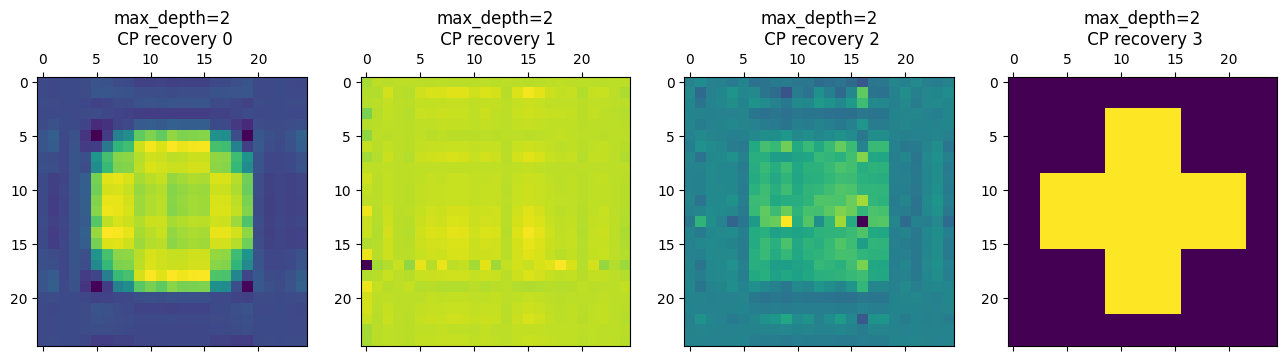}\\

\includegraphics[height=2.5cm]{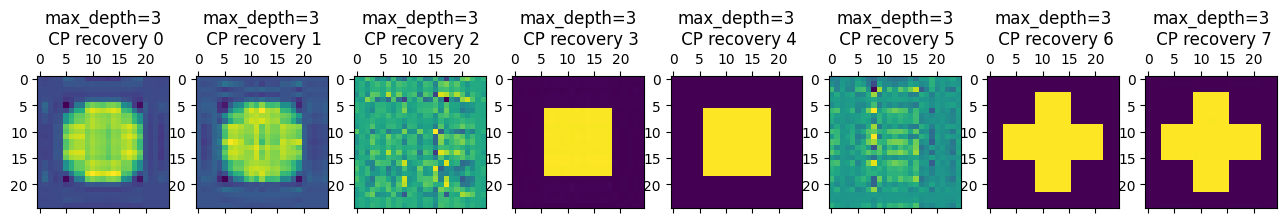}

\caption{\label{fig:Comparison-of-different-depths}Tensor-input decision
tree regression estimated coefficients $\hat{\bm{B}}$ on a mixture-of-classes
model ($\tau=0.1$) with rank-2 CP regression at leaf models from the circle-rectangle-cross example in \citet{zhou2013tensor}.
Each row corresponds to a max\_depth$\in\{0,1,2,3\}$.
}
\end{figure}

The key idea behind the LRE loss is to incorporate low-rank tensor
regression as the baseline models on the leaves of a partition tree.
Trees can capture non-linear local structures
in the underlying mean functions but are ineffective at approximating
global smooth structures such as linear patterns on some latent dimensions
which low-rank tensor regression targets. As such, we construct
trees using novel splitting criteria and complexity measures,
then fit tensor regressions at each leaf node.
{
To show the effectiveness of our model, we first generate three classes of coefficient tensors
$\bm{B}_{1},\bm{B}_{2},\bm{B}_{3}$, and then simulate i.i.d. Gaussian
tensors $\bm{X}_{i}$ to yield the response $\bm{y}_{i}=\bm{X}_{i}\circ\bm{B}_{i}+\delta_{i}$
for each class, where the elements of $\delta_{i}$ are independent 
shifted noises $N(i,\sigma^{2})$
to ensure identifiability between these three classes.
}
Figure \ref{fig:Comparison-of-different-depths} shows a scenario
of mixed class data \citep{zhou2013tensor} in which existing tensor
regressions (i.e., the top row) fail to recover each class but a \textit{tree}-based regression model is successful with appropriate splits. 
In addition, this tree-based model can realize
typical low-rank CP and Tucker regression models as special cases
when the max depth of the tree is set to be zero (i.e., when the tree has only one node) and the only model is a single low-rank regression model. 
The effects of the splitting and pruning  criteria will be further investigated in Section~\ref{subsec:Effect-of-different} and Appendix~\ref{subsec:Pruning-exp}.

{As a final note, the variance and LRE splitting criteria can be used to tackle a regression problem with a tensor response and vector predictor. On the other hand, a CP or Tucker decomposition of a vector predictor is just the predictor itself, so the LAE error would be zero for any pair of potential children and hence cannot be used with a vector predictor.}

\begin{algorithm}

\SetAlgoLined \KwResult{Train a tensor-input single decision tree
regressor with exhaustive search}

\textbf{Function} fit($\bm{X}\in\mathbb{R}^{n\times d_{1}\times d_{2}}$,
$\bm{y}\in\mathbb{R}^{n}$):

\Begin{ Initialize root node.\ %
\For{node
in tree}{ Calculate for all potential splits along each dimension pair $(j_{1},j_{2})$:
Find $\arg\min_{j_{0}\in\{1,\cdots,n\},j_{1}\in\{1,\cdots,d_{1}\},j_{2}\in\{1,\cdots,d_{2}\}}LAE(j_{0},j_{1},j_{2})$
or $LRE(j_{0},j_{1},j_{2})$  \; 
With the minimizer $(j_{0}^{*},j_{1}^{*},j_{2}^{*})$, split the dataset on
the chosen dimension pair $(j_{1}^{*},j_{2}^{*})$ and split value $\bm{X}[j_{0}^{*},j_{1}^{*},j_{2}^{*}]$, creating left and right
child nodes using data in $R_{1},R_{2}$ \; 
} 
Fit the chosen mean/CP/Tucker models at each of the leaf nodes.
}


\caption{\label{alg:Tensor-input-decision-exhaustive}\textbf{Exhaustive search method for fitting tensor-input decision tree regressor with low-rank
splitting criteria \eqref{eq:low-rank LAE} and \eqref{eq:low-rank-reg LRE}.}}
\end{algorithm}

\subsection{\label{subsec:Reducing-Computational-Cost}Reducing Computational
Cost }

There are two major computational overheads in fitting the tree tensor
regression model we just proposed: leaf model fitting and splitting criteria computation. 
The cost of the former is
alleviated by the tree structure providing a partitioning
regime that allows divide-and-conquer scaling. 
The latter, however, is a new computational challenge born from
the nature of tensor data as we explain below.

It is computationally very expensive to exhaustively search (see Algorithm \ref{alg:Tensor-input-decision-exhaustive}) for the best split node in tensor-input decision trees. This process evaluates all possible splits along each dimension pair $(j_{1}, j_{2})$, which requires examining every combination of $j_0$, $j_1$, and $j_2$.
Consequently, the search space size for each split decision is $n \cdot d_{1} \cdot d_{2}$, which scales poorly with increasing dataset size and tensor dimensions, making exhaustive search computationally intensive and impractical for large-scale data or real-time applications.%

In addition, the computational complexity of \eqref{eq:low-rank LAE}
and \eqref{eq:low-rank-reg LRE} is rather high since both CP and
Tucker decompositions rely on expensive alternating algorithms, which
often have convergence issues. Therefore, the greedy search approach
over a space of size $d_{1}d_{2}$ to solve $\min_{(j_{0},j_{1},j_{2})}\mathcal{L}(j_{0},j_{1},j_{2})$
is no longer an efficient algorithmic design, given the frequency
of splitting in fitting a regression tree structure.

\paragraph{Splitting criteria complexity analysis.}

Computing the variance in \eqref{eq:sse-tensor} for scalar responses
has complexity $\mathcal{O}\left(n\right)$, which can
be vectorized and does not scale with dimensions. On the other hand, the complexities
of \eqref{eq:low-rank LAE} and \eqref{eq:low-rank-reg LRE} are dominated
by the CP and Tucker decomposition of the coefficient tensors.
Although many scalable methods are proposed for CP and Tucker
decompositions, the well-adopted alternating least square (ALS) \citep{malik2018low} solvers
are of iterative nature and hence can have relatively high complexities: 
\begin{lem}
\label{lem:(CP-ALS-per-iteration}(CP-ALS per iteration complexity
\citet{minster2023cp} Section 3.2) In the alternating least square
algorithm for the rank-$R$ CP decomposition of a 
$n\times d_{1}\times d_{2}\times\cdots\times d_{K}$ tensor
where $R<n$ and $R\ll d_{i},i=1,\cdots,K$, each iteration has time complexity
$\mathcal{O}((K+1)\cdot\prod_{i=1}^{K}d_{i}\cdot R)$.
Therefore, for $N^{*}$ iterations in CP-ALS, the overall complexity
is bounded from above by $\mathcal{O}\left(N^{*}\cdot(K+1)R\cdot(\max_{i}d_{i}\right)^{K})$. 
\end{lem}

\begin{lem}
\label{lem:(Tucker-ALS-per-iteration}(Tucker-ALS per iteration complexity
\citet{oh2018scalable} Section II.C) In the alternating least square
algorithm for the rank-$(R,d_{1}^{'},d_{2}^{'},\cdots,d_{K}^{'})$ (i.e.,
core tensor dimension) Tucker decomposition of a $n\times d_{1}\times d_{2}\times\cdots\times d_{K}$ tensor,
each iteration has complexity 
\[
\mathcal{O}\left(\min\left\{n\prod_{j=1}^{K}d_{j}^{'2},n^{2}\prod_{j=1}^{K}d_{j}^{'}\right\}+\sum_{i=1}^{K}\min\left\{R^{2}\cdot d_{i}\prod_{j\neq i}d_{j}^{'2},R\cdot d_{i}^{2}\prod_{j\neq i}d_{j}^{'}\right\}\right).
\]
When $R<n$ and $d_{i}^{'}\asymp d_{i}\ll n,i=1,\cdots,K$, for $N^{*}$
iterations in Tucker-ALS, the overall complexity is at most $\mathcal{O}\left(N^{*}\cdot n\cdot(\max_{i}d_{i}\right)^{2K})$. 
\end{lem}

To simplify our subsequent analysis, we assume $n\asymp R$
and always choose a constant maximal iteration number $N^{*}$
regardless of the convergence. Then both decompositions have time
complexities at most $\mathcal{O}(N^{*}\cdot n\cdot\left(\max_{i}d_{i}\right)^{2K})$
which scales linearly with $n$. However, the 
factor $N^{*}\cdot\left(\max_{i}d_{i}\right)^{2K}$ greatly increases complexity in greedy-style tree-fitting.

\paragraph{Adoption of mean splitting values.}

Minimizing these three 
criteria
can be a strategy to choose the split-value index $j_{0}$
and splitting coordinates $(j_{1},j_{2})$, and can easily generalize to regression trees for tensor inputs. Finding the best combination of $(j_{0},j_{1},j_{2})$ using \eqref{eq:SSE},
\eqref{eq:low-rank LAE}, \eqref{eq:low-rank-reg LRE}, in a worst-case
scenario, needs $n\cdot\prod_{j=1}^{K}d_{j}$-many variance evaluations
and contributes to a quadratic complexity in $n$. 
But if we do not optimize over $j_{0}$ and instead split on the sample mean $m_{j_1,j_2} \coloneqq \frac{1}{n}\sum_{i=1}^{n}\bm{X}[i,j_{1},j_{2}]$, 
we can simplify the partitions
in \eqref{eq:SSE}, \eqref{eq:low-rank LAE}, \eqref{eq:low-rank-reg LRE}
as: 
\begin{align}
\bar{R}_{1}\coloneqq\left\{ \bm{X}\mid\bm{X}[:,j_{1},j_{2}]\leq m_{j_1,j_2}\right\}, \quad
\bar{R}_{2}\coloneqq\left\{ \bm{X}\mid\bm{X}[:,j_{1},j_{2}]> m_{j_1,j_2}\right\} \label{eq:tensor_leftright-1}. 
\end{align}
Since we no longer need to determine $j_{0}$, we can complete
the search within linear time in $n$, without losing too much empirical
prediction performance. %
Then the corresponding loss functions in our reduced optimization problems become (here $I_{j} \coloneqq \{i\mid\bm{X}[i,j_{1},j_{2}]\in\bar{R}_{j}\}$):
%
\begin{align}
\overline{\text{SSE}}(j_{1},j_{2})=\sum_{j=1}^{2}\frac{1}{|I_j|}\sum_{i\in I_{j}}\left\Vert \hat{y}_{\bar{R}_j} - y_{i} \right\Vert _{2}^{2} ,\label{eq:mean-SSE}\\
\overline{\text{LAE}}(j_{1},j_{2})=\sum_{j=1}^{2}\sum_{i\in I_{j}}\left\Vert \widetilde{\bm{X}}[i,:,:]-\bm{X}^{(n)}[i,:,:]\right\Vert _{F}^{2},\label{eq:mean-LAE}\\
\overline{\text{LRE}}(j_{1},j_{2})=\sum_{j=1}^{2}\sum_{i\in I_{j}}\left\Vert \widetilde{\bm{y}}\left(\bm{X}^{(n)}[i,:,:]\right)-y_{i}\right\Vert _{F}^{2}.\label{eq:mean-LRE}
\end{align}
Removing the search along the first mode shrinks the dimensionality of the optimization
problem $\min_{(j_{0},j_{1},j_{2})}\mathcal{L}(j_{0},j_{1},j_{2})$
to reduce the complexity per evaluation of the split criterion.

Figure~\ref{fig:Comparison-of-different-SMs} shows
that the loss functions LAE and LRE lead to similar complexities and predictive behavior.
Although both losses are based on low-rank
tensor decomposition, LRE (like SSE) focuses on the children's predictive
loss. In what follows, we mainly study TT models with LRE
in \eqref{eq:low-rank-reg LRE} or its mean version \eqref{eq:mean-LRE}
as our default loss functions, unless otherwise stated.
Appendix~\ref{sec:LS_BB_methods} details two advanced searching techniques, namely leverage score sampling (LS) and branch-and-bound (BB) methods for even more efficient searches. 

In the spirit of Lemma 11 of \citet{hrluo_2022e} and the above Lemmas \ref{lem:(CP-ALS-per-iteration} and \ref{lem:(Tucker-ALS-per-iteration}, we can derive the per-iteration complexity (of ALS used for computing CP and Tucker decompositions) by observing that there are $\mathcal{O}(n \log k )$ evaluations of the loss function in a tree with $k$ nodes ($\log k$ layers) and $n\geq 1$ sample points:
\begin{prop}
    
\label{prop:(Computational-complexity-for}(Computational complexity
for TT) If there are $n\geq 1$ samples with at most $k$ nodes  for the splitted  
tree structure and no more than $N^*<\infty$ iterations for all decompositions, the worst-case computational complexity 
for generating this tree structure (via exhaustive search for problem \eqref{eq:generic_opt}) is 

(1) $\mathcal{O}(nd_1d_2 \log k )$ if the loss function is \eqref{eq:sse-tensor}. 

(2) $\mathcal{O}(n^2d_1^2 d_2^2 \log k )$ if the loss function is \eqref{eq:low-rank LAE} or \eqref{eq:mean-LAE} with CP low-rank approximations. 

(3) $\mathcal{O}(n^2d_1^2 d_2^2 \cdot  \left[ \min( d_1 d_2, n) + \min(d_1, d_2) \right] \log k  )$ if the loss function is \eqref{eq:low-rank-reg LRE} or \eqref{eq:mean-LRE} with Tucker low-rank approximations. 
\end{prop}
\begin{proof}
See Appendix~\ref{sec:Proof-of-Complexity}. 
\end{proof}
This complexity result indicates that the tensor tree model has $\mathcal{O}(n^2\cdot C)$ complexity (where $C=C(d_1,d_2)$ depends on the dimensionality of the tensor inputs) 
compared to the larger $\mathcal{O}(n^3)$ complexity (whose constant inside big O also depends on the kernel evaluation) incurred by tensor GP models \citep{yu2018tensor} as shown later in Figure \ref{fig:Scaling_ssize}. %
We can also see that both LAE and LRE share the same big O complexity in terms of $n$. As explained in Section \ref{subsec:Splitting-Criterion}, we will focus on LRE since it parallels the SSE in regular regression trees.  

\begin{figure}[t]
\centering
\includegraphics[width=0.75\paperwidth]{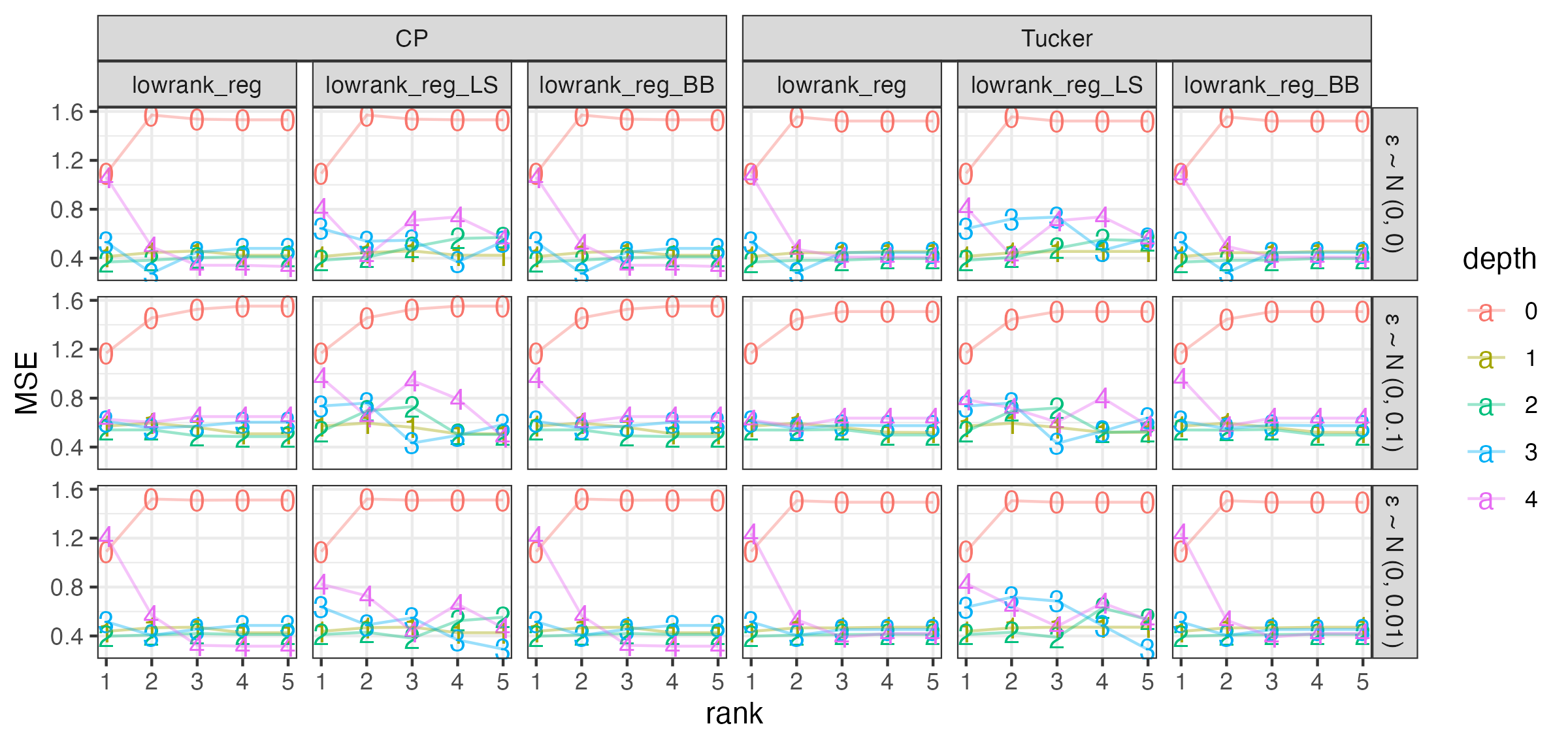}
\caption{\label{fig:Comparison-of-different-SMs}Out-of-sample MSEs for tree regression models of different maximal
depths, using a $(500,5,4)$ input tensor $\bm{X}$ sampled uniformly
randomly from $[-1,1]$ and scalar output $\bm{y}=2\bm{X}[:,0,1]\cdot\bm{X}[:,2,3]+3\bm{X}[:,1,0]\cdot\bm{X}[:,2,0]\cdot\bm{X}[:,3,0]+\varepsilon$.
Columns 1\&4: Full search; Columns 2\&5: LS with $\tau=0.5$; Columns 3\&6: BB with $\delta=0.5$.
}
\end{figure}

\subsection{\label{subsec:Complexity-Measure}Complexity-based Pruning}

As with traditional tree-based regression methods, pruning and other
forms of regularization are necessary to avoid overfitting. Adaptive
pruning can be achieved based on a complexity measure $C_{\alpha}(T)$
following the definition in (9.16) in \citet{hastie2009elements}.

Revisiting the three possible splitting criteria in Section \ref{subsec:Splitting-Criterion},
we find that variance splitting can correspond naturally to $Q_{m}(T)$ in (9.16) of \citet{hastie2009elements}. 
For the low-rank splitting proposed in Algorithm \ref{alg:Tensor-input-decision-LS},
we propose the following modified complexity measure $C_{\text{tl}_{\alpha}}(T)$
which includes a low-rank induced error term $Q_{\text{tl}_{m}}(T)$
and is defined as: 
\begin{equation}
C_{\text{tl}_{\alpha}}(T)=\sum_{m=1}^{T}N_{m}Q_{\text{tl}_{m}}(T)+\alpha T_{l}\label{eq:tl_complexity}
\end{equation}
The measure $C_{\text{tl}_{\alpha}}$ generalizes $C_{\alpha}$
by accounting for the structure and relationships in multi-dimensional
data. 
The term $Q_{\text{tl}_{m}}(T)$ can be chosen as one of \eqref{eq:SSE},
\eqref{eq:low-rank LAE} or \eqref{eq:low-rank-reg LRE} to evaluate
the fitness of the leaf models given a tree structure $\mathcal{T}$
(i.e., bias component). 
This captures the inherent complexity in tensor data and reflects
how well the model fits the underlying tensor structure. The variance
term remains the same, representing the complexity of the tree
through the number of leaves.
By considering the low-rank prediction errors, the measure $C_{\text{tl}_{\alpha}}$ (as discussed in \citet{hastie2009elements})
also balances the trade-off between fitting the tensor structure (the first summation term in \eqref{eq:tl_complexity} as bias)
and the complexity of the decision tree (the second summation term in \eqref{eq:tl_complexity} as variance). 
The measure $C_{\text{tl}_{\alpha}}$ provides a way to prune the tree
$\mathcal{T}$ and maintain a simple model (high bias, low variance),
which is essential for achieving good generalization performance.

In our TT models, we implement recursive pruning (without swapping
or other kinds of reduction operations, for simplicity) based on whether
a specific branch will reduce complexity. %
When we compute the first term in \eqref{eq:tl_complexity} (the second term will be affected by the choice of $\alpha$) we can observe %
that a deeper tree may not necessarily improve the overall fit. 

%
%
Pruning is needed to get a good fit and prediction
in a tensor-input tree regression model. We may adjust $\alpha$ in the term $\alpha T_{l}$
 to favor shallower trees. 
If we only use $\sum_{m=1}^{T}N_{m}Q_{\text{tl}_{m}}(T)$
in \eqref{eq:tl_complexity} to guide our pruning, the depth of the
tree tensor model should be 0. We provide an example showing the effect of pruning in Appendix \ref{subsec:Pruning-exp}. 
%
%

%
%

%

%
%
%

%

%

%
%
%
%
%
%
%
%
%
%
%
%
%
%
%
%
%
%

%
%
%
%
%
%
%
%

\paragraph{Summary of Single Tensor Tree 
Model.}
Our efficient \emph{tensor tree}
(TT) models are constructed for scalar responses and tensor inputs 
with $n$ observations and tensor dimensions $d_{1},d_{2}$.
At leaf nodes we propose to fit low-rank scalar-on-tensor
regression CP/Tucker models \citep{zhou2013tensor,li2018tucker} to
adapt to low-rank structures in tensors. This is a  novel
generalization to non-linear tree models \citep{chaudhuri1994piecewise,li2000interactive}
onto non-standard tensor-input domains, and also a natural generalization
of low-rank tensor models to handle heterogeneity (Figure~\ref{fig:Comparison-of-different-depths}).

In tensor regression trees, splitting involves entire tensor dimensions, based on low-rank approximation criteria using functions \eqref{eq:low-rank LAE} to \eqref{eq:mean-LRE}. 
Unlike traditional trees that split on single features, tensor trees assess multiple dimensions' interactions, increasing computational demands and potential overfitting. This complexity necessitates regularization through pruning and challenges interpretability due to complex decision boundaries. Low-rank approximations are computationally demanding, prompting scalability solutions in tensor settings \citep{kolda2008scalable,liu2017low}.
\begin{enumerate}
\item We propose criteria \eqref{eq:mean-SSE}, \eqref{eq:mean-LAE}
and \eqref{eq:mean-LRE} to split the tree structure $\mathcal{T}$,
which reduces the dimensionality of the optimization problem $\min_{(j_{0},j_{1},j_{2})}\mathcal{L}(j_{0},j_{1},j_{2})$ along with randomized and branch-and-bound optimization methods (i.e., LS and BB). 
\item We generalize complexity measures into \eqref{eq:tl_complexity} using
new criteria for tree pruning, which achieves parsimonious partition
structures. 
\end{enumerate}

\section{\label{sec:Ensemble-of-Trees}Ensemble of Trees}

Continuing our discussion of TT models, 
a single tree model tends to have large variance and overfit the data, which can produce inconsistent predictions.
The idea behind tree ensemble modeling is to combine multiple models in order to improve the performance
of a single TT model. %
Ensemble methods are foundational in modern machine learning
due to their effectiveness in improving predictive performance.

In the context of tensor regressions (see Appendix~\ref{subsec:Random-Forest-and} for detailed discussions), gradient boosting is particularly
powerful. Each tree corrects the residuals of the previous ensemble,
enabling the model to capture complex patterns in multi-dimensional
data, which is crucial for tensor inputs and outputs. Gradient Boosting (GB) can also
incorporate regularization, such as shrinkage and tree constraints,
to prevent overfitting, which is vital when modeling high-dimensional
tensor data. 
Boosting algorithms iteratively
build decision trees, and the complexity of constructing trees grows
with the dimensionality of the input space. Tensor inputs significantly
increase this dimensionality, making tree construction inefficient
and potentially leading to overfitting 
or redundant splits over the flattening space. 

Algorithm~\ref{alg:Gradient-Boosted-Tensor} combines gradient boosting and adaptive sampling with the multi-dimensional
prowess of tensor regression trees in order to achieve highly accurate
predictive models suited especially for complex, multi-dimensional
datasets. We will be able to combine $m$ different weak learners
in a residualized back-fitting scheme, and choose a learning rate
(which is usually greater than $1/m$) that allows the aggregated
model to leverage the prediction power. 

\subsection{\label{subsec:Multi-way-Output-Tensor}Tensor-on-tensor: Multi-way
Output Tensor Tree Ensemble}

We now consider a regression problem for tensor inputs $\bm{X}_{i}$ and tensor responses $\bm{y}_{i}$: 
\begin{align}\label{eq:tensor-tensor regression}
\bm{y}_{i}=\bm{g}^*(\bm{X}_{i})+\bm{E}_{i},i=1,\ldots,n.
\end{align}
for a tensor-valued function $\bm{g}^*\colon \mathbb{R}^{d_1 \times d_2}\rightarrow\mathbb{R}^{p_1\times p_2}$ and independent mean zero noises $\bm{E}_{i}$.
We focus on this special case but note that this approach can extend to higher order tensors.
\citet{lock2018tensor} proposes to treat a \emph{linear} %
tensor-on-tensor regression problem. 
Their methodology performs optimization directly on $\|\bm{y}-\bm{X}\circ\bm{B}\|_{F}$
with an optional regularization term penalizing the tensor rank of
the coefficient tensor $\bm{B}$. 
To tackle the more general tensor-on-tensor problem \eqref{eq:tensor-tensor regression}, we propose two approaches for TT models based on
ensemble modeling using tree regressions \citep{breiman1984classification,Breiman01}.
In what follows, each GB ensemble will use pruning parameter $\alpha=0.1$
and $10$ single trees unless otherwise stated.
We can also use other ensemble methods (e.g., random forest, Adaboost
\citep{breiman1999prediction}) or even different kinds of ensembles
for different entries (in entrywise approach) or components (in lowrank
approach). {(Appendix~\ref{sec:forestcomparison} contains simulation experiments} comparing to existing random-forest methods.) 
As mentioned earlier, we will focus on ``ensemble of GB'': 


\noindent \textbf{Entry-wise approach.} The first approach (TTentrywise-CP,
TTentrywise-Tucker) is to use a single GB ensemble to predict vector
slices $\bm{y}[:,i,j],i=1,\cdots,p_{1};j=1,\cdots,p_{2}$ of the tensor
output as if the slices are mutually independent.
This approach ignores dependence between slices,
but improves computational cost since the break-down into $p_{1}\times p_{2}$ single TT trees and the overall complexity is scaled up by a factor of $\mathcal{O}( p_{1}\times p_{2})$.
This method is prone to overfitting, but can work well even when evaluating on testing
data if the data has a large signal-to-noise ratio or if the elements of the output tensor have minimal
or no correlation.

\noindent \textbf{Low-rank approach.} The second approach (TTlowrank-CP, TTlowrank-Tucker) is to perform CP/Tucker decomposition on the tensor output $\bm{y}$ and use a single GB ensemble
to predict each component (i.e., $\mathbf{a}_{r},\mathbf{b}_{r},\mathbf{c}_{r}$
in \eqref{eq:CP}; $\bm{A}_{1},\bm{A}_{2},\bm{A}_{3}$ in \eqref{eq:Tucker})
in the low-rank decomposition. The low-rank decompositions
align tensor structures but can be computationally intensive (the cost of decomposing $\bm{y}$ can be obtained from 
Lemmas \ref{lem:(CP-ALS-per-iteration} and \ref{lem:(Tucker-ALS-per-iteration}),
especially when leaf nodes contain fewer samples than the selected rank
of the decomposition which may not be enough to ensure a reasonable
fit and generalization ability. {To avoid overfitting, a general rule of thumb is to choose a higher rank for tensors with highly correlated entries in order to retain more correlation, but the optimal rank of decomposition for the tensor output is generally data-dependent.} This approach allows parallelism only within
each component as matrices. In addition to higher dimensionality,
evaluating tensor gradients puts additional computational burden
for this direct generalization. %
This approach allows us to break down the regressions into vector components for fitting using single trees %
and the overall complexity is scaled up by these factors that is dependent on the ranks we choose for CP or Tucker decompositions.
%

%
    
%
%
%

%

%
%

%
%


The approach \emph{TTlowrank} offers a low-rank representation of multi-way
data by leveraging the structures and relationships within tensor
outputs. Decomposing the output tensor in this way enables powerful regression
models that can capture complex patterns in the data. {We implicitly assume the response admits a low-rank structure (which may not always be realistic) and hence perform
a low-rank decomposition on the response} and record weights (if CP) or core tensor
(if Tucker) but fit our TT regressions on the factors. Then on a new input $\bm{X}$, we predict the corresponding factors and reconstruct the tensor output predictions using the recorded weights (if CP) or core tensor
(if Tucker). 
Intuitively,
CP decomposition preserves potential diagonality, while
Tucker decomposition preserves potential orthogonality in the
output tensor. This approach is suitable when the elements of the
output tensor are correlated and exhibit complex relationships, and
yield much smaller models compared to the naïve approach. 

%
\subsection{Discussion on performance}
Here we empirically compare the predictive performance of \citet{lock2018tensor}'s approach with that of our two TT approaches. 
To generate our datasets, we
sample entries of $\bm{X}$ from a uniform distribution on $[0,1]$
and compute $\bm{y}$ using functions listed in the table in Appendix~\ref{sec:RPErrr}
(whose ranges are all within $[-1,2]$), and add a uniform noise of
scale 0.01. Training and testing sets are generated independently.
For rrr methods \citep{lock2018tensor},
we set the $\mathtt{R}$ rank parameter and leave the remaining parameters
as the package default. 
For TT methods, we set both splitting
ranks and regression ranks to be the same as $\mathtt{R}$, with{\small{}{}{}
max\_depth} = 3 and pruning measures $\alpha=0.01$. We use a GB ensemble
with 10 estimators and learning rate 0.1 uniformly.

The table in Appendix~\ref{sec:RPErrr} shows the out-of-sample errors of each model, where we see that at least one of the TT methods either ties or outperforms both rrr methods in 23 of the 24 tested scenarios.
The comparison can be separated into two groups. One group (Linear/Non-linear) explores the
power of using GB ensemble TT models (with pruning) along with the
two approaches for tensor-on-tensor regression tasks (here for simplicity we choose the same regression rank
$R$ for CP models or $(R,R,R)$ for Tucker models for both input
$\bm{X}$ and output $\bm{y}$). %
Using the RPE metric, we can observe
that TT with GB for tensor output performs quite competitively
with the rrr models and improves as the low-rank models' rank increases,
and only TT model can possibly capture non-linear interactions as
expected, since higher rank means better approximation to the $\bm{y}$. 
In the TTlowrank method
with Tucker decomposition on $\bm{y}$, the Tucker decomposition
is behaving very badly regardless of ranks and shows non-convergence
after maximal iterations are reached, since the components in $\bm{y}$ possess high correlations.
This partially explains why its RPE is greater than 1 for the linear
and non-linear examples, and its performance is the worst for this
group due to the fact that we do not assume low-rank structure in
the output $\bm{y}$. 
It seems that the practice of using rank $(R,R,R)$ for Tucker decomposition
on $\bm{y}$ does not reflect the tensor structure in the output well in this group of experiments.
In both linear and non-linear scenarios, the lowest RPE values are often seen with rrrBayes and TTlowrank\_CP methods, which closely align with rrr methods. This indicates that tree-based tensor regression methods improve in efficiency as rank increases due to better approximation quality.

The second set of experiments (Exact CP/Exact Tucker) explores scenarios with significant TT trees and pruning in GB ensembles. By setting max\_depth>0, localized effects on data from tree partitioning help capture periodic behaviors in learners. Both TTentrywise and TTlowrank methods perform well under CP decomposition, showing competitive performance against rrr models. However, Tucker decomposition can experience non-convergence issues at $(R,R,R)$ ranks, identified by $\mathtt{tensorly}$, necessitating re-runs to confirm convergence. In these cases, signals depend solely on low-rank $\bm{X}$, benefiting TT models with appropriate ranks $R$. {In practice, the mode-specific ranks can be different like $(R_1,R_2,R_3)$ for better flexibility; and the model-specific ranks can also vary for different leaf nodes to account for possible heterogeneity.} 

\paragraph{Summary of Tensor Tree Ensembles}

Boosting methods like GB employ functional gradient descent to minimize a specified loss function. The process involves calculating pseudo-residuals, fitting new trees to these residuals, and updating the model iteratively. 

To address multi-way regression problems by performing optimization directly on tensor outputs, we deploy a GB ensemble of TT along with two alternative approaches --- entry-wise and low-rank ---  for tensor regression trees (TT models) using ensemble methods.

\begin{figure}
\centering

\includegraphics[width=0.75\textwidth]{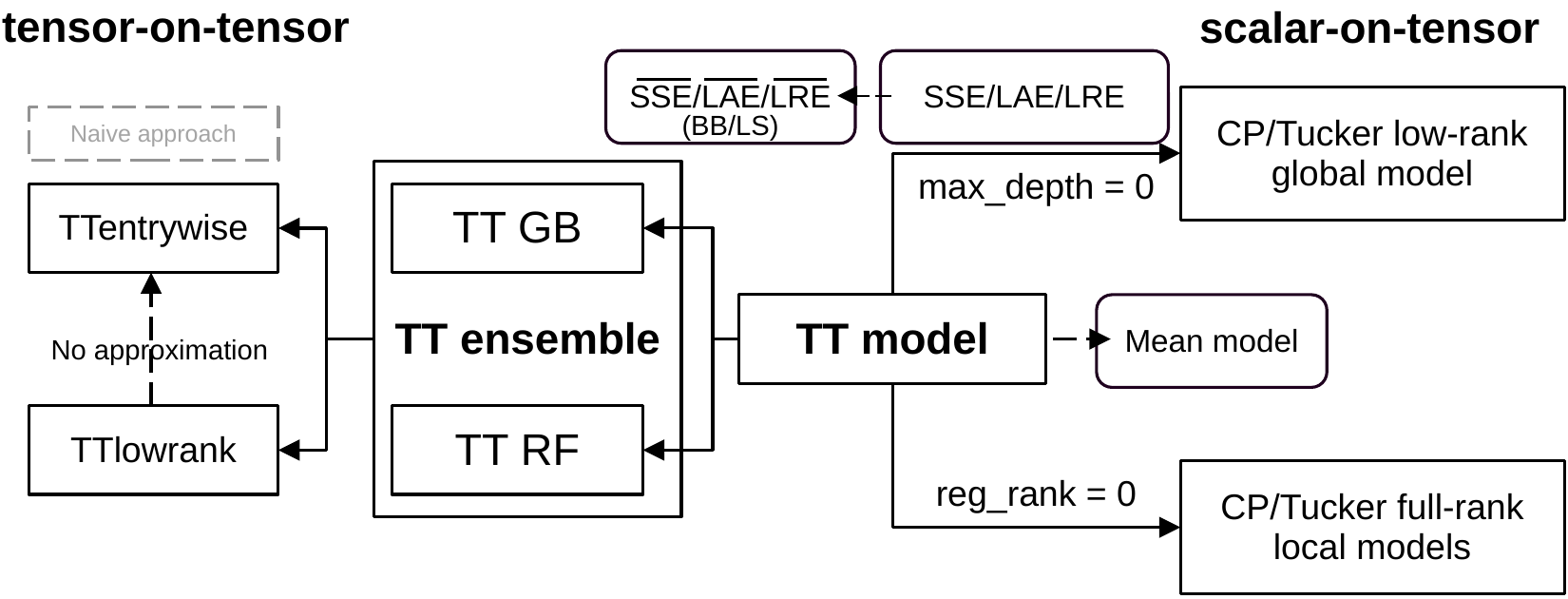}

$ $%

\caption{\label{fig:summary_figure_TT}Design of the TT models (scalar-on-tensor)
and ensembles (tensor-on-tensor).}
\end{figure}
Figure~\ref{fig:summary_figure_TT} illustrates how we construct ensemble
models and how GB (or RF) ensemble of TT models can be utilized to
perform tensor-on-tensor regressions. 
We adopt boosting for tensor output regressions for their ability to sequentially improve models by focusing on residuals, making them suitable for complex tensor-on-tensor regression tasks.

\section{\label{sec:Theoretical-Guarantees}Theoretical Guarantees}

Although this paper focuses on developing new modeling methodologies and fast training algorithms, we use existing theoretical results to support the validity of our modeling designs. The results in this section all assume scalar responses.

\subsection{Coefficient Asymptotics}

A fixed tree of depth $K$ has at most
$2^{K}$ leaf nodes. %
\textcolor{black}{In the recursive partition literature (e.g., (2.27) in \citet{gordon1978asymptotically}), it is commonly assumed that} 
for each
leaf node $t$, the number of samples $n_{t}$ tends to infinity when
the overall sample size $n$ tends to infinity. This condition obviously holds when
we split on the median in the vector-input case, but is not necessarily
true if we split according to \eqref{eq:tensor_leftright} or \eqref{eq:tensor_leftright-1}.
\textcolor{black}{Under this assumption, within each leaf node we can directly apply existing coefficient estimate consistency results for CP and Tucker regression models \citep{zhou2013tensor,li2018tucker}.}
\begin{prop}
\label{prop:For-a-fixed}For a \textit{fixed complete binary tree} of depth
$K<\infty$, 
\textcolor{black}{assume for each leaf node $t$ that the number of i.i.d. samples $n_{t}\rightarrow\infty$ and the leaf model parameter space is compact. }

\textcolor{black}{(a) (Theorem 1 in Section 4.3 of \citet{zhou2013tensor}) For the leaf node CP
model 
%
with mean-zero normal error and $\bm{X}[i,:,:]\in R_{t}$, suppose
the true coefficient $\bm{B}_{0}(t)=\sum_{r=1}^{R}\lambda_{r}(t)\mathbf{a}_{r}(t)\times\mathbf{b}_{r}(t)\times\mathbf{c}_{r}(t)$
of node $t$ is identifiable up to permutation in the sense of Proposition
4 in \citet{zhou2013tensor}. Then the MLE $\hat{\bm{B}}_{n}(t)$
converges to $\bm{B}_{0}(t)$ in probability.}

\textcolor{black}{(b) (Theorem 1 in Section 4.3 of \citet{li2018tucker}) For the leaf node Tucker model 
with mean-zero normal error and $\bm{X}[i,:,:]\in R_{t}$, suppose
the true coefficient $\bm{B}_{0}(t)=\mathbf{G}(t)\times_{1}\mathbf{A}_{1}(t)\times_{2}\mathbf{A}_{2}(t)\times_{3}\mathbf{A}_{3}(t)$
of node~$t$ is identifiable up to permutation in the sense of Proposition
3 in \citet{li2018tucker}. Then the MLE $\hat{\bm{B}}_{n}(t)$
converges to $\bm{B}_{0}(t)$ in probability.}
\end{prop}

Though often presupposed, the condition $n_{t}\rightarrow\infty$ assumes fixed tree and leaf node partitions $R_{t}$, which is
unrealistic since typically the partition changes
as $n\rightarrow\infty$. It also requires an additive
true function $f(\bm{X})=\sum_{t\text{ is a leaf node}}\bm{X}\circ\bm{B}_{0}(t)\cdot\bm{1}\left(\bm{X}\in R_{t}\right)$
whose summands are piece-wise multilinear functions.
Also, even under these restrictions we still do not have estimates for the
ranks of tensor coefficients, as pointed out by \citet{zhou2013tensor}
and \citet{li2018tucker}.
However, Proposition \ref{prop:For-a-fixed} %
confirms one important
aspect that if $f(\bm{X})=\bm{X}\circ\bm{B}_{0}$ for some global
$\bm{B}_{0}$ (for CP or Tucker model), our tree model produces
consistent coefficient estimates for \emph{arbitrary} (hence dynamic)
partitions, since $\bm{B}_{0}=\bm{B}_{0}(t)$ and
\begin{equation}
\bm{X}\circ\bm{B}_{0}=\bm{X}\circ\bm{B}_{0}\cdot \sum_{t\text{ is a leaf node}}\bm{1}\left(\bm{X}\in R_{t}\right) =\sum_{t\text{ is a leaf node}}\bm{X}\circ\bm{B}_{0}\cdot\bm{1}\left(\bm{X}\in R_{t}\right).\label{eq:no_part}
\end{equation}
With this positive result, our tensor tree model for any \textcolor{black}{maximum depth} $K$ 
is no worse than the standalone CP or Tucker model in terms of consistency
under $f(\bm{X})=\bm{X}\circ\bm{B}_{0}$.
\subsection{Oracle Error Bounds for \eqref{eq:sse-tensor}}

As seen above, the fitted mean function from a decision tree regression
model lies in the additive class with iterative summation across tensor
dimensions $d_{1},d_{2}$:%
\begin{equation}
\mathcal{G}^{1} \coloneqq \left\{ g \colon \mathbb{R}^{d_1 \times d_2} \rightarrow \mathbb{R}  \middle|\ g(\bm{X})=\sum_{i_{2}=1}^{d_{2}}\sum_{i_{1}=1}^{d_{1}}g_{i_{1},i_{2}}(\bm{X}[i_{1},i_{2}])\right\} ,\label{eq:G1_class}
\end{equation}
which is essentially the same functional class as that in \citet{klusowski2021universal}
but with both dimensions additively separable. Within this class, a more useful
result regarding the prediction risk can be established.

\begin{thm}
\label{thm:(Empirical-bounds-for}(Empirical bounds for TT) 
Suppose data $\bm{X}^{(n)}\in\mathbb{R}^{n\times d_{1}\times d_{2}}$ and
$\bm{y}\in\mathbb{R}^{n\times1}$ are generated by model %
\eqref{eq:tensor-input regression} 
where $g^*$ is not necessarily in $\mathcal{G}^1$.
Let $g_K$ be the regression function of a TT model split by minimizing \eqref{eq:sse-tensor} with maximal depth $K$ and fitting mean models at the leaves.
Then the empirical risk $\hat{\mathcal{R}}(g)=n^{-1}\sum_{i=1}^n(y_{i} - g(\bm{X}^{(n)}[i,:,:]))^2$ has the bound:
\begin{align*}
\hat{\mathcal{R}} \left(g_{K}\right) & \leq \inf_{g\in\mathcal{G}^{1}}\left\{ \hat{\mathcal{R}} \left(g\right)+\frac{\left\Vert g\right\Vert _{\text{TV}}^{2}}{K+3}\right\}
\end{align*}
where $\left\Vert g\right\Vert _{\text{TV}}$ is (the infimum over all additive representations of $g \in \mathcal{G}^1$ of) the aggregated total variation of the individual component functions of $g$.
\end{thm}

\begin{proof}
See Appendix~\ref{sec:Proof-of-Theorem-1oracle}. 
\end{proof}

Using the same argument as in Theorem
4.3 in \citet{klusowski2021universal}, the following theorem bounds the expected $L_{2}$
error between the true signal $g^{*}$ (not necessarily in $\mathcal{G}^1$) and the mean function $g_{K}$
of a complete tree of depth $K\geq1$ constructed by minimizing
\eqref{eq:sse-tensor}.

\begin{thm}
(Oracle inequalities for TT) Suppose data $\bm{X}^{(n)}\in\mathbb{R}^{n\times d_{1}\times d_{2}}$
and $\bm{y}\in\mathbb{R}^{n\times1}$ are generated by the model $y=g^{*}(\bm{X})+\varepsilon$
for additive sub-Gaussian noise
$\varepsilon$ (i.e., $\mathbb{P}(\varepsilon\geq u)\leq2\exp(-u^{2}/2\sigma^{2})$
for some $\sigma^{2}>0$). 
Let $g_K$ be the regression function of a TT model split by minimizing \eqref{eq:sse-tensor} with maximal depth $K$ and fitting mean models at the leaves.
Then
\begin{align*}
\mathbb{E}\left(\left\Vert g^{*}-g_{K}\right\Vert ^{2}\right) & \leq2\inf_{g\in\mathcal{G}^{1}}\left\{ \left\Vert g^{*}-g\right\Vert ^{2}+\frac{\left\Vert g\right\Vert _{\text{TV}}^{2}}{K+3}+C_1\cdot\frac{2^{K}\log^{2}(n)\log(nd_1d_2)}{n}\right\} 
\end{align*}
where $C_{1}$ is a positive constant depending only on $\|g^{*}\|_{\infty}$
and $\sigma^{2}$. 
\end{thm}

This result does not cover CP
or Tucker regressions at the leaf nodes. Also,
if the tree is split using the LAE criterion \eqref{eq:low-rank LAE}, then the correlation-based expansion in \citet{klusowski2021universal}
will no longer work. 
It remains an open question whether the LAE-based Algorithm~\ref{alg:Tensor-input-decision-LS} is consistent even when $K=1$.

\section{\label{sec:Data-Experiments}Data Experiments}

\subsection{\label{subsec:Effect-of-different}Effect of Different Splitting
Criteria}

\begin{figure}[H]
\includegraphics[width=1\textwidth]{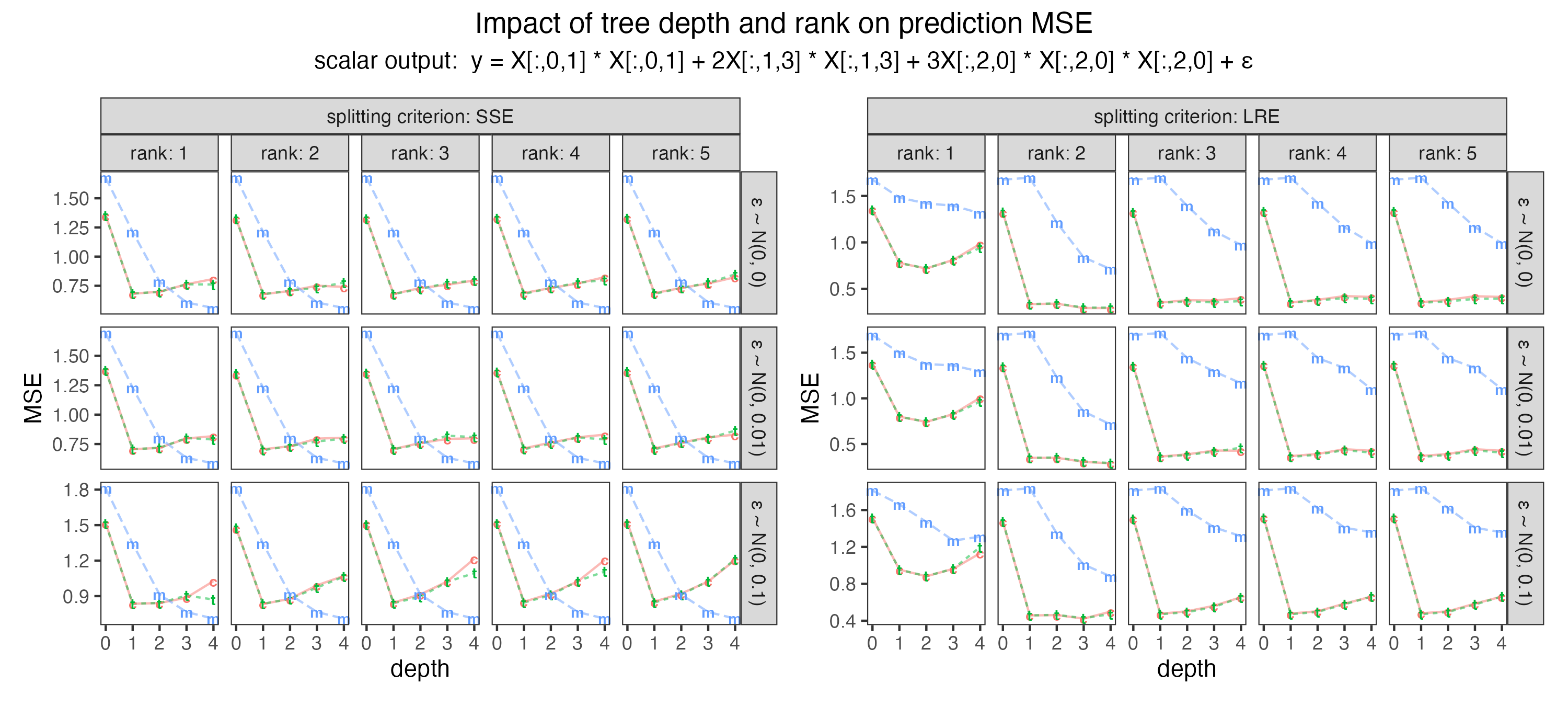} 
\includegraphics[width=1\textwidth]{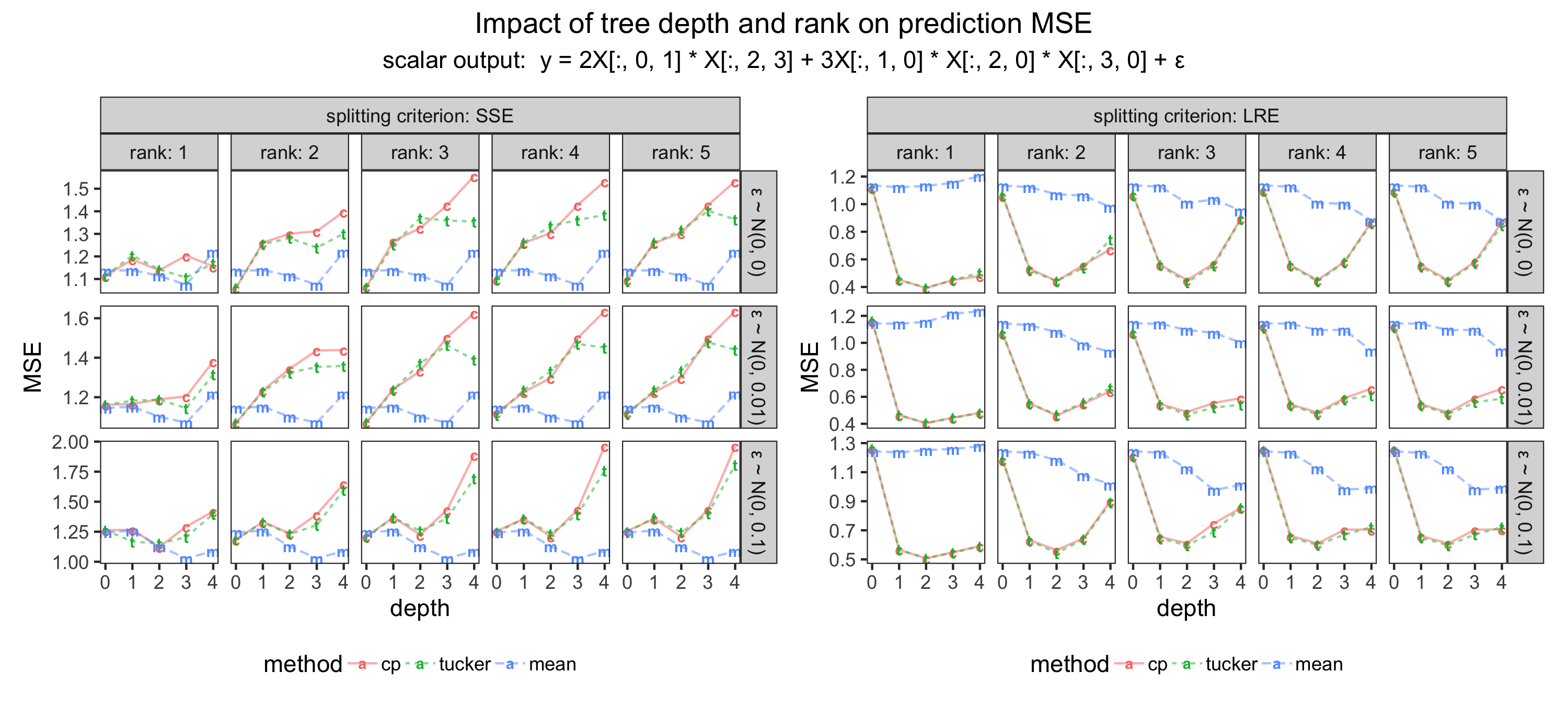}
\caption{\label{fig:COM-2}
Out-of-sample MSE comparison of CP, Tucker 
($\tau=1$), and mean predictions for tree models of different
max\_depth and split by either the SSE \eqref{eq:sse-tensor} or LRE
\eqref{eq:low-rank-reg LRE}, we use the same splitting and fitting ranks. 
We sample a $(1000,5,4)$ input tensor uniformly from $[-1,1]$. MSEs are averaged over 10 seeds.  
}
\end{figure}
 %

%
%
%
%
%
%
%
%
%
%
%
%
%
%
%
%
%
%
%
%
%
%
%




Splitting criteria allow tree-based regression models to capture interactions between the
predictors that would not be captured by a standard regression model \citep{hastie2009elements,klusowski2020sparse}. This approach can also be generalized
to tensors of any order. Conditioned on the splitted partitions, we
may take the sample mean as predictive values, or fit a regression
model (e.g., linear or tensor regression) to the data on each partition.

The standard setting of our model is
to let the split\_rank parameter to be the same as CP\_reg\_rank (or
Tucker\_reg\_rank, see Table \ref{tab:Supported-combination-methods:} for detailed parameter specification). 
In most of the applications considered in this paper, we
set split\_rank to be the same as the rank chosen for the leaf
node regressions. This practice reflects our belief that the
model and recursive partition regime share the same low-rank structure
assumption. However, it is possible to choose split and regression
ranks to be different and even adaptively when the low-rank structure within each group or as the tree grows deeper. 

\textbf{Selecting splitting criteria.} We examine the MSE in Figure~\ref{fig:COM-2}\footnote{In the SSE experiment with rank 2 and $N(0,0.1)$ noise,
we found 4 out of 10 cases raising possible non-convergence warnings
from $\mathtt{tensorly}$, for other experiments there are 1 or 2.}. 
When split using the SSE criterion \eqref{eq:sse-tensor}, the tensor models perform comparably with the mean model, with neither consistently outperforming the other. The underlying mechanisms that split the data according to the variation of $\bm{y}$ do not necessarily capture the tensor structure better.

In contrast, when split using the LRE criterion \eqref{eq:low-rank-reg LRE}, the tensor models perform better than the mean models. 
Hence the remainder of the subsection highlights the enhanced performance of tensor models when utilizing the LRE criterion.
Interestingly, the MSE of the CP/Tucker models does not monotonically decrease as decomposition rank and tree depth increase. 
In Figure~\ref{fig:COM-2}, the MSEs at rank 1 are different from the MSEs at ranks 2,3,4,5. Rank also seems to affect MSEs at larger tree depths. Otherwise, rank seems to not make much difference in these two examples. 
Regarding depth, the top half of Figure~\ref{fig:COM-2} shows that with an appropriate rank (i.e., ranks 2,3,4,5), 
our tree-based low-rank model exhibit an L-shaped MSE pattern with respect to depth can outperform the mean model for any of the tested tree depths. 
However, in the bottom half of Figure~\ref{fig:COM-2},
the CP/Tucker models exhibit a U-shaped MSE pattern with respect to depth and may perform worse than a mean model if the tree depth is not appropriately selected.
%
%
%
%
%
%
%
%
%
%
%
%
%
%
%
%
%
%
%
%

\textbf{Interaction in signals.} The tensor model is quite good at
capturing the non-separable interactions like $\bm{X}[:,1,3]*\bm{X}[:,1,3]$
and $\bm{X}[:,1,0]*\bm{X}[:,2,0]*\bm{X}[:,3,0]$ in the tensor input.
In Figure~\ref{fig:COM-2}, where triplet interactions
$\bm{X}[:,2,0]*\bm{X}[:,2,0]*\bm{X}[:,2,0]$ or $\bm{X}[:,1,0]*\bm{X}[:,2,0]*\bm{X}[:,3,0]$
exist, the 3-way tensor decomposition helps to fit a low-rank regression
model to outperform the usual mean model used along with tree regressions.

On one hand, increasing the regression
ranks\footnote{We choose the split\_rank and CP\_reg\_rank/Tucker\_reg\_rank to be
the same as our default.} will have a saturation effect, where the performance may not change
after a certain rank threshold. This is clear in both LAE and LRE
columns. The difference is that increasing LAE may deteriorate the
performance (Figure \ref{fig:COM-2}) while LRE cannot deteriorate
the performance, since it is purely splitting based on reducing the tensor regression ranks of the chosen leaf models. %

On the other hand, the tree depth (max\_depth) will not have this
saturation effect. We can observe that there is usually a 'sweetspot'
depth (e.g., depth 1,2,3 in the bottom half of Figure~\ref{fig:COM-2}). If we keep increasing
the tree depth, we will enter the second half of the U-shape and deteriorate
the tree-based tensor-model.
As shown by Figure~\ref{fig:COM-2}, the low-rank
CP regression models at leaf nodes improve prediction
MSE compared to the traditional mean predictions but depends on the depth. 
%
%
%
%
%
%

\subsection{\label{subsec:Comparison-with-Other}Comparison with Other Tensor
Models}

The Tensor Gaussian Process (TensorGP, \citet{yu2018tensor}) extends
Gaussian Processes (GPs, \citet{hrluo_2019a}) to high-dimensional
tensor data using a multi-linear kernel with a low-rank approximation.
In regular GPs, the covariance matrix $\bm{K}$ is derived from a
kernel function $k(\bm{X},\bm{X}')$ which typically depends on the
Euclidean distance between points $\bm{X},\bm{X}'$. For $n$ data points, the matrix $\bm{K}$ has size $n\times n$ with elements $\bm{K}_{ij}=k(\bm{X}[i,:],\bm{X}[j,:])$.


%

\begin{figure}[H]
\centering

\includegraphics[width=\textwidth]{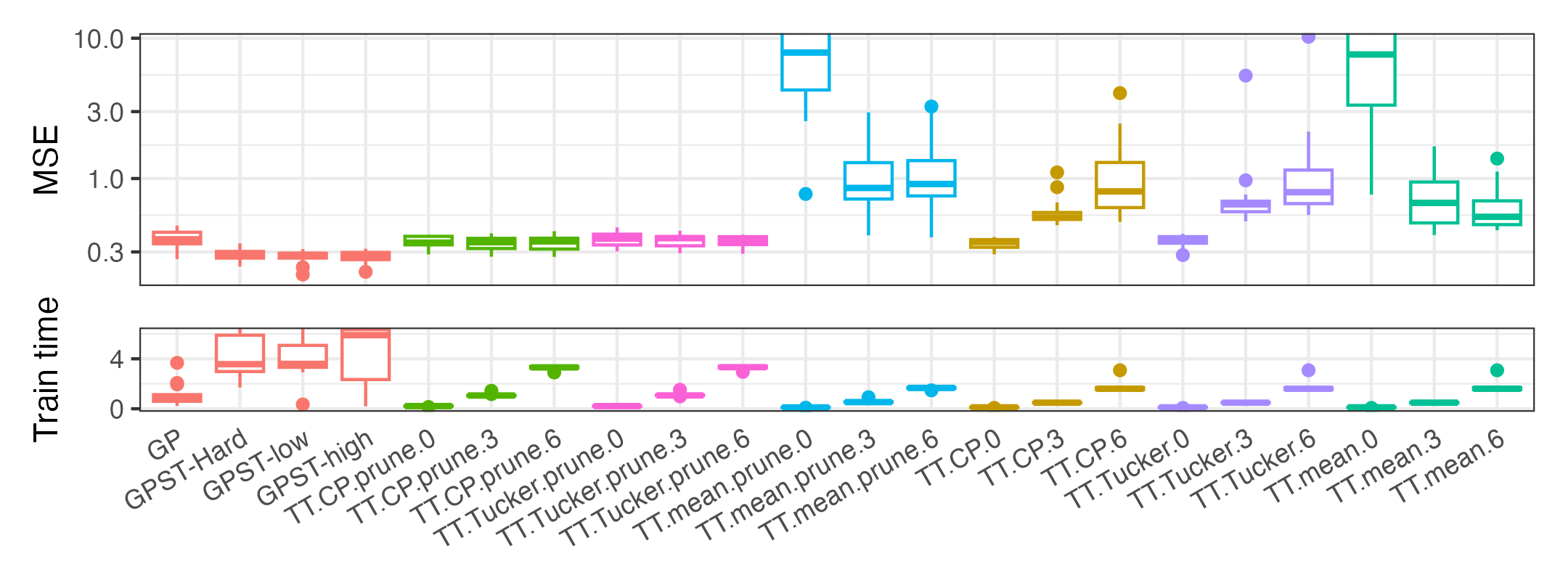}

\caption{\label{fig:MSE-time-of-nonpara-models-GP} 
{Out-of-sample 
MSE and training time (minutes)} of our tensor-input tree
(TT), tensor-GP, GPST, CP (ranks = 3), and Tucker (ranks = {[}3,3,3{]}). 
All models are fitted (with 75\% training
set of size 1000 4-dimensional tensor) on 20 batches of the synthetic
datasets in Table~1 of \citet{sun2023tensor}. 
The final character of each TT method label indicates the method's set maximum depth.
Our TT model is split
using SSE criterion \eqref{eq:sse-tensor}. For the TT\_prune methods, the tree is pruned with $\alpha=0.1$ in \eqref{eq:tl_complexity}.
}
\end{figure}

Our TT method is significantly faster than the $\mathcal{O}(n^{3})$ complexity of GP models, making it feasible for high-dimensional data. The TensorGP model optimization seeks optimal low-rank matrices to preserve data structures using a covariance kernel in tensor form. Similarly, the Tensor-GPST model \citep{sun2023tensor}, uses tensor contraction for dimension reduction, akin to pre-PCA in standard GP models. This includes anisotropic total variation regularization for generating sparse, smooth latent tensors.

\begin{figure}[H]
\centering

\includegraphics[width=\textwidth]{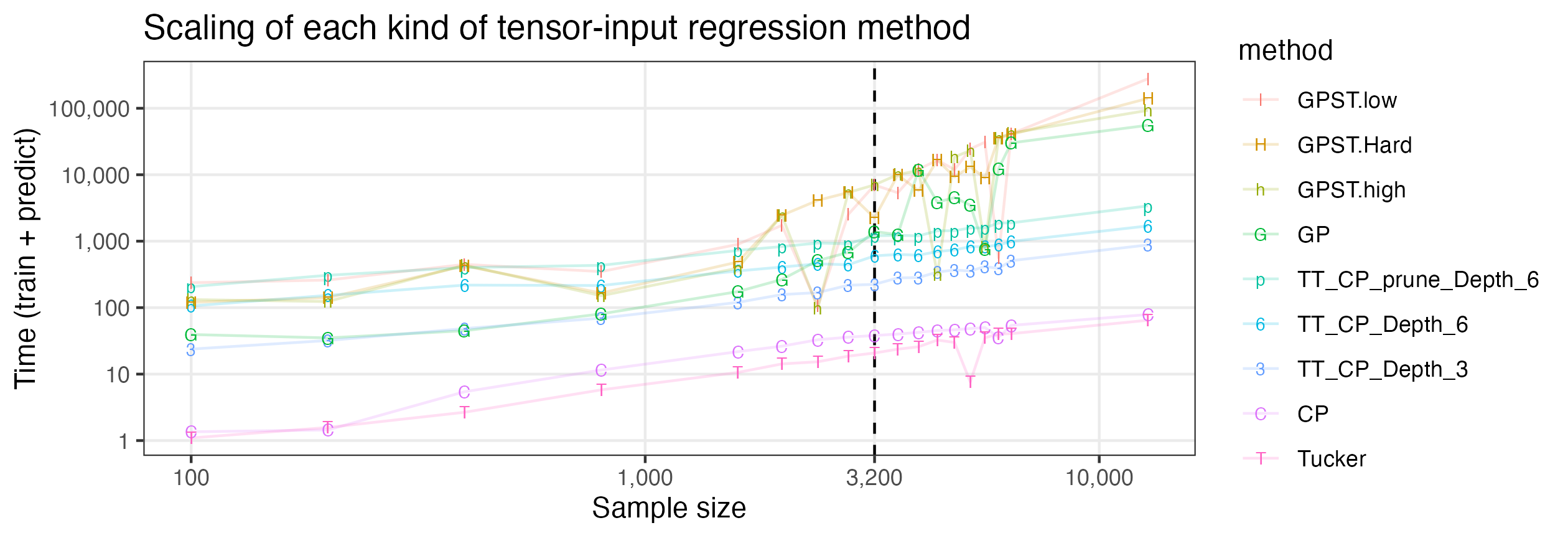}

\caption{\label{fig:Scaling_ssize}  In-sample and out-of-sample training
time (seconds) against the sample size for the experiments
in Figure~\ref{fig:MSE-time-of-nonpara-models-GP}. For TT models with
pruning, we experiment with max\_depth=3,6 and select $\alpha=0.5$ in \eqref{eq:tl_complexity}
and perform only one fit on an Intel i9-10885H 2.40GHz machine. 
}
\end{figure}

{Figures \ref{fig:MSE-time-of-nonpara-models-GP} and \ref{fig:Scaling_ssize} in the main text and Figure~\ref{fig:Scaling_ssize_d2} in Appendix~\ref{sec:increased2}} compare model performances {and training times} using \citet{sun2023tensor}'s synthetic data. GPST models excel in testing MSE, outperforming tensor-GP models in fitting and generalization. However, TT models show lower training but higher testing MSE, indicating overfitting; pruning improves CP/Tucker leaf models' MSE, making them nearly comparable to GPST but with less depth pruning compared to mean models. Although GPST has longer training and prediction times, unpruned TT models are faster, with even pruned versions having lower computational times than tensor GP models, as shown by Proposition~\ref{prop:(Computational-complexity-for}.
With pruning, our TT models
take longer to compute, but all TT models with depth $\leq3$
are still much faster to compute than the tensor GP models, echoing our Proposition~\ref{prop:(Computational-complexity-for}.

\subsection{More Applications}
\label{subsec:Tensor-on-tensor:-Image-Recovery}

\paragraph{EEG Data}

Our first example uses the EEG dataset from  
in the $\mathtt{TRES}$ R package \citep{zeng2021tres}. We use
$61\times64\times64$ tensor inputs of $n=61,d_{1}=d_{2}=64$ parsed
from EEG dataset and a binary {label}  $\bm{y}$. {Using TRR ($\bm{y}$ as vector predictor) and TPR ($\bm{X}$ as tensor predictor) from this package, we illustrate the fitted coefficients from TRS models with different fitting methods (i.e., 1D/FG/OLS/PLS).}

Analogous to Figure~\ref{fig:Comparison-of-different-depths}, the {fourth and fifth} rows in the figure in Appendix~\ref{sec:EEGcoeff} each consist of eight panels; each panel represents a {$64\times 64$} coefficient corresponding to a low-rank
CP or Tucker model at leaf nodes, which are based on a subset (induced by low-rank splitted trees with target rank $R=5$) along
the first mode of the input tensor $\bm{X}$. 
The eight panels in the {fourth and fifth} rows appear to overlap
and collectively approximate the first two {panels (1D and FG coefficients) in the first (TRES.TRR) and second (TRES.TPR) rows.}
This can be attributed to several factors. One is the redundancy in the
subsets along the first dimension; if these subsets either overlap
significantly or are highly correlated, the resulting reconstructions
in each panel would naturally appear similar, otherwise dis-similar. 
{We find that the TT coefficients with sufficient depth are more similar to the TPR coefficients, seeing as both models take tensors as inputs, than they are to the TRR coefficients, where TRR takes tensors as the response. In particular, the 3,5,6-th leaf node models' coefficients seem to 
capture the TRES.TPR (1D, FG, PLS) coefficient patterns.}
Another explanation could be the complementarity of the features captured
in each panel. If each panel captures unique but complementary structural
or informational elements of the original tensor, their collective
representation would naturally approximate the entire tensor effectively; {this may explain why the deeper layer models seem like a decomposition of the depth 0 model}.
Lastly, the inherent nature of CP and Tucker decomposition as a low-rank
approximation implies that even a subset of the first dimension's
factor vectors, when integrated with the corresponding vectors from
other modes, can capture significant global features of the original
tensor $\bm{X}$.

\paragraph{Facial Image}

Our second example uses frontalized facial images as in \citet{lock2018tensor}
showing forward-facing faces achieved by rotating, scaling, and cropping
original images. The highly aligned images facilitate appearance
comparisons. Each image has $90\times90$ pixels,
with each pixel containing red, green, and blue channel intensities.
We randomly
sample 500 images, so the predictor tensor $\bm{X}$ has dimensions
$500\times90\times90\times3$, and the response tensor $\bm{Y}$ has
dimensions $500\times72$. We center each image tensor by subtracting
the mean of all image tensors. Another randomly sampled 500 images
and response are used as a validation set.

\citet{lock2018tensor} found the rrr method with $\lambda=10^{5}$ achieved an RPE of 0.568, ranging from 0.5 to 0.9, but fairly comparing regularizations is complex. At each leaf $t$, introducing a unique regularization term $\lambda(t)$ is challenging due to varying sample sizes across leaf nodes. Hence, all models in our study use no regularization ($\lambda_{TT}=0$), simplifying comparisons. TTentrywise and TTlowrank perform reasonably (see Table~\ref{tab:The-tensor-input-decision}), but they do not match the tensor-on-tensor rrr and Bayesian rrrBayes models \citep{lock2018tensor}. TT offers an alternative for tensor-on-tensor regression without complex prior specifications or MCMC computations \citep{gahrooei2021multiple,luo2022tensor,wang2024bayesian}, but managing the high-dimensional input and computational demands for 720 different tree models in TTentrywise, each requiring intricate splitting and pruning, remains a significant challenge.

\begin{table}
\centering

{\small{}{}}%
\begin{tabular}{ccccccccc}
\toprule 
 & \multicolumn{2}{c}{{\small{}{}{}TTentrywise\_CP}} & \multicolumn{2}{c}{{\small{}{}{}TTentrywise\_Tucker}} & \multicolumn{2}{c}{{\small{}{}{}TTlowrank\_CP}} & \multicolumn{2}{c}{{\small{}{}{}TTlowrank\_Tucker}}\tabularnewline
\midrule 
{\small{}{}{}rank }  & {\small{}{}{}RMSE }  & {\small{}{}{}RPE }  & {\small{}{}{}RMSE }  & {\small{}{}{}RPE }  & {\small{}{}{}RMSE }  & {\small{}{}{}RPE }  & {\small{}{}{}RMSE }  & {\small{}{}{}RPE}\tabularnewline
\midrule
\midrule 
{\small{}{}{}3 }  & {\small{}{}{}9.0007}  & {\small{}{}{}0.7083}  & {\small{}{}{}9.0218}  & {\small{}{}{}0.7150}  & {\small{}{}{}8.8316}  & {\small{}{}{}0.6566}  & {\small{}{}{}11.7500}  & {\small{}{}{}2.0573}\tabularnewline
\midrule 
{\small{}{}{}5 }  & {\small{}{}{}9.0523}  & {\small{}{}{}0.7247}  & {\small{}{}{}9.0729}  & {\small{}{}{}0.7313}  & {\small{}{}{}8.8761}  & {\small{}{}{}0.6699}  & {\small{}{}{}12.072}  & {\small{}{}{}2.2924}\tabularnewline
\midrule 
{\small{}{}{}15 }  & {\small{}{}{}9.0582}  & {\small{}{}{}0.7266 }  & {\small{}{}{}9.1162 }  & {\small{}{}{}0.7454}  & {\small{}{}{}8.9741 }  & {\small{}{}{}0.7000 }  & {\small{}{}{}10.5614 }  & {\small{}{}{}1.3447}\tabularnewline
\bottomrule
\end{tabular}

$ $
\newline

\caption{\label{tab:The-tensor-input-decision}
RMSE and RPE from TT models trained on a resampled Facial Image dataset \citep{lock2018tensor} 
with CP/Tucker methods and TTentrywise/TTlowrank approaches. 
%
}
\end{table}

\section{\label{sec:Conclusion}Conclusion}

Our paper contributes to scalar-on-tensor and tensor-on-tensor regressions using tree-based models (Algorithm \ref{alg:Tensor-input-decision-exhaustive}), especially designed for high-dimensional tensor data. Unlike traditional regression trees which handle vector inputs, our models are tailored for multi-way array inputs, significantly enhancing their utility for complex data types.

We developed both randomized and deterministic algorithms to efficiently fit these models (Algorithm \ref{alg:Tensor-input-decision-LS} and \ref{alg:Tensor-input-decision-BB}), crucial due to the high computational demand of tensor operations. These models perform competitively against established tensor-input Gaussian Process models \citep{sun2023tensor, yu2018tensor} and are computationally more efficient, utilizing techniques like leverage score sampling and branch-and-bound optimization.

{There are other types of tensor products, such as the Khatri-Rao products,
which offer granular feature spaces and more detailed splitting rules} (Section
\ref{subsec:Splitting-Criterion}) analogous to the oblique trees \citep{murthy1994system,breiman1984classification}. {Additionally, kernel-based approaches
are well-suited for capturing complex, non-linear relationships and
could enhance the predictive power for mixed inputs} \citep{hrluo_2022c}. {However, tensor product inequalities are not necessarily computationally tractable, and efficient formats like tensor-train are needed for practical implementations} \citep{oseledets2011tensor}.

Further, we expand our methods to tensor-on-tensor regressions through tree ensembles, addressing a gap in existing non-parametric models that can process both tensor inputs and outputs. This includes implementing new tensor-specific splitting criteria like variance splitting and low-rank approximation, and exploring potential applications such as tensor compression \citep{kielstra2023tensor}.

While we demonstrate the effectiveness of these models and their ensembling methods, theoretical aspects like the consistency of dynamic trees and error bounds for ensemble tensor models remain open for further research. {Our tree-based method could be extended to be robust to rotations}\citep{van2018learning} {via rotation forests}~\citep{rodriguez2006rotation}, {which remains to be explored in tensor data setting}. Our work paves the way for future studies to explore these areas within the framework of complex data analysis. 
\section*{Acknowledgment}
HL thanks LBNL for providing computational resources in pilot experiments.
HL was supported by U.S. Department of Energy under Contract DE-AC02-05CH11231
and U.S. National Science Foundation NSF-DMS 2412403. LM and AH were partly supported by NSF grants DMS-1749789 and NIGMS grant R01-GM135440. LM was also partly supported by NSF grant DMS-2152999.
\newline 
{We thank the anonymous referees and the Associate Editor for their constructive feedback.}
\singlespacing
\bibliographystyle{chicago}
\bibliography{fastTensorTree_refs}

\begin{thebibliography}{}

\bibitem[\protect\citeauthoryear{Amit and Geman}{Amit and
  Geman}{1997}]{amit1997shape}
Amit, Y. and D.~Geman (1997).
\newblock Shape quantization and recognition with randomized trees.
\newblock {\em Neural computation\/}~{\em 9\/}(7), 1545--1588.

\bibitem[\protect\citeauthoryear{Athey, Tibshirani, and Wager}{Athey
  et~al.}{2019}]{athey2019generalized}
Athey, S., J.~Tibshirani, and S.~Wager (2019).
\newblock Generalized random forests.

\bibitem[\protect\citeauthoryear{Bi, Tang, Yuan, Zhang, and Qu}{Bi
  et~al.}{2021}]{bi2021tensors}
Bi, X., X.~Tang, Y.~Yuan, Y.~Zhang, and A.~Qu (2021).
\newblock Tensors in statistics.
\newblock {\em Annual review of statistics and its application\/}~{\em 8},
  345--368.

\bibitem[\protect\citeauthoryear{Breiman}{Breiman}{1999}]{breiman1999prediction}
Breiman, L. (1999).
\newblock Prediction games and arcing algorithms.
\newblock {\em Neural computation\/}~{\em 11\/}(7), 1493--1517.

\bibitem[\protect\citeauthoryear{Breiman}{Breiman}{2001}]{Breiman01}
Breiman, L. (2001).
\newblock Random forests.
\newblock {\em Machine Learning\/}~{\em 45}, 5--32.

\bibitem[\protect\citeauthoryear{Breiman et~al.}{Breiman
  et~al.}{1984}]{breiman1984classification}
Breiman, L. et~al. (1984).
\newblock Classification and regression trees.

\bibitem[\protect\citeauthoryear{Bult{\'e} and S{\o}rensen}{Bult{\'e} and
  S{\o}rensen}{2024}]{bulte2024medoid}
Bult{\'e}, M. and H.~S{\o}rensen (2024).
\newblock Medoid splits for efficient random forests in metric spaces.
\newblock {\em Computational Statistics \& Data Analysis\/}, 107995.

\bibitem[\protect\citeauthoryear{Capitaine, Bigot, Thi{\'e}baut, and
  Genuer}{Capitaine et~al.}{2024}]{capitaine2024frechet}
Capitaine, L., J.~Bigot, R.~Thi{\'e}baut, and R.~Genuer (2024).
\newblock Fr{\'e}chet random forests for metric space valued regression with
  non euclidean predictors.
\newblock {\em Journal of Machine Learning Research\/}~{\em 25\/}(355), 1--41.

\bibitem[\protect\citeauthoryear{Chaudhuri, Huang, Loh, and Yao}{Chaudhuri
  et~al.}{1994}]{chaudhuri1994piecewise}
Chaudhuri, P., M.-C. Huang, W.-Y. Loh, and R.~Yao (1994).
\newblock Piecewise-polynomial regression trees.
\newblock {\em Statistica Sinica\/}, 143--167.

\bibitem[\protect\citeauthoryear{Chen and Guestrin}{Chen and
  Guestrin}{2016}]{Chen16}
Chen, T. and C.~Guestrin (2016).
\newblock Xgboost: A scalable tree boosting system.
\newblock In {\em Proceedings of the 22Nd ACM SIGKDD International Conference
  on Knowledge Discovery and Data Mining}, KDD '16, New York, NY, USA, pp.\
  785--794. ACM.

\bibitem[\protect\citeauthoryear{Chickering, Meek, and Rounthwaite}{Chickering
  et~al.}{2001}]{chickering2001efficient}
Chickering, D.~M., C.~Meek, and R.~Rounthwaite (2001).
\newblock Efficient determination of dynamic split points in a decision tree.
\newblock In {\em Proceedings 2001 IEEE international conference on data
  mining}, pp.\  91--98. IEEE.

\bibitem[\protect\citeauthoryear{Chipman, George, and McCulloch}{Chipman
  et~al.}{1998}]{chipman1998bayesian}
Chipman, H., E.~I. George, and R.~E. McCulloch (1998).
\newblock Bayesian cart model search.
\newblock {\em Journal of the American Statistical Association\/}~{\em
  93\/}(443), 935--948.

\bibitem[\protect\citeauthoryear{Chipman, Ranjan, and Wang}{Chipman
  et~al.}{2012}]{Chipman12}
Chipman, H., P.~Ranjan, and W.~Wang (2012).
\newblock Sequential design for computer experiments with a flexible {Bayesian}
  additive model.
\newblock {\em Canadian Journal of Statistics\/}~{\em 40\/}(4), 663--678.

\bibitem[\protect\citeauthoryear{Chipman, George, and McCulloch}{Chipman
  et~al.}{2010}]{chipman2010bart}
Chipman, H.~A., E.~I. George, and R.~E. McCulloch (2010).
\newblock Bart: Bayesian additive regression trees.
\newblock {\em The Annals of Applied Statistics\/}~{\em 4\/}(1), 266--298.

\bibitem[\protect\citeauthoryear{Denison, Smith, and Mallick}{Denison
  et~al.}{1998}]{denison1998bayesian}
Denison, D., A.~Smith, and B.~Mallick (1998).
\newblock Bayesian cart model search: Comment.
\newblock {\em Journal of the American Statistical Association\/}~{\em
  93\/}(443), 954--957.

\bibitem[\protect\citeauthoryear{Freund and Schapire}{Freund and
  Schapire}{1997}]{Freund97}
Freund, Y. and R.~E. Schapire (1997).
\newblock A decision-theoretic generalization of on-line learning and an
  application to boosting.
\newblock {\em Journal of computer and system sciences\/}~{\em 55}, 119--139.

\bibitem[\protect\citeauthoryear{Friedman, Hastie, Rosset, Tibshirani, and
  Zhu}{Friedman et~al.}{2004}]{friedman2004discussion}
Friedman, J., T.~Hastie, S.~Rosset, R.~Tibshirani, and J.~Zhu (2004).
\newblock Discussion of boosting papers.
\newblock {\em Annual Statistics\/}~{\em 32}, 102--107.

\bibitem[\protect\citeauthoryear{Friedman}{Friedman}{2001}]{Friedman01}
Friedman, J.~H. (2001).
\newblock Greedy function approximation: A gradient boosting machine.
\newblock {\em The Annals of Statistics\/}~{\em 19}, 1189--1232.

\bibitem[\protect\citeauthoryear{Gahrooei, Yan, Paynabar, and Shi}{Gahrooei
  et~al.}{2021}]{gahrooei2021multiple}
Gahrooei, M.~R., H.~Yan, K.~Paynabar, and J.~Shi (2021).
\newblock Multiple tensor-on-tensor regression: An approach for modeling
  processes with heterogeneous sources of data.
\newblock {\em Technometrics\/}~{\em 63\/}(2), 147--159.

\bibitem[\protect\citeauthoryear{Geurts, Ernst, and Wehenkel}{Geurts
  et~al.}{2006}]{geurts2006extremely}
Geurts, P., D.~Ernst, and L.~Wehenkel (2006).
\newblock Extremely randomized trees.
\newblock {\em Machine learning\/}~{\em 63}, 3--42.

\bibitem[\protect\citeauthoryear{Gordon and Olshen}{Gordon and
  Olshen}{1978}]{gordon1978asymptotically}
Gordon, L. and R.~A. Olshen (1978).
\newblock Asymptotically efficient solutions to the classification problem.
\newblock {\em The Annals of Statistics\/}, 515--533.

\bibitem[\protect\citeauthoryear{Guhaniyogi, Qamar, and Dunson}{Guhaniyogi
  et~al.}{2017}]{guhaniyogi2017bayesian}
Guhaniyogi, R., S.~Qamar, and D.~B. Dunson (2017).
\newblock Bayesian tensor regression.
\newblock {\em Journal of Machine Learning Research\/}~{\em 18\/}(79), 1--31.

\bibitem[\protect\citeauthoryear{Guo, Kotsia, and Patras}{Guo
  et~al.}{2011}]{guo2011tensor}
Guo, W., I.~Kotsia, and I.~Patras (2011).
\newblock Tensor learning for regression.
\newblock {\em IEEE Transactions on Image Processing\/}~{\em 21\/}(2),
  816--827.

\bibitem[\protect\citeauthoryear{Hastie, Tibshirani, Friedman, and
  Friedman}{Hastie et~al.}{2009}]{hastie2009elements}
Hastie, T., R.~Tibshirani, J.~H. Friedman, and J.~H. Friedman (2009).
\newblock {\em The elements of statistical learning: data mining, inference,
  and prediction}, Volume~2.
\newblock Springer.

\bibitem[\protect\citeauthoryear{Johndrow, Bhattacharya, and Dunson}{Johndrow
  et~al.}{2017}]{johndrow2017tensor}
Johndrow, J.~E., A.~Bhattacharya, and D.~B. Dunson (2017).
\newblock Tensor decompositions and sparse log-linear models.
\newblock {\em Annals of statistics\/}~{\em 45\/}(1), 1.

\bibitem[\protect\citeauthoryear{Kielstra, Shi, Luo, Qian, and Liu}{Kielstra
  et~al.}{2024}]{kielstra2023tensor}
Kielstra, P.~M., T.~Shi, H.~Luo, J.~Qian, and Y.~Liu (2024+).
\newblock A linear-complexity tensor butterfly algorithm for compressing
  high-dimensional oscillatory integral operators.
\newblock {\em ongoing\/}, 1--39.

\bibitem[\protect\citeauthoryear{Klusowski}{Klusowski}{2020}]{klusowski2020sparse}
Klusowski, J. (2020).
\newblock Sparse learning with cart.
\newblock {\em Advances in Neural Information Processing Systems\/}~{\em 33},
  11612--11622.

\bibitem[\protect\citeauthoryear{Klusowski and Tian}{Klusowski and
  Tian}{2024}]{klusowski2021universal}
Klusowski, J.~M. and P.~M. Tian (2024).
\newblock Large scale prediction with decision trees.
\newblock {\em Journal of the American Statistical Association\/}~{\em
  119\/}(545), 525--537.

\bibitem[\protect\citeauthoryear{Kolda and Bader}{Kolda and
  Bader}{2009}]{kolda2009tensor}
Kolda, T.~G. and B.~W. Bader (2009).
\newblock Tensor decompositions and applications.
\newblock {\em SIAM review\/}~{\em 51\/}(3), 455--500.

\bibitem[\protect\citeauthoryear{Kolda and Sun}{Kolda and
  Sun}{2008}]{kolda2008scalable}
Kolda, T.~G. and J.~Sun (2008).
\newblock Scalable tensor decompositions for multi-aspect data mining.
\newblock In {\em 2008 Eighth IEEE international conference on data mining},
  pp.\  363--372. IEEE.

\bibitem[\protect\citeauthoryear{Kossaifi, Panagakis, Anandkumar, and
  Pantic}{Kossaifi et~al.}{2019}]{JMLR:v20:18-277}
Kossaifi, J., Y.~Panagakis, A.~Anandkumar, and M.~Pantic (2019).
\newblock Tensorly: Tensor learning in python.
\newblock {\em Journal of Machine Learning Research\/}~{\em 20\/}(26), 1--6.

\bibitem[\protect\citeauthoryear{Krawczyk}{Krawczyk}{2021}]{krawczyk2021tensor}
Krawczyk, B. (2021).
\newblock Tensor decision trees for continual learning from drifting data
  streams.
\newblock {\em Machine Learning\/}~{\em 110\/}(11-12), 3015--3035.

\bibitem[\protect\citeauthoryear{Lawler and Wood}{Lawler and
  Wood}{1966}]{lawler1966branch}
Lawler, E.~L. and D.~E. Wood (1966).
\newblock Branch-and-bound methods: A survey.
\newblock {\em Operations research\/}~{\em 14\/}(4), 699--719.

\bibitem[\protect\citeauthoryear{Li, Lue, and Chen}{Li
  et~al.}{2000}]{li2000interactive}
Li, K.-C., H.-H. Lue, and C.-H. Chen (2000).
\newblock Interactive tree-structured regression via principal hessian
  directions.
\newblock {\em Journal of the American Statistical Association\/}~{\em
  95\/}(450), 547--560.

\bibitem[\protect\citeauthoryear{Li and Zhang}{Li and
  Zhang}{2017}]{li2017parsimonious}
Li, L. and X.~Zhang (2017).
\newblock Parsimonious tensor response regression.
\newblock {\em Journal of the American Statistical Association\/}~{\em
  112\/}(519), 1131--1146.

\bibitem[\protect\citeauthoryear{Li, Xu, Zhou, and Li}{Li
  et~al.}{2018}]{li2018tucker}
Li, X., D.~Xu, H.~Zhou, and L.~Li (2018).
\newblock Tucker tensor regression and neuroimaging analysis.
\newblock {\em Statistics in Biosciences\/}~{\em 10}, 520--545.

\bibitem[\protect\citeauthoryear{Liu}{Liu}{2017}]{liu2017low}
Liu, Y. (2017).
\newblock Low-rank tensor regression: Scalability and applications.
\newblock In {\em 2017 IEEE 7th International Workshop on Computational
  Advances in Multi-Sensor Adaptive Processing (CAMSAP)}, pp.\  1--5. IEEE.

\bibitem[\protect\citeauthoryear{Lock}{Lock}{2018}]{lock2018tensor}
Lock, E.~F. (2018).
\newblock Tensor-on-tensor regression.
\newblock {\em Journal of Computational and Graphical Statistics\/}~{\em
  27\/}(3), 638--647.

\bibitem[\protect\citeauthoryear{Luo, Cho, Demmel, Li, and Liu}{Luo
  et~al.}{2024}]{hrluo_2022c}
Luo, H., Y.~Cho, J.~W. Demmel, X.~Li, and Y.~Liu (2024).
\newblock {Hybrid Parameter Search and Dynamic Model Selection for
  Mixed-Variable Bayesian Optimization}.
\newblock {\em Journal of Computational and Graphical Statistics\/}~{\em
  0\/}(0), 1--14.

\bibitem[\protect\citeauthoryear{Luo and Ma}{Luo and Ma}{2025}]{LLM2023}
Luo, H. and A.~Ma (2025).
\newblock Tensor randomized frontal sketching method to large-scale linear
  systems.
\newblock {\em Numerical Mathematics: Theory, Methods and Applications\/},
  1--30, to appear.

\bibitem[\protect\citeauthoryear{Luo, Nattino, and Pratola}{Luo
  et~al.}{2022}]{hrluo_2019a}
Luo, H., G.~Nattino, and M.~T. Pratola (2022).
\newblock {Sparse Additive Gaussian Process Regression}.
\newblock {\em Journal of Machine Learning Research\/}~{\em 23\/}(61), 1--34.

\bibitem[\protect\citeauthoryear{Luo and Pratola}{Luo and
  Pratola}{2022}]{hrluo_2022e}
Luo, H. and M.~T. Pratola (2022).
\newblock {Sharded Bayesian Additive Regression Trees}.
\newblock {\em arXiv:2306.00361\/}, 1--46.

\bibitem[\protect\citeauthoryear{Luo and Zhang}{Luo and
  Zhang}{2022}]{luo2022tensor}
Luo, Y. and A.~R. Zhang (2022).
\newblock Tensor-on-tensor regression: Riemannian optimization,
  over-parameterization, statistical-computational gap, and their interplay.
\newblock {\em arXiv preprint arXiv:2206.08756\/}.

\bibitem[\protect\citeauthoryear{Malik and Becker}{Malik and
  Becker}{2018}]{malik2018low}
Malik, O.~A. and S.~Becker (2018).
\newblock Low-rank tucker decomposition of large tensors using tensorsketch.
\newblock {\em Advances in neural information processing systems\/}~{\em 31}.

\bibitem[\protect\citeauthoryear{Mentch and Zhou}{Mentch and
  Zhou}{2020}]{mentch2020randomization}
Mentch, L. and S.~Zhou (2020).
\newblock Randomization as regularization: A degrees of freedom explanation for
  random forest success.
\newblock {\em The Journal of Machine Learning Research\/}~{\em 21\/}(1),
  6918--6953.

\bibitem[\protect\citeauthoryear{Minster, Viviano, Liu, and Ballard}{Minster
  et~al.}{2023}]{minster2023cp}
Minster, R., I.~Viviano, X.~Liu, and G.~Ballard (2023).
\newblock Cp decomposition for tensors via alternating least squares with qr
  decomposition.
\newblock {\em Numerical Linear Algebra with Applications\/}~{\em 30\/}(6),
  e2511.

\bibitem[\protect\citeauthoryear{Morrison, Jacobson, Sauppe, and
  Sewell}{Morrison et~al.}{2016}]{morrison2016branch}
Morrison, D.~R., S.~H. Jacobson, J.~J. Sauppe, and E.~C. Sewell (2016).
\newblock Branch-and-bound algorithms: A survey of recent advances in
  searching, branching, and pruning.
\newblock {\em Discrete Optimization\/}~{\em 19}, 79--102.

\bibitem[\protect\citeauthoryear{Murray, Demmel, Mahoney, Erichson,
  Melnichenko, Malik, Grigori, Derezi\'{n}ski, Liang, Luo, and Dongarra}{Murray
  et~al.}{2023}]{murray2023randomized}
Murray, R., J.~Demmel, M.~W. Mahoney, N.~B. Erichson, M.~Melnichenko, O.~A.
  Malik, L.~Grigori, M.~a. M. E.~L. Derezi\'{n}ski, T.~Liang, H.~Luo, and J.~J.
  Dongarra (2023).
\newblock Randomized numerical linear algebra: A perspective on the field with
  an eye to software.
\newblock {\em arXiv preprint arXiv:2302.11474\/}.

\bibitem[\protect\citeauthoryear{Murthy, Kasif, and Salzberg}{Murthy
  et~al.}{1994}]{murthy1994system}
Murthy, S.~K., S.~Kasif, and S.~Salzberg (1994).
\newblock A system for induction of oblique decision trees.
\newblock {\em Journal of artificial intelligence research\/}~{\em 2}, 1--32.

\bibitem[\protect\citeauthoryear{Oh, Park, Lee, and Kang}{Oh
  et~al.}{2018}]{oh2018scalable}
Oh, S., N.~Park, S.~Lee, and U.~Kang (2018).
\newblock Scalable tucker factorization for sparse tensors-algorithms and
  discoveries.
\newblock In {\em 2018 IEEE 34th International Conference on Data Engineering
  (ICDE)}, pp.\  1120--1131. IEEE.

\bibitem[\protect\citeauthoryear{Oseledets}{Oseledets}{2011}]{oseledets2011tensor}
Oseledets, I.~V. (2011).
\newblock Tensor-train decomposition.
\newblock {\em SIAM Journal on Scientific Computing\/}~{\em 33\/}(5),
  2295--2317.

\bibitem[\protect\citeauthoryear{Papadogeorgou, Zhang, and
  Dunson}{Papadogeorgou et~al.}{2021}]{papadogeorgou2021soft}
Papadogeorgou, G., Z.~Zhang, and D.~B. Dunson (2021).
\newblock Soft tensor regression.
\newblock {\em Journal of Machine Learning Research\/}~{\em 22\/}(219), 1--53.

\bibitem[\protect\citeauthoryear{Pratola}{Pratola}{2016}]{Pratola16}
Pratola, M.~T. (2016).
\newblock Efficient metropolis-hastings proposal mechanisms for bayesian
  regression tree models.
\newblock {\em Bayesian Analysis\/}~{\em 11}, 885--911.

\bibitem[\protect\citeauthoryear{Pratola, Chipman, Gattiker, Higdon, McCulloch,
  and Rust}{Pratola et~al.}{2014}]{Pratola14}
Pratola, M.~T., H.~A. Chipman, J.~R. Gattiker, D.~M. Higdon, R.~McCulloch, and
  W.~N. Rust (2014).
\newblock Parallel {Bayesian} additive regression trees.
\newblock {\em Journal of Computational and Graphical Statistics\/}~{\em 23},
  830--852.

\bibitem[\protect\citeauthoryear{Prokhorenkova, Gusev, Vorobev, Dorogush, and
  Gulin}{Prokhorenkova et~al.}{2018}]{prokhorenkova2018catboost}
Prokhorenkova, L., G.~Gusev, A.~Vorobev, A.~V. Dorogush, and A.~Gulin (2018).
\newblock Catboost: unbiased boosting with categorical features.
\newblock {\em Advances in neural information processing systems\/}~{\em 31}.

\bibitem[\protect\citeauthoryear{Qiu, Yu, and Zhu}{Qiu
  et~al.}{2024}]{qiu2024random}
Qiu, R., Z.~Yu, and R.~Zhu (2024).
\newblock Random forest weighted local fr{\'e}chet regression with random
  objects.
\newblock {\em Journal of Machine Learning Research\/}~{\em 25\/}(107), 1--69.

\bibitem[\protect\citeauthoryear{Rodriguez, Kuncheva, and Alonso}{Rodriguez
  et~al.}{2006}]{rodriguez2006rotation}
Rodriguez, J.~J., L.~I. Kuncheva, and C.~J. Alonso (2006).
\newblock Rotation forest: A new classifier ensemble method.
\newblock {\em IEEE transactions on pattern analysis and machine
  intelligence\/}~{\em 28\/}(10), 1619--1630.

\bibitem[\protect\citeauthoryear{Sun, Manchester, Jin, Liu, and Chen}{Sun
  et~al.}{2023}]{sun2023tensor}
Sun, H., W.~Manchester, M.~Jin, Y.~Liu, and Y.~Chen (2023, 23--29 Jul).
\newblock Tensor {G}aussian process with contraction for multi-channel imaging
  analysis.
\newblock ~{\em 202}, 32913--32935.

\bibitem[\protect\citeauthoryear{van~der Wilk, Bauer, John, and
  Hensman}{van~der Wilk et~al.}{2018}]{van2018learning}
van~der Wilk, M., M.~Bauer, S.~John, and J.~Hensman (2018).
\newblock Learning invariances using the marginal likelihood.
\newblock {\em Advances in Neural Information Processing Systems\/}~{\em 31}.

\bibitem[\protect\citeauthoryear{Wang and Xu}{Wang and
  Xu}{2024}]{wang2024bayesian}
Wang, K. and Y.~Xu (2024).
\newblock Bayesian tensor-on-tensor regression with efficient computation.
\newblock {\em Statistics and its interface\/}~{\em 17\/}(2), 199.

\bibitem[\protect\citeauthoryear{Yu, Li, and Liu}{Yu
  et~al.}{2018}]{yu2018tensor}
Yu, R., G.~Li, and Y.~Liu (2018).
\newblock Tensor regression meets gaussian processes.
\newblock In {\em International Conference on Artificial Intelligence and
  Statistics}, pp.\  482--490. PMLR.

\bibitem[\protect\citeauthoryear{Zeng, Wang, and Zhang}{Zeng
  et~al.}{2021}]{zeng2021tres}
Zeng, J., W.~Wang, and X.~Zhang (2021).
\newblock Tres: An r package for tensor regression and envelope algorithms.
\newblock {\em Journal of Statistical Software\/}~{\em 99}, 1--31.

\bibitem[\protect\citeauthoryear{Zhou, Li, and Zhu}{Zhou
  et~al.}{2013}]{zhou2013tensor}
Zhou, H., L.~Li, and H.~Zhu (2013).
\newblock Tensor regression with applications in neuroimaging data analysis.
\newblock {\em Journal of the American Statistical Association\/}~{\em
  108\/}(502), 540--552.

\end{thebibliography}

\newpage
\renewcommand\thesection{\Alph{section}}
\renewcommand\thesubsection{\thesection.\arabic{subsection}}
\setcounter{section}{0}

\begin{center}
{\Large\textbf{SUPPLEMENTARY MATERIAL}}
\end{center}
\onehalfspacing
\section{Performance Metrics}

The performance of a tensor-input predictive model $\hat{\bm{y}}$ can be quantified
using the classical regression metric Mean Square Error (MSE) $\|\bm{y}- \hat{\bm{y}}\|_{2}^{2}$
(or root MSE). We also consider the Relative Prediction Error
(RPE) \citep{lock2018tensor} defined as ${\rm RPE}=\|\bm{y}-\hat{\bm{y}}\|_{F}^{2} / \|\bm{y}\|_{F}^{2}$
on testing data (a.k.a. SMSPE \citep{gahrooei2021multiple}), where
$\|\cdot\|_{F}$ denotes the Frobenius norm on the vectorized
form of a tensor. Unlike MSE, RPE is a \textit{normalized} discrepancy between
the predicted and actual values. 
To supplement these prediction error metrics, 
we also provide complexity analysis with actual time benchmarks. 
Many existing tensor models \citep{sun2023tensor} including the low-rank models \citep{liu2017low}
have been shown to be bottlenecked by scalability, and our tree models and ensemble variant
provide a natural divide-and-conquer approach addressing this.
\section{Methods of efficient search for splitting coordinates.}\label{sec:LS_BB_methods}

Existing efficient methods \citep{hrluo_2022e,chickering2001efficient}
for splitting point optimization (i.e., search among possible points
for the above loss functions) usually rely on the distribution assumptions
on the candidates. In our case, this assumption is usually not realistic
due to the complex correlation between dimensions of the input tensor.
To reduce the $N^{*}\cdot\left(\max_{i}d_{i}\right)^{K}$ factor further,
we propose two methods for shrinking the search space of these optimization problems in order 
to quickly solve the minimization problems for \eqref{eq:low-rank LAE}
or \eqref{eq:low-rank-reg LRE}. Both methods have existed in the optimization
community for a while and exhibit trade-offs between efficiency and
accuracy, but to our best knowledge are applied in
regression trees for the first time.

\begin{table}
\begin{adjustbox}{center}
{\smaller
\begin{tabular}{ccccc}
\toprule 
\backslashbox{Item}{Model}  & Decision Trees  & Extra-Trees (ERT)  & Algorithm \ref{alg:Tensor-input-decision-LS}  & Algorithm \ref{alg:Tensor-input-decision-BB}\tabularnewline
\midrule
\midrule 
Reference  & \citep{breiman1984classification}  & \citep{geurts2006extremely}  & \multicolumn{2}{c}{This paper}\tabularnewline
\midrule 
split coordinate  & Exact optimization  & Exact optimization  & Importance sampling  & Branch-and-bound\tabularnewline
\midrule 
\multirow{2}{*}{split value} & \multirow{2}{*}{Exact optimization} & \multirow{2}{*}{Random selection} & Random $\tau<1$/  & Random $\xi>0$/\tabularnewline
 &  &  & Exact $\tau=1$  & Exact $\xi=0$\tabularnewline
\cmidrule{1-5} \cmidrule{1-5} 
\end{tabular}
}
\end{adjustbox}
\caption{\label{tab:Comparison of different tree methods}Comparison of state-of-the-art
methods, where we essentially replace the splitting value and coordinate
optimization $\min_{(j_{0},j_{1},j_{2})}\mathcal{L}(j_{0},j_{1},j_{2})$
with a surrogate problem and support approximate choices of splitting
values to allow more expressive complexity trade-offs.}
\end{table}

Leverage score sampling (LS) uses a sample rate $\tau$
to shrink the search space size from $d_{1}d_{2}$ down to $\tau d_{1}d_{2}$
for constrained optimization (see Algorithm \ref{alg:Tensor-input-decision-LS}). When $\tau=1$, LS reduces  to exhaustive
search. The subset $\mathcal{D}$ of dimension pairs $(j_{1},j_{2})$
is chosen from $\{1,2,...,d_{1}\}\times\{1,2,...,d_{2}\}$ following
a LS (without replacement) scheme \citep{murray2023randomized,malik2018low}.
Pairs with higher variance are preferred since they offer diversity
and may benefit low-rank approximations for $R_{1},R_{2}$. 
The probability of choosing a constant-value dimension is zero. 
LS keeps the same asymptotic computational complexity but sacrifices
optimal solutions for faster computation during the greedy search for
splits. This per-split strategy also shares the spirit of the traditional
$\mathtt{mtry}$-strategy used in random forests \citep{mentch2020randomization}
which searches $\mathtt{mtry}$ (rather than all) features. 

Branch-and-bound optimization (BB) considers a divide-and-conquer
strategy for this optimization problem. In Algorithm \ref{alg:Tensor-input-decision-BB},
we introduce the tolerance $\xi$ to divide the search
space via the branch-and-bound strategy \citep{lawler1966branch,morrison2016branch}
and consider the optimization of \eqref{eq:low-rank LAE} or \eqref{eq:low-rank-reg LRE}
within the parameter tolerance $\xi$. When $\xi=0$, this method reduces
to regular exhaustive search.
BB divides the entire space of $d_{1}d_{2}$ potential splits into
smaller search subspaces (i.e., branching). For each subspace, it
estimates a bound on the quality of the best split coordinates (minimizing
the chosen criterion) that can be found. If this estimated bound
is worse than the best split (within the prespecified tolerance
$\xi$) found so far, the subspace is excluded from further
consideration (i.e., bounding). This method efficiently narrows
the search to the most relevant splits %
and provides
a practical alternative to exhaustive search at the price of possible
local minima. 
To support our claim, Figure \ref{fig:Comparison-of-different-SMs}
shows that in the zero-noise cases, both LS and BB tend to have larger MSE compared to exhaustive search, regardless
of the choice of leaf node fitting models, and the model quality deterioration becomes severe when
the maximal depth increases.
When there is more noise
(i.e., lower signal-to-noise ratio in data), such differences in MSE
between exhaustive search and BB/LS become less
obvious. As rank of the leaf model increases,
the performance of LS is less stable compared to BB. As the depth
of the model increase, BB takes more time to fit compared to the linearly
increasing sampling time of LS. Table~\ref{tab:Comparison of different tree methods}
summarizes the difference between these fast search tree methods.

On one hand, our LS sampling (Algorithm~\ref{alg:Tensor-input-decision-LS})
\textit{randomly} chooses split coordinates but
not the splits. It also serves as a surrogate importance
measure for each coordinate combination. The LS design of choosing
the optimal split among a fraction of all possible splits is related
to that of extra-trees \citep{geurts2006extremely}. Extra-trees differ
from classic decision trees in that instead of exhaustively
optimizing for the optimal split, random splits are drawn using randomly
selected features and the sub-optimal split among those candidates
is chosen. When the candidate possible features is 1, this builds
a totally random decision tree.

On the other hand, our BB (Algorithm \ref{alg:Tensor-input-decision-BB})
\textit{deterministically} selects those split coordinates
with the best approximation error at a certain tolerance level. BB exhibits
a trade-off between the accuracy of the best pair $(j_{1},j_{2})$
minimizing splitting criterion and search complexity but its quality
will be affected by the additional tolerance hyper-parameter, which
is not always straightforward to tune. 
BB also generalizes the classical constraint of limiting the maximal
number of features used in all possible splits: limiting the number
of features restricts the possible partitions induced by
the tree, whereas BB also ensures that the ``resolution'' of tree-induced
partition will not exceed a certain tolerance.
\newpage
\section{Algorithms for Fitting Tensor Trees}

\begin{algorithm}

\SetAlgoLined \KwResult{Train a vector-input single decision tree
regressor and predict with it}

\textbf{Function} fit($\bm{X}\in\mathbb{R}^{n\times d}$, $y\in\mathbb{R}^{n}$):\
\Begin{ Initialize root node

\For{node in tree}{ Calculate for all potential splits along each
column $j_{1}\in\{1,2,...,d\}$: Find $\arg\min_{j_{0}\in\{1,\cdots,n\},j_{1}\in\{1,\cdots,d\}}SSE(j_{0},j_{1})$\;
With the minimizer $(j_{0}^{*},j_{1}^{*})$ we split the dataset on the chosen column $j_{1}^{*}$ and splitting value $\bm{X}[j_{0}^{*},j_{1}^{*}]$, creating
left and right child nodes using data in $R_{1},R_{2}$\; }
Fit the chosen mean/CP/Tucker models at each of the leaf nodes.
}


\caption{\label{alg:Vector-input-decision}Fitting vector-input decision tree
regressor with SSE criterion \eqref{eq:SSE}.}
\end{algorithm}

\begin{algorithm}[h!]

\SetAlgoLined \KwResult{Train a tensor input single decision tree
regressor with leverage score sampling }

\textbf{Function} fit($\bm{X}\in\mathbb{R}^{n\times d_{1}\times d_{2}}$,
$\bm{y}\in\mathbb{R}^{n}$):

\Begin{ Initialize root node and sample\_rate $\tau\in(0,1]$ (when
$\tau=1$ there is no subsampling).\ Compute the variance matrix
$\bm{V}=\left(\bm{X}[:,i,j]\right)_{i,j=1}^{d_{1},d_{2}}$\; \For{node
in tree}{ Take a random subset $D\subset\{1,2,...,d_{1}\}\times\{1,2,...,d_{2}\}$
of cardinality $\tau d_{1}d_{2}$ such that the probability of selecting
dimension pair $(j_{1},j_{2})$ is proportional to $\bm{V}_{j_{1},j_{2}}$\;
Calculate for all potential splits along each dimension pair $(j_{1},j_{2})\in D$:
Find $\arg\min_{j_{0}\in\{1,\cdots,n\},(j_{1},j_{2})\in D}LAE(j_{0},j_{1},j_{2})$
or $LRE(j_{0},j_{1},j_{2})$\; 
With the minimizer $(j_{0}^{*},j_{1}^{*},j_{2}^{*})$ we split the dataset on
the chosen dimension pair $(j_{1}^{*},j_{2}^{*})$ and splitting value $\bm{X}[j_{0}^{*},j_{1}^{*},j_{2}^{*}]$, creating left and right
child nodes using data in $R_{1},R_{2}$\; } 
Fit the chosen mean/CP/Tucker models at each of the leaf nodes.
}


\caption{\label{alg:Tensor-input-decision-LS}\textbf{Method 1: Leverage score
sampling for} fitting tensor input decision tree regressor with low-rank
splitting criteria \eqref{eq:low-rank LAE} and \eqref{eq:low-rank-reg LRE}.}
\end{algorithm}

\begin{algorithm}[!t]

\SetAlgoLined \KwResult{Train a tensor input single decision tree
regressor with branch-and-bound optimization}

\textbf{Function} fit($\bm{X}\in\mathbb{R}^{n\times d_{1}\times d_{2}}$,
$\bm{y}\in\mathbb{R}^{n}$):

\Begin{ Initialize root node\; \For{node in tree}{ Take the
full set $D=\{1,2,...,d_{1}\}\times\{1,2,...,d_{2}\}$ of cardinality
$d_{1}d_{2}$\; Calculate for all potential splits along each dimension
pair $(j_{1},j_{2})\in D$: Use $best\_idx$ as $(j_{1},j_{2})$ from
the return of BnB\_minimize($LAE$,$D$,$\xi$) or BnB\_minimize($LRE$,$D$,$\xi$)\;
With the minimizer $(j_{0}^{*},j_{1}^{*},j_{2}^{*})$ we split the dataset on the chosen dimension pair $(j_{1}^{*},j_{2}^{*})$ and splitting value $\bm{X}[j_{0}^{*},j_{1}^{*},j_{2}^{*}]$,
creating left and right child nodes using data in $R_{1},R_{2}$\;
} 
Fit the chosen mean/CP/Tucker models at each of the leaf nodes.
}

\textbf{Function} BnB\_minimize($obj$,$D$,$\xi$):

\Begin{ Set $best\_obj$ to infinity; Set initial bounds $D$ for
tensor dimensions and append $D$ to queue\;

\While{queue is not empty}{ $D_{c}$ = queue.pop(0)\;

mid\_feature\_index = tuple((b{[}0{]} + b{[}1{]}) // 2 \textbf{for}
b \textbf{in} current $D_{c}$)\;Calculate $obj$ at mid\_feature\_index\;
\textbf{if} $obj<best\_obj$ \textbf{then} Update $best\_obj$, $best\_idx$\;
\For{i in range(len($D_{c}$))}{ \If{$D_{c}${[}i{]}{[}1{]}
- $D_{c}${[}i{]}{[}0{]} \textgreater{} $\text{tolerance}$ $\xi$}{
mid\_point = ($D_{c}${[}i{]}{[}0{]} + $D_{c}${[}i{]}{[}1{]}) //
2\; left\_bounds = right\_bounds = $D_{c}$\; left\_bounds{[}i{]}
= ($D_{c}${[}i{]}{[}0{]}, mid\_point)\; right\_bounds{[}i{]} = (mid\_point
+ 1, current\_bounds{[}i{]}{[}1{]})\; queue.append(left\_bounds,right\_bounds)\;
break\; } } } Return $best\_obj$, $best\_idx$ \; }


\caption{\label{alg:Tensor-input-decision-BB}\textbf{Method 2: Branch-and-bound
optimization for }fitting tensor input decision tree regressor with
low-rank splitting criteria \eqref{eq:low-rank LAE} and \eqref{eq:low-rank-reg LRE}.}
\end{algorithm}
\newpage
\section{\label{subsec:Pruning-exp} Pruning Experiment}

Here we show the effect of pruning using the tensor-input scalar-output synthetic function 
\begin{align}
    f(\bm{X})=
    \begin{cases}
     5 & \text{if }\bm{X}[:,0,1,0] \geq 0.4\\
    -1 & \text{if }\bm{X}[:,0,1,0] < 0.4 \text{ and } \bm{X}[:,2,2,0] \geq 0.65\\
    -4 & \text{otherwise}
    \label{eq:prune_fn}
    \end{cases},
\end{align}
which can be represented by a tree with three leaves. %
As shown in Figure~\ref{fig:prune}, %
we fit trees using the variance splitting criterion and various pruning parameter values $\alpha$.
First we observe that for all shown $\alpha$ values that the MSE decreases with tree's maximal depth, and that this decrease levels off once the maximal depth is two or larger. 
We also see that trees without pruning ($\alpha=0$) can grow to have many leaves, but once pruning is introduced ($\alpha>0$), trees with max\_depth $\geq 2$ are typically pruned to have three or four leaves, especially at $\alpha=0.1$ (i.e., the largest tested $\alpha$ value).
Hence, in this example, an appropriately large $\alpha$ value induces sufficient pruning for trees to have the ``correct'' number of leaves.

\begin{figure}[h]
\centering
\includegraphics[width=0.85\textwidth]{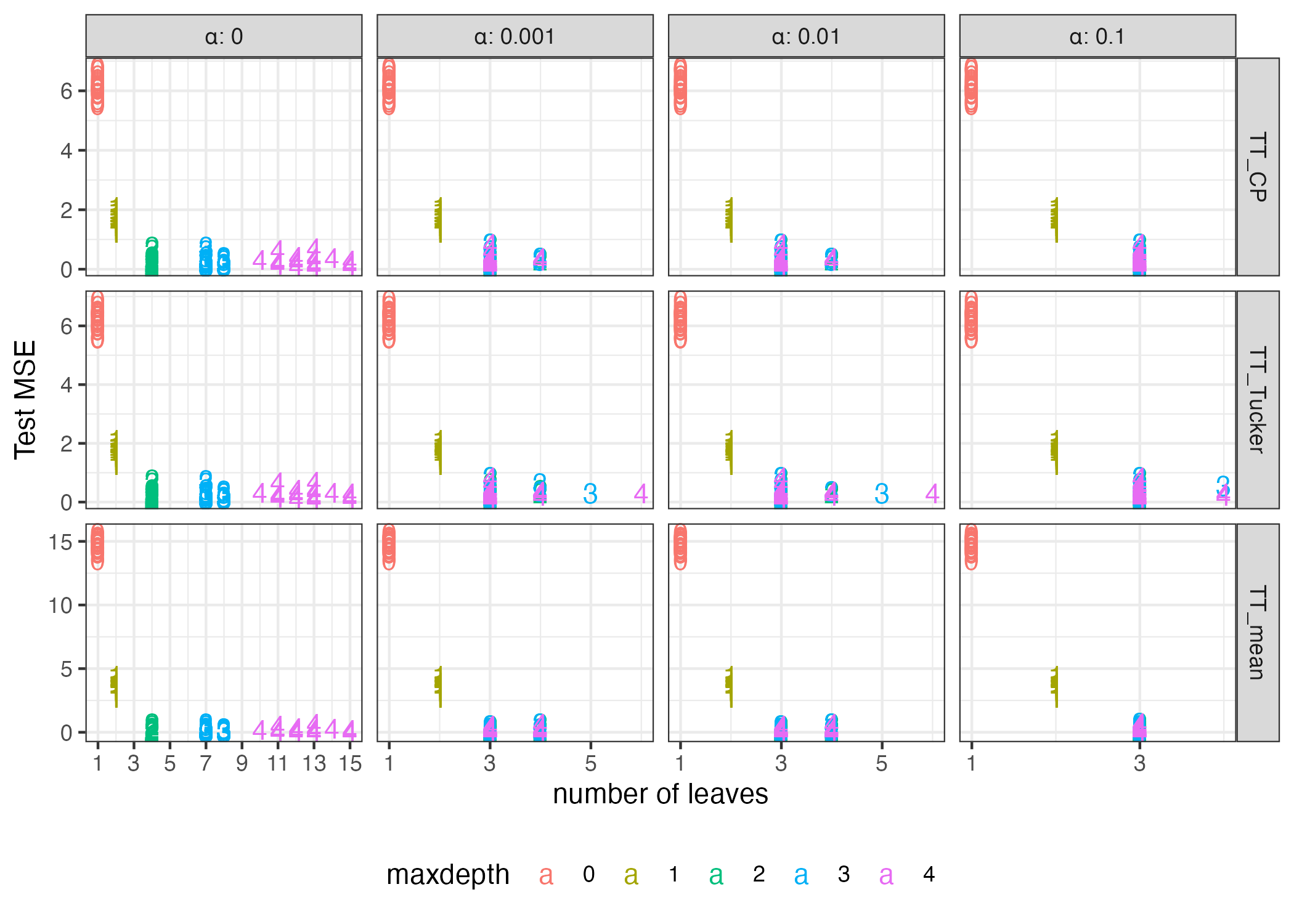}
\caption{\label{fig:prune}Out-of-sample MSE of trees trained on data generated by 
first sampling $(500,4,4,4)$ input tensor $\bm{X}$ uniformly randomly from $[0,1]$
and then evaluating the function \eqref{eq:prune_fn}  %
at the input tensor before adding i.i.d.\ Gaussian noise with mean zero and variance $0.1$. 
The figure shows 20 replicates for each combination of $\alpha$, {\small{}{}{}
max\_depth}, and leaf model.}
\end{figure}

\section{\label{subsec:Random-Forest-and} Boosted Trees}
Boosting is an ensemble learning technique that enhances model performance by combining multiple weak learners, usually decision trees, into a strong learner. This method sequentially trains base models, with each new model correcting the errors of the previous ones, thus reducing bias and variance. Pioneering boosting algorithms include AdaBoost \citep{Freund97}, which adjusts weights of misclassified instances, and Gradient Boosting \citep{Friedman01,friedman2004discussion}, which fits new models to residual errors.

Before describing boosted trees in detail, we first introduce a different tree ensemble method called Random Forests (RFs).
Random Forests \citep{Breiman01,amit1997shape} create an ensemble of independent decision trees using bootstrap samples and random feature subsets, enhancing diversity and robustness through randomization over the feature space. For a given set of $n$ data pairs $(\bm{X}_{i},y_{i})$, where $\bm{X}_{i}$ are input tensors and $y_{i}$ are scalar responses, a RF builds $B$ trees. Each tree is trained on a bootstrap sample of the data, and at each split, a random subset of features is considered as shown in Algorithm~\ref{alg:Random-Forest-with}. 
This approach reduces the variance of the model by averaging the predictions of independently trained trees, which helps mitigate overfitting. The main challenge of ensembling trees is ensuring that the trees are diverse enough to capture different patterns in the data without making inconsistent predictions. 

Although RFs combined with TT (See Algorithm \ref{alg:Random-Forest-with}) are powerful, its performance on predicting entries in a tensor output is not ideal. We focus on boosting methods, which iteratively improve the model by focusing on previously mispredicted instances. Boosting constructs trees sequentially, where each new tree corrects errors from the previous ones, leading to a strong aggregated model. Therefore, the concept behind boosting diverges fundamentally from RF. Boosting does not aggregate weak learners; instead, it averages the predictions of strong learners.  This method is particularly effective for our tensor regression tree models, as detailed in the following sections.
Gradient Boosting \citep{Friedman01} constructs an ensemble of decision trees sequentially, where each new tree is trained to correct the errors of the existing ensemble. Unlike RFs, where trees are built independently, GB involves a functional gradient descent approach to minimize a specified loss function \citep{hastie2009elements}. 

In the general case, gradient boosting involves computing ``pseudo-residuals,'' which are the negative gradients of the loss function with respect to the model's predictions. For a given loss function $L(y, F(\bm{X}))$, the pseudo-residuals at iteration $b$ are given by:
\[
r_{i}^{(b)} = -\left[ \frac{\partial L(y_{i}, F(\bm{X}_{i}))}{\partial F(\bm{X}_{i})} \right]_{F(\bm{X}) = F_{\text{GB}}^{(b-1)}(\bm{X})}.
\]
Each new tree $g_{b}$ is then fit to these pseudo-residuals, and the model is updated as follows:

\[
F_{\text{GB}}^{(b)}(\bm{X}) = F_{\text{GB}}^{(b-1)}(\bm{X}) + \eta g_{b}(\bm{X};T_{b},M_{b}),
\]
where $\eta$ is a learning rate that controls the contribution of each tree.
The new tree minimizes the loss function:
\[
\text{argmin}_{T_{b},M_{b}} \sum_{i=1}^{n} L\left(y_{i}, F_{\text{GB}}^{(b-1)}(\bm{X}_{i}) + g_{b}(\bm{X}_{i};T_{b},M_{b})\right).
\]
In our case, we focus on a simpler version of tree boosting algorithm, which uses the squared error loss (for the scalar response $y$). This approach iteratively fits trees to the residuals of the previous trees without explicitly computing pseudo-residuals.
For a set of $n$ data pairs $(\bm{X}_{i}, y_{i})$, one popular tree boosting model constructs an additive series of $m$ regression trees:
\begin{align}
F_{\text{GB}}(\bm{X}) = \sum_{b=1}^{m} g_{b}(\bm{X};T_{b},M_{b}),
\end{align}
where each tree $g_{b}(\bm{X};T_{b},M_{b})$ is fit to the residuals of the previous ensemble:
\begin{align}
r_{i}^{(b)} = y_{i} - F_{\text{GB}}^{(b-1)}(\bm{X}_{i}).
\end{align}
Each new tree is trained to minimize the squared error loss on these residuals:
\[
\text{argmin}_{T_{b},M_{b}} \sum_{i=1}^{n} \left(r_{i}^{(b)} - g_{b}(\bm{X}_{i};T_{b},M_{b})\right)^2.
\]
By sequentially fitting trees to the residuals, the forward-stagewise fitting approach builds a strong aggregated model, effectively capturing the complex patterns in the data while mitigating overfitting. This method is particularly suitable for our tensor regression tree models.

In addition, Bayesian regression trees \citep{chipman1998bayesian,denison1998bayesian,chipman2010bart,Chipman12} extend this approach by placing a prior
distribution on the space of possible trees and their hyper-parameters
like splitting coordinates and values, hence providing uncertainty
quantification over predictions. Posterior distributions of the trees
are then explored and updated using Markov chain Monte Carlo (MCMC)
methods. The Bayesian approach allows for full posterior inference,
including point and interval estimates of the unknown regression function
and the effects of potential predictors \citep{Pratola14, Pratola16,hrluo_2022e}. Bayesian Additive Regression Trees (BART), introduced by \citet{chipman2010bart}, combine Bayesian methods with boosting principles, modeling responses as a sum of regression trees and providing a probabilistic framework with uncertainty quantification.

\section{Proof of Proposition \ref{prop:(Computational-complexity-for}}\label{sec:Proof-of-Complexity}
\begin{proof}
Since each decomposition when constructing the tree structure takes no more than $N^*<\infty$ iterations, it suffices to consider per iteration complexity as stated in Lemmas \ref{lem:(CP-ALS-per-iteration} and \ref{lem:(Tucker-ALS-per-iteration}.

(1) The computational complexity at any node $t$ with $n_t$ samples is $\mathcal{O}(n_t d_1d_2)$ for finding the best split. 
In a balanced binary tree with $n$ samples in $\mathbb{R}^{d_1\times d_2}$, $i$-th level has $2^i$ nodes, each processing $\mathcal{O}(n2^{-i})$ samples. The complexity per node is then $\mathcal{O}\left(n2^{-i} \cdot d_1d_2\right)$. 
Hence, the total complexity for level $i$ is $2^i \cdot \mathcal{O}\left(n 2^{-i}\cdot d_1d_2\right) = \mathcal{O}(nd_1d_2)$. Summing over all $K$ levels, the overall complexity is $\mathcal{O}(n \cdot d_1d_2 \cdot \log k)$.

(2) Now the total complexity for level $i$ is $2^i \cdot \mathcal{O}\left(d_{1}d_{2}n2^{-i} \right)\cdot\mathcal{O}\left(d_{1}d_{2}n2^{-i} \right)$. The first factor $\left(d_{1}d_{2}n2^{-i} \right)$ comes from the search of all samples at each coordinate combination in each node at this level. The second factor $\left(d_{1}d_{2}n2^{-i} \right)$ comes from setting $R=n2^{-i}$ and $K=2$ in Lemma \ref{lem:(CP-ALS-per-iteration} since the maximum rank of tensor decomposition cannot exceed the number of sample sizes. Summing over all $\log k$ levels, the overall complexity is $\mathcal{O}(n^2d_1^2 d_2^2 \cdot \log k)$. 

(3) Now the total complexity for level $i$ is $$2^i \cdot \mathcal{O}\left(d_{1}d_{2}n2^{-i} \right) \cdot \mathcal{O}\left( n2^{-i}d_1 d_2\cdot \min( d_1 d_2, n2^{-i}) + n2^{-i}d_1 d_2\cdot\min(n2^{-i} d_1, n2^{-i} d_2, d_2, d_1) \right).$$ The first factor $\left(d_{1}d_{2}n2^{-i} \right)$ comes from the search of all samples at each coordinate combination in each node at this level. The second factor come from setting $R=n2^{-i}$ and $K=2$ in Lemma \ref{lem:(Tucker-ALS-per-iteration}. Summing over all $\log k$ levels, the overall complexity is   $$\mathcal{O} \left( \log k \cdot n^2d_1^2 d_2^2 \cdot \left( \min( d_1 d_2, n) + \min(n d_1, n d_2, d_2, d_1) \right) \right).$$ Then we can simplify the term $\min(n d_1, n d_2, d_2, d_1)$ to $\min(d_2, d_1)$ since we know that $n \geq 1$ by assumption.
\end{proof}

\section{Proof of Lemma \ref{lem:(Tucker-ALS-per-iteration}}
\begin{proof}
The argument is as follows, since the complexity bottleneck per iteration
comes from the SVD step: 
\begin{align*}
 & \mathcal{O}\left(\min\left\{n\prod_{j=1}^{K}d_{j}^{'2},n^{2}\prod_{j=1}^{K}d_{j}^{'}\right\}+\sum_{i=1}^{K}\min\left\{R^{2}\cdot d_{i}\prod_{j\neq i}d_{j}^{'2},R\cdot d_{i}^{2}\prod_{j\neq i}d_{j}^{'}\right\}\right)\\
\asymp & \mathcal{O}\left(\min\left\{n\prod_{j=1}^{K}d_{j}^{'2},n^{2}\prod_{j=1}^{K}d_{j}^{'}\right\}+\sum_{i=1}^{K}\min\left\{R^{2}\cdot\frac{d_{i}}{d_{i}^{'2}}\prod_{j=1}^{K}d_{j}^{'2},R\cdot\frac{d_{i}^{2}}{d_{i}^{'}}\prod_{j=1}^{K}d_{j}^{'}\right\}\right)\\
\asymp & \mathcal{O}\left(\min\left\{n\prod_{j=1}^{K}d_{j}^{2},n^{2}\prod_{j=1}^{K}d_{j}\right\}+\sum_{i=1}^{K}\min\left\{R^{2}\cdot\frac{1}{d_{i}}\prod_{j=1}^{K}d_{j}^{2},R\cdot d_{i}\prod_{j=1}^{K}d_{j}\right\}\right)\text{ since }d_{i}^{'}\asymp d_{i},\\
\lesssim & \mathcal{O}\left(n\prod_{j=1}^{K}d_{j}^{2}+\sum_{i=1}^{K}R\cdot d_{i}\prod_{j=1}^{K}d_{j}\right)
\end{align*}
Because $R<n$, in the preceding line the first summand dominates the second summand.
We conclude the proof by noting that $\prod_{i=1}^{K}d_{i}^{2} \leq (\max_{i}d_{i})^{2K}$.
\end{proof}

\section{\label{sec:Proof-of-Theorem-1oracle}Proof of Theorem \ref{thm:(Empirical-bounds-for}}
\begin{proof}
The proof is almost identical to the proof of Theorem 4.2 in \citet{klusowski2021universal}
since when minimizing \eqref{eq:sse-tensor} and splitting at each
interior node we indeed ``flatten'' the input tensor $\bm{X}$ along
its second and third coordinates. We consider the maximum impurity
gain $IG(t)$ for $N_{t}$ samples in node $t$ to decrease the
empirical risk for the regressor $g$ 
\begin{align*}
\hat{\mathcal{R}}_{t}(g)=\frac{1}{N_{t}}\sum_{\bm{X}^{(n)}[i,:,:]\in t}\mathcal{L}(y_{i},g(\bm{X}^{(n)}[i,:,:])) & ,
\end{align*}
for the chosen loss function $\mathcal{L}$ (e.g., \eqref{eq:sse-tensor})
in the parent node $t$.
When we split node $t$ into left and right nodes 
\begin{align*}
t_{L} & \coloneqq\left\{ \bm{X}\in t\colon\bm{X}[:,j_{1},j_{2}]\leq s\right\}, N_{t_{L}}=\#\{\bm{X}^{(n)}[i,:,:] \in t_{L}\}, \\ t_{R} & \coloneqq\left\{ \bm{X}\in t\colon\bm{X}[:,j_{1},j_{2}]>s\right\}, N_{t_{R}}=\#\{\bm{X}^{(n)}[i,:,:] \in t_{R}\}
\end{align*}
we have an analog to Lemma A.1 in \cite{klusowski2021universal} for the $\ell_{2}$ inner product
between multivariate functions $u$ and $v$. At node $t$, this
inner product is evaluated only at those observed locations $\bm{X}^{(n)}[i,:,:]\in t$:
\begin{align*}
\left\langle u,v\right\rangle _{t}  \coloneqq\frac{1}{N_{t}}\cdot\sum_{\bm{X}^{(n)}[i,:,:]\in t}u(\bm{X}^{(n)}[i,:,:])\cdot v(\bm{X}^{(n)}[i,:,:]) & \\=\frac{1}{N_{t}}\cdot\sum_{\bm{X}^{(n)}[i,:,:]\in t}\sum_{j,k}u(\bm{X}[i,j,k])\cdot v(\bm{X}[i,j,k]) \text{ for \eqref{eq:sse-tensor}}.
\end{align*}
The key idea in proving the impurity gain lower bound is to use a linear
interpolator over the domain $\mathbb{R}$ indexed by $\mathbb{N}\times\mathbb{N}$,
for the ordered statistics $\bm{X}[(1),j_{1},j_{2}]\leq\cdots\leq\bm{X}[(n),j_{1},j_{2}]$,
to control the empirical risk. To attain this goal, we define
the probability measure $\Pi(j_{1},j_{2},s)$ for splitting point
$s$ and variables $\bm{X}[:,j_{1},j_{2}]$ where its RN derivative
with respect to the counting measure on $\mathbb{N}\times\mathbb{N}$
and the Lebesgue measure on $\mathbb{R}$ (this is still defined on
$\mathbb{R}$ not $\mathbb{R}^{d_{1}\times d_{2}}$ since the split
still happens by one threshold $s$) can be written as 
\[
\frac{d\Pi(j_{1},j_{2},s)}{d(j_{1},j_{2})\times ds}\coloneqq\frac{\left|Dg_{j_{1},j_{2}}(s)\right|\sqrt{N_{t_{L}}N_{t_{R}}/N_{t}^{2}}}{\sum_{i_{1}=1}^{d_{1}}\sum_{i_{2}=1}^{d_{2}}\int_{\mathbb{R}}\left|Dg_{i_{1},i_{2}}(s')\right|ds'}
\]
\begin{eqnarray*}
Dg_{j_{1},j_{2}}(s) & = & \begin{cases}
\frac{g_{j_{1},j_{2}}\left(\bm{X}[(i+1),j_{1},j_{2}]\right)-g_{j_{1},j_{2}}\left(\bm{X}[(i),j_{1},j_{2}]\right)}{\bm{X}[(i+1),j_{1},j_{2}]-\bm{X}[(i),j_{1},j_{2}]} & \bm{X}[(i),j_{1},j_{2}]<s<\bm{X}[(i+1),j_{1},j_{2}]\\
0 & \text{otherwise}
\end{cases}.
\end{eqnarray*}
Using these bi-indexed linear interpolators $Dg_{j_{1},j_{2}}(s)$
as piece-wise density functions, %
we obtain an impurity gain formula for a tensor-input tree when splitting with the SSE criterion and fitting with mean leaf models.

\textbf{Analog of Lemma 4.1 (Tensor Input Case):}
Let \( g \in \mathcal{G}_1 \) and \( K \geq 1 \) be any depth. For any terminal node \( t \) of the tree \( T_{K-1} \) such that \( \hat{\mathcal{R}}_{t}(g_{T_{K-1}}) > \hat{\mathcal{R}}_{t}(g) \), we have
\[
IG(t) \geq \frac{\left(\hat{\mathcal{R}}_{t}(\hat{g}_{T_{K-1}}) - \hat{\mathcal{R}}_{t}(g)\right)^2}{V^2(g)}
\]
where \( V(g)=\|g\|_{TV} \) is a complexity constant dependent on \( g \). %

Next, we proceed with the proof of Theorem \ref{thm:(Empirical-bounds-for}, which is an analog of Theorem 4.2 of \citet{klusowski2021universal}. Since the output is still scalar, we use recursion for empirical risk reduction and get the total risk
   \[
   \hat{\mathcal{R}}(g_{T_{K}}) = \hat{\mathcal{R}}(g_{T_{K-1}}) - \sum_{t \in T_{K-1}} \frac{N_{t}}{N} IG(t)
   \]
   where \( t \in T_{K-1} \) means that \( t \) is a terminal node of tree \( T_{K-1} \).
By the above tensor analog of Lemma 4.1, the total impurity gain over all terminal nodes \( t \in T_{K-1} \) such that \( \hat{\mathcal{R}}_{t}(g_{T_{K-1}}) > \hat{\mathcal{R}}_{t}(g) \) is bounded by
   \[
   \sum_{t \in T_{K-1}\colon\hat{\mathcal{R}}_{t}(g_{T_{K-1}}) > \hat{\mathcal{R}}_{t}(g)} \frac{N_{t}}{N} IG(t) \geq 
   \sum_{t \in T_{K-1}\colon\hat{\mathcal{R}}_{t}(g_{T_{K-1}}) > \hat{\mathcal{R}}_{t}(g)} \frac{N_{t}}{N} \frac{\left(\hat{\mathcal{R}}_{t}(\hat{g}_{T_{K-1}}) - \hat{\mathcal{R}}_{t}(g)\right)^2}{V^2(g)}.
   \]
   Applying Jensen's inequality to the sum, we get
   \begin{align*}
   \sum_{t \in T_{K-1}\colon\hat{\mathcal{R}}_{t}(g_{T_{K-1}}) > \hat{\mathcal{R}}_{t}(g)} \frac{N_{t}}{N} \left(\hat{\mathcal{R}}_{t}(\hat{g}_{T_{K-1}}) - \hat{\mathcal{R}}_{t}(g)\right)^2 
   &\geq \left( \sum_{t \in T_{K-1}\colon\hat{\mathcal{R}}_{t}(g_{T_{K-1}}) > \hat{\mathcal{R}}_{t}(g)} \frac{N_{t}}{N} \left(\hat{\mathcal{R}}_{t}(\hat{g}_{T_{K-1}}) - \hat{\mathcal{R}}_{t}(g)\right) \right)^2 \\
   &\geq \left( \sum_{t \in T_{K-1}} \frac{N_{t}}{N} \left(\hat{\mathcal{R}}_{t}(\hat{g}_{T_{K-1}}) - \hat{\mathcal{R}}_{t}(g)\right) \right)^2.
   \end{align*}
   The global excess risk can be defined as
   \[
   E_{K} := \hat{\mathcal{R}}(g_{T_{K}}) - \hat{\mathcal{R}}(g).
   \]
   We can rewrite the total impurity gain as
   \[
   \sum_{t \in T_{K-1}} \frac{N_{t}}{N} \left(\hat{\mathcal{R}}_{t}(\hat{g}_{T_{K-1}}) - \hat{\mathcal{R}}_{t}(g)\right) = E_{K-1}.
   \]

Combining the results from steps above, we have
   \[
   E_{K} \leq E_{K-1} \left(1 - \frac{E_{K-1}}{V^2(g)}\right).
   \]
   By induction, we iterate this inequality to obtain
   \[
   E_{K} \leq \frac{V^2(g)}{K + 3}.
   \]
   Finally, substituting back into the empirical risk, we obtain
   \[
   \hat{\mathcal{R}}(g_{T_{K}}) \leq \hat{\mathcal{R}}(g) + \frac{V^2(g)}{K + 3}.
   \]
   Taking the infimum over all functions \( g \), we get 
   \[
   \hat{\mathcal{R}}(g_{T_{K}}) \leq \inf_{g\in{\mathcal{G}_1}} \left\{ \hat{\mathcal{R}}(g) + \frac{V^2(g)}{K + 3} \right\}.
   \]

Hence, we have proved an analog of Theorem 4.2 for the tensor input case, where the only difference is the statement and definitions used in the analog of Lemma 4.1. \end{proof}

\section{Algorithms for Ensembles of Trees}

\begin{table}[h!]
\centering
{\smaller
\begin{tabular}{cccc}
\toprule 
\backslashbox{Leaf model $m_{j}$}{Split criteria}  & variance  & clustering  & low-rank\tabularnewline
\midrule
\midrule 
Sample average of the $y$  & max\_depth  & max\_depth, cluster\_method  & split\_rank, max\_depth\tabularnewline
\midrule 
\multirow{3}{*}{CP regression on $y$} & max\_depth,  & max\_depth,  & split\_rank,\tabularnewline
 & CP\_reg\_rank  & cluster\_method,  & max\_depth,\tabularnewline
 &  & CP\_reg\_rank  & CP\_reg\_rank\tabularnewline
\midrule 
\multirow{3}{*}{Tucker regression on $y$} & max\_depth,  & max\_depth,  & split\_rank,\tabularnewline
 & Tucker\_reg\_rank  & cluster\_method,  & max\_depth,\tabularnewline
 &  & Tucker\_reg\_rank  & Tucker\_reg\_rank\tabularnewline
\cmidrule{1-4} %
\end{tabular}
}
$ $%

\caption{\label{tab:Supported-combination-methods:}Supported combination methods:
TensorDecisionTreeRegressor and associated parameters, and how it
generalizes existing models.}
\end{table}

\begin{algorithm}[h]

\SetAlgoLined \KwResult{Forest of decision trees $\{T_{t}\}$}
\textbf{Input:} Tensor data $\bm{X}\in\mathbb{R}^{n\times p_{1}\times p_{2}\times\ldots\times p_{d}}$,
Target values $y\in\mathbb{R}^{n}$, Number of trees $T$\; \For{$t\leftarrow1$
\KwTo $T$}{ $(X_{t},y_{t})\leftarrow\text{Bootstrap sample from }(X,y)$\;
$T_{t}\leftarrow\text{Train a decision tree on }(X_{t},y_{t})$ with random subsampling of predictors at root node\;
Store all observations in each leaf of $T_{t}$\; 
Fit the chosen mean/CP/Tucker models at each of the leaf nodes of $T_{t}$\;
}
$\bm{y}_{i}^{pred}=\sum_{t=1}^{T}F_{t}(\bm{X}^{(n)}[i,:,:])$\;

\caption{\label{alg:Random-Forest-with}Random Forest with Tensor Input}
\end{algorithm}

\begin{algorithm}[t!]

\SetAlgoLined \KwResult{Train a generalized boosting regressor
model and predict with it}

\textbf{Function} GeneralizedBoostingRegressor($\bm{X}\in\mathbb{R}^{n\times d_{1}\times d_{2}}$,
$\bm{y}\in\mathbb{R}^{n}$, $M$, $\eta$, $p_{resample}$):

\Begin{ Initialize $F_{0}(x)$ to the mean of $\bm{y}$\;

\For{$m=1$ \KwTo $M$}{ Compute residuals: $r_{im}=y_{i}-F_{m-1}(\bm{X}^{(n)}[i,:,:])$
for all $i$\;

\eIf{$p_{resample}>0$}{ Sample indices according to weights:
$\bm{I}_{resample}\sim\text{Sample}(\text{Indices},\text{Weights},\text{Size}=n\times p_{resample})$\;
$\bm{h}_{m}$ = fit($\bm{X}[\bm{I}_{resample},:,:]$, $r_{im}[\bm{I}_{resample}]$)\;
Update weights: $\text{Weights}=\text{Weights}\times\exp(\left|r_{im}\right|)$\;
Normalize weights: $\text{Weights}=\text{Weights}/\sum\text{Weights}$\;
}

\For{each sample $i$}{ $\bm{y}_{i}^{pred}=F_{m-1}(\bm{X}^{(n)}[i,:,:])+\eta\bm{h}_{m}(\bm{X}^{(n)}[i,:,:])$\;
} Update $F_{m}(x)=F_{m-1}(x)+\eta\bm{h}_{m}(x)$\; } \Return{$F_{M}(x)$}\;
}


\caption{\label{alg:Gradient-Boosted-Tensor}Generalized Boosting Regressor
with Optional AdaBoost-Like Resampling}
\end{algorithm}

\newpage
\section{Training time for fixed $N$, increasing $d_2$}
\label{sec:increased2}

Here we perform an experiment to illustrate how the training time of our and other tensor methods increases as tensor dimension $d_2$ increases under a fixed value of the sample size $N=100$ and tensor dimension $d_1=25$.
The data is generated in the same way as in the experiments in Figure~\ref{fig:MSE-time-of-nonpara-models-GP} of the main text.

Figure~\ref{fig:Scaling_ssize_d2} shows that as tensor dimension $d_2$ increases, the GP-based methods (TensorGP and GPST) have the largest increase and total training time, followed by the TT methods, then by CP, then by Tucker.
(All TT models with the same max\_depth have roughly the same training time, so for each depth we show only TT\_CP without pruning. Also, we show only max\_depth=3,6 to reduce visual clutter.)
These results complement and are similar in spirit to the findings from our analysis of Figures \ref{fig:MSE-time-of-nonpara-models-GP} and \ref{fig:Scaling_ssize} in the main text.

\begin{figure}[H]
\centering

\includegraphics[width=\textwidth]{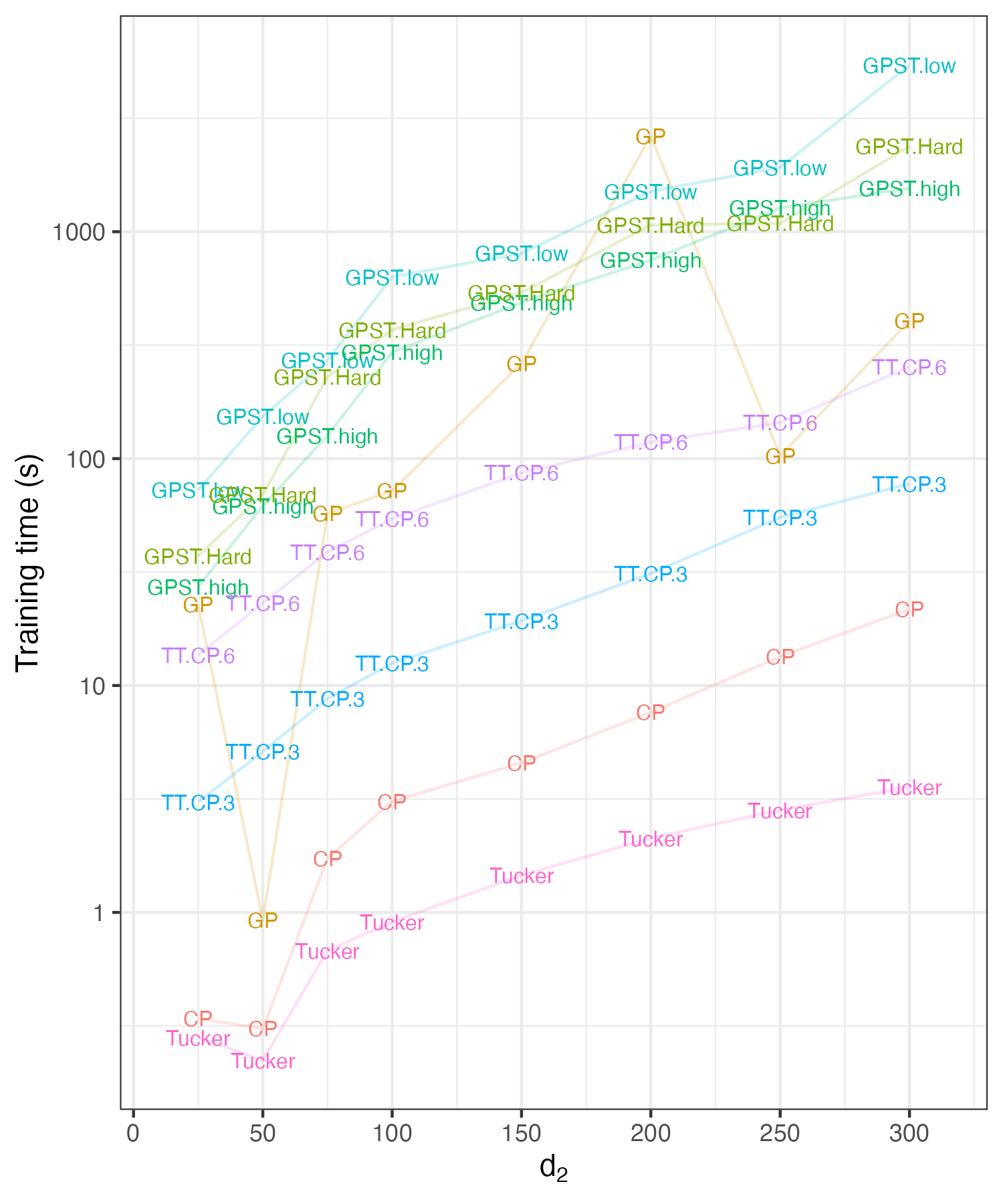}

\caption{\label{fig:Scaling_ssize_d2}  In-sample and out-of-sample training
time (seconds) against the second tensor dimension $d_2$ for the experiments
in Figure~\ref{fig:MSE-time-of-nonpara-models-GP}. 
Here the sample size is $N=100$ and the first tensor dimension is $d_1=25$. For TT models, we experiment with max\_depth=3,6 and select $\alpha=0.5$ in \eqref{eq:tl_complexity}
and perform only one fit on an AMD Ryzen 5 3600XT 6-core, 12-thread processor. 
}
\end{figure}

\section{Out-of-sample MSE with entrywise input noise}

Here we perform an experiment to illustrate the robustness of our and other tensor methods to independent entrywise noise in the input tensors.
The experiment includes various levels of input noise; we set the standard deviation of each entry's noise equal to the value \texttt{input\_noise\_sd} times the input tensor's standard deviation.

Figure~\ref{fig:inputnoise} shows that TT\_CP and TT\_Tucker with pruning are robust to \texttt{input\_noise\_sd} up to $0.3$, regardless of the max\_depth. We also see that without pruning, TT\_CP and TT\_Tucker perform dramatically worse as depth increases, which indicates that the pruning procedure is correctly decreasing the depth of the tree to mitigate overfitting the input noise.
Once \texttt{input\_noise\_sd} equals 1, all methods (except for TT\_mean with pruning, which has a large test MSE even with no input noise) display a noticeable increase in test MSE, indicating that TT\_CP and TT\_Tucker with pruning are among the most robust methods tested in this experiment.

\begin{figure}[H]
\centering
\includegraphics[width=\textwidth]{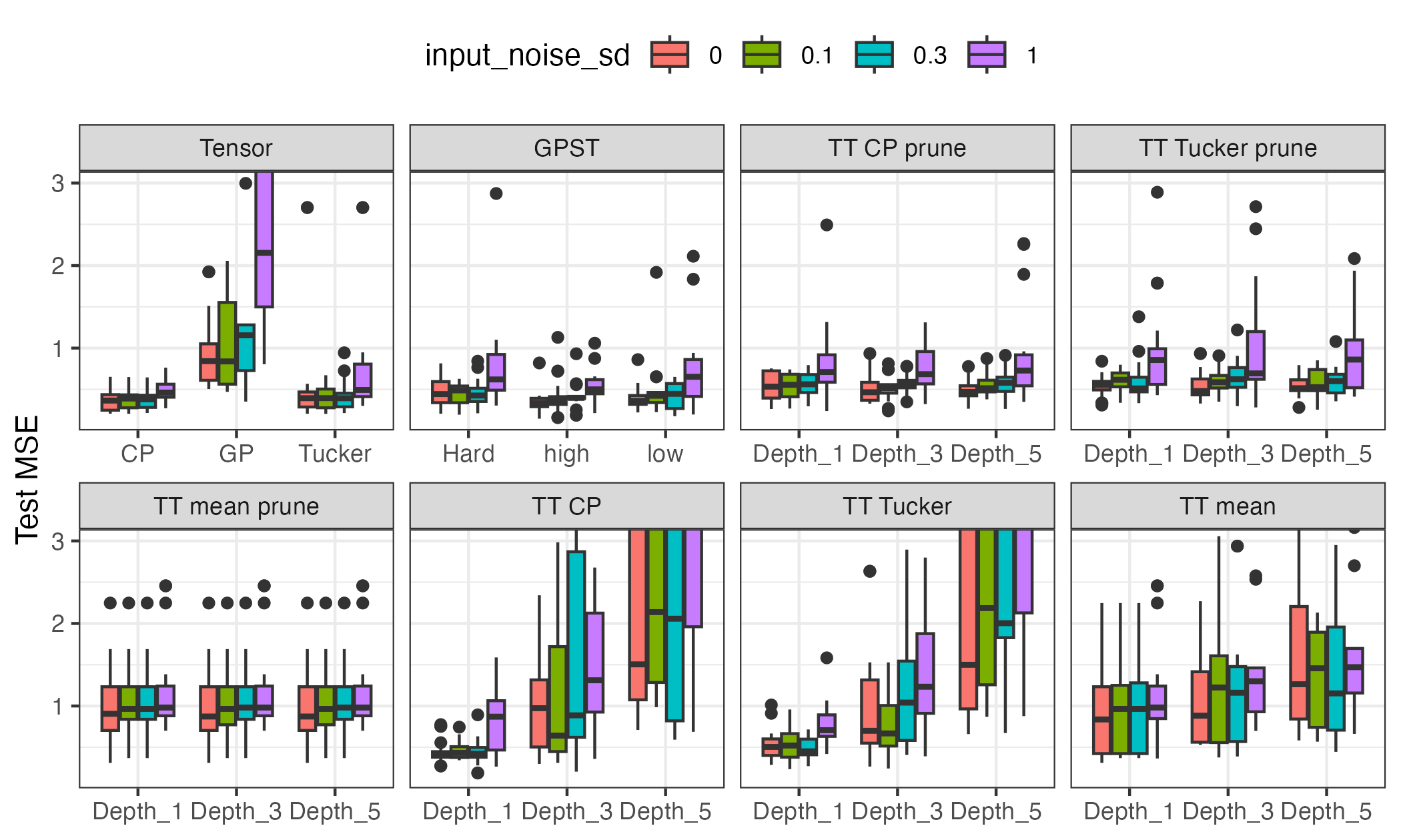}
\caption{\label{fig:inputnoise} Test MSE for the experiments
in Figure~\ref{fig:MSE-time-of-nonpara-models-GP} for input tensors of size $(N, d_1, d_2) = (100, 25, 25)$.
Zero-mean Gaussian noise is added entrywise to the input tensors so that each model is provided only noisy versions of the input tensors. The noise standard deviation is \texttt{input\_noise\_sd} multiplied by the input's standard deviation.
For TT models, we experiment with max\_depth=1,3,5 and select $\alpha=0.5$ in \eqref{eq:tl_complexity}
and perform 20 replicates. 
}
\end{figure}

\newpage
\section{MSE comparison to existing forest methods}
\label{sec:forestcomparison}

This section explores the performance of regression models --- Fréchet
Random Forest (FRF, bootstrap ratio = 0.75 and other parameters are default
values), global Fréchet regression (GRF), and Tensor Tree (TT) with gradient boosting (GB) --- on synthetic
datasets.
For code implementation, pyfrechet
(\url{https://github.com/matthieubulte/pyfrechet}) by \citet{bulte2024medoid} only supports 2-mode tensor
as a Fréchet space input and scalar outputs, since their focus is spherical
data. FrechForest (CRAN package, last maintained 2019, used by
\citet{capitaine2024frechet,qiu2024random}) supports mainly curves as
Fréchet space input and scalar output, but also supports 2- and 3-mode tensors. 
In contrast, our open-source
implementation supports up to 4-mode tensor inputs
and outputs, with improved efficiency, significantly extending the
usability of the existing tensor nonparametric regression software.

The first experiment explores different noise levels and
number of trees in the ensemble methods. 
The target variable is $y=2\cdot X[:,0]\cdot X[:,2]+3\cdot X[:,1]\cdot X[:,2]\cdot X[:,3]+\varepsilon$,
where $\varepsilon$ is Gaussian noise with a configurable noise variance.
The input $X$ is an $n=1000,d=5$ matrix with i.i.d.\ standard uniform entries, and we embed $X$ as a 3-tensor (adding an all-zero
third dimension to be compatible with TT implementation) tailored
to simulate structured data scenarios. 
For the parameter settings of the tested methods,
pyfrechet's FRF has no option for limiting each tree's maximum depth, so we set the minimum number of observations in each leaf node to two.
For FrechForest's FRF we use the default parameter settings other than the number of trees to grow.
The TT models use max\_depth=4, split\_rank=4 and full exploration rate
$\alpha=1$, and we also execute TTentrywise GB ensembles (with maximum depth 4 and at least two observations per leaf node) in order to compete with the two mentioned ensemble methods.
(If desired, Algorithm~\ref{alg:Random-Forest-with} shows how to train a random forest using TT.)
The models are trained and evaluated over 10 random seeds to ensure
robustness in MSE (on a held-out
test data) and time measurements shown in Table~\ref{tbl:forest}. 
Comparing single-tree models, 
TT has competitive performance and fast fits compared to either FRF implementation with one tree.
We expect that TT will outperform FRF for higher mode tensor inputs for a similarly simple input-response relationship, since higher mode tensors require more samples to have accurate estimates of Fréchet distances.
Comparing ensemble methods, 
TT GB with 5 or 50 trees has larger test MSEs than FRF with the same number of trees.
However, FRF seems to improve very little past 50 trees, whereas TT GB with 500 trees has a much smaller test MSE than FRF with any of the tested number of trees.
Thus, of all tested ensemble methods, TT GB is able to provide by far the smallest test MSE for this scenario.

{}
\begin{table}
{ }%
\centering
\begin{tabular}{rrlll}
\toprule 
{method } & \# trees & {noise variance 0 } & {noise variance 0.01 } & {noise variance 0.1}\tabularnewline
\midrule
\midrule 
{TT} & 1 & {0.108 (1s) } & {0.089 (1s) } & {0.138 (1s)}\tabularnewline
\midrule 
{TTentrywise GB} & $5$  & {1.136 (12s) } & {1.135 (11s)} & {1.138 (9s)}\tabularnewline
\midrule 
{TTentrywise GB} & $50$ & {0.510 (33s)} & {0.521 (41s)} & {0.516 (42s)}\tabularnewline
\midrule 
{TTentrywise GB} & $250$ & {0.039 (196s) } & {0.038 (221s)} & {0.038 (204s)}\tabularnewline
\midrule 
{TTentrywise GB} & $500$ & {0.015 (588s)} & {0.014 (574s)} & {0.019 (557s)}\tabularnewline
\midrule 
{GFR} & n/a & {0.258 (1s) } & {0.194 (1s) } & {0.199 (2s)}\tabularnewline
\midrule
\midrule 
{pyfrechet FRF} & $1$ & {0.098 (4s) } & {0.077 (4s) } & {0.114 (5s)}\tabularnewline
\midrule 
{pyfrechet FRF} & $50$ & {0.055 (24s) } & {0.045 (19s) } & {0.067 (15s)}\tabularnewline
\midrule 
{pyfrechet FRF} & $250$ & {0.052 (111s) } & {0.043 (109s) } & {0.066 (114s)}\tabularnewline
\midrule 
{pyfrechet FRF} & $500$ & {0.047 (246s) } & {0.049 (245s) } & {0.068 (241s)}\tabularnewline
\midrule
\midrule 
{FrechForest FRF} & $1$ & {0.145 (2s) } & {0.150 (2s) } & {0.176 (2s)}\tabularnewline
\midrule 
{FrechForest FRF} & $5$ & {0.054 (4s) } & {0.055 (4s) } & {0.071 (4s)}\tabularnewline
\midrule 
{FrechForest FRF} & $50$ & {0.033 (23s) } & {0.033 (23s) } & {0.044 (23s)}\tabularnewline
\midrule 
{FrechForest FRF} & $250$ & {0.031 (108s) } & {0.031 (108s) } & {0.042 (108s)}\tabularnewline
\midrule 
{FrechForest FRF} & $500$ & {0.031 (210s) } & {0.030 (209s) } & {0.042 (210s)}\tabularnewline
\bottomrule
\end{tabular}
\caption{Out-of-sample MSEs and fit-predict time (seconds) for each forest method with different number of trees in the ensemble method.\label{tbl:forest}}
\end{table}

The second experiment compares TT with one tree against (FrechForest's) FRF for higher mode tensor inputs and a more complex input-response relationship.
To be compatible with FrechForest, we reduce the three channels to one in the simulation code in the experiments in Figure~\ref{fig:MSE-time-of-nonpara-models-GP} of the main text in order to reduce the number of modes of the input tensors from 4 to 3. This 3-tensor input has dimensions $n=750$ and $d_1=d_2=25$. The test MSE is computed on a held-out set with $n=250$. Figure~\ref{fig:FRF} shows that with pruning, TT\_CP and TT\_Tucker with max\_depth 1 have a similar test MSE to FRF with at least five trees but fit much faster (FRF with 250 trees took around 11 minutes to fit). These two TT methods also have about half of the test MSE \emph{and} half of the training time than FRF with one tree.
In this scenario, TT with one tree is faster to fit and has competitive performance to FRF with any number of trees.

\begin{figure}[H]
\centering
\includegraphics[width=.8\textwidth]{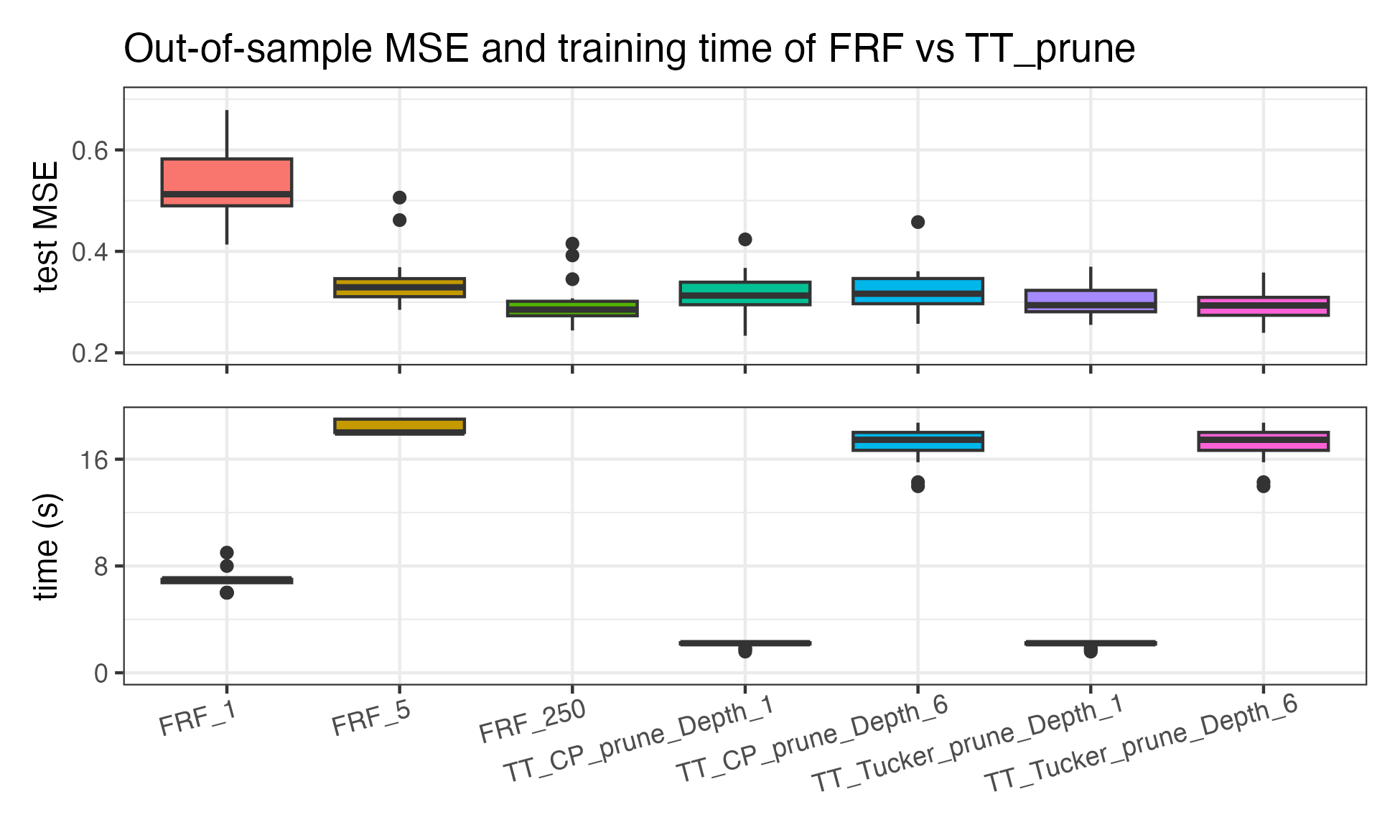}
\caption{\label{fig:FRF} Test MSE for modified experiments
in Figure~\ref{fig:MSE-time-of-nonpara-models-GP} for input tensors of size $(N, d_1, d_2) = (1000, 25, 25)$. 
The number after the ``Depth\_'' label indicates the max\_depth of the TT method.
The number after the ``FRF\_'' label indicates the number of trees.
All models are run with 20 replicates. 
Each FRF with 250 trees took about 11 minutes (660 seconds) to fit.
}
\end{figure}

\section{RPE results comparing the \texttt{rrr} method to TT}
\label{sec:RPErrr}

\begin{table}
{\scriptsize{}\centering }{\scriptsize\par}

{\scriptsize{}}\subfloat[{{Linear. $(d_{1},d_{2})=(3,4)$. $p_{1}=15$. $\bm{y}[:,i]=\protect\begin{cases}
\bm{X}[:,0,1]+\bm{X}[:,1,1] & \text{if }i\mod3=0\protect\\
\bm{X}[:,1,1]+\bm{X}[:,2,0] & \text{if }i\mod3=1\protect\\
\bm{X}[:,2,2]+\bm{X}[:,0,3] & \text{if }i\mod3=2
\protect\end{cases}$.}}]{{\scriptsize{}{}{}{}\centering }%
\begin{tabular}{ccccccc}
\toprule 
{\scriptsize{}{}{}{}Rank } & {\scriptsize{}{}{}{}rrr } & {\scriptsize{}{}{}{}rrrBayes } & {\scriptsize{}{}{}{}TTentrywise\_CP } & {\scriptsize{}{}{}{}TTentrywise\_Tucker } & {\scriptsize{}{}{}{}TTlowrank\_CP } & {\scriptsize{}{}{}{}TTlowrank\_Tucker}\tabularnewline
\midrule 
{\scriptsize{}{}{}{}2 } & \textbf{\scriptsize{}143 \textpm{} 9}{\scriptsize{} } & \textbf{\scriptsize{}136 \textpm{} 9}{\scriptsize{} } & {\scriptsize{}218 \textpm{} 3 } & {\scriptsize{}223 \textpm{} 4 } & {\scriptsize{}172 \textpm{} 2 } & {\scriptsize{}3581 \textpm{} 79}\tabularnewline
{\scriptsize{}{}{}{}3 } & {\scriptsize{}172 \textpm{} 9 } & {\scriptsize{}162 \textpm{} 7 } & {\scriptsize{}219 \textpm{} 3 } & {\scriptsize{}219 \textpm{} 3 } & \textbf{\scriptsize{}132 \textpm{} 2}{\scriptsize{} } & {\scriptsize{}3633 \textpm{} 84}\tabularnewline
{\scriptsize{}{}{}{}4 } & {\scriptsize{}177 \textpm{} 10 } & {\scriptsize{}171 \textpm{} 9 } & {\scriptsize{}219 \textpm{} 3 } & {\scriptsize{}219 \textpm{} 3 } & \textbf{\scriptsize{}134 \textpm{} 2}{\scriptsize{} } & {\scriptsize{}3678 \textpm{} 90}\tabularnewline
{\scriptsize{}{}{}{}5 } & {\scriptsize{}178 \textpm{} 5 } & {\scriptsize{}170 \textpm{} 5 } & {\scriptsize{}216 \textpm{} 3 } & {\scriptsize{}216 \textpm{} 3 } & \textbf{\scriptsize{}133 \textpm{} 2}{\scriptsize{} } & {\scriptsize{}3628 \textpm{} 92}\tabularnewline
{\scriptsize{}{}{}{}6 } & {\scriptsize{}174 \textpm{} 7 } & {\scriptsize{}207 \textpm{} 24 } & {\scriptsize{}219 \textpm{} 4 } & {\scriptsize{}219 \textpm{} 4 } & \textbf{\scriptsize{}137 \textpm{} 2}{\scriptsize{} } & {\scriptsize{}3715 \textpm{} 106}\tabularnewline
{\scriptsize{}{}{}{}7 } & {\scriptsize{}260 \textpm{} 17 } & {\scriptsize{}319 \textpm{} 24 } & {\scriptsize{}219 \textpm{} 3 } & {\scriptsize{}219 \textpm{} 3 } & \textbf{\scriptsize{}138 \textpm{} 2}{\scriptsize{} } & {\scriptsize{}3659 \textpm{} 85}\tabularnewline
\bottomrule
\end{tabular}{\scriptsize{}{}{}{}}{\scriptsize\par}

{\scriptsize{} }{\scriptsize\par}

{\scriptsize{}}{\scriptsize\par}}{\scriptsize\par}

{\scriptsize{}}\subfloat[{{Non-linear. $(d_{1},d_{2})=(3,4)$. $p_{1}=6$. $\bm{y}[:,i]=\sin(\bm{X}[:,i\mod3,i\mod4])$.}}]{{\scriptsize{}{}{}{}\centering }%
\begin{tabular}{ccccccc}
\toprule 
{\scriptsize{}{}{}{}Rank } & {\scriptsize{}{}{}{}rrr } & {\scriptsize{}{}{}{}rrrBayes } & {\scriptsize{}{}{}{}TTentrywise\_CP } & {\scriptsize{}{}{}{}TTentrywise\_Tucker } & {\scriptsize{}{}{}{}TTlowrank\_CP } & {\scriptsize{}{}{}{}TTlowrank\_Tucker}\tabularnewline
\midrule 
{\scriptsize{}{}2 } & {\scriptsize{}219 \textpm{} 5 } & \textbf{\scriptsize{}209 \textpm{} 3}{\scriptsize{} } & \textbf{\scriptsize{}205 \textpm{} 2}{\scriptsize{} } & \textbf{\scriptsize{}206 \textpm{} 3}{\scriptsize{} } & {\scriptsize{}297 \textpm{} 2 } & {\scriptsize{}1249 \textpm{} 18}\tabularnewline
{\scriptsize{}{}3 } & {\scriptsize{}231 \textpm{} 8 } & \textbf{\scriptsize{}211 \textpm{} 4}{\scriptsize{} } & \textbf{\scriptsize{}205 \textpm{} 2}{\scriptsize{} } & \textbf{\scriptsize{}206 \textpm{} 3}{\scriptsize{} } & {\scriptsize{}260 \textpm{} 2 } & {\scriptsize{}1265 \textpm{} 25}\tabularnewline
{\scriptsize{}{}4 } & {\scriptsize{}240 \textpm{} 9 } & {\scriptsize{}218 \textpm{} 6 } & \textbf{\scriptsize{}205 \textpm{} 2}{\scriptsize{} } & \textbf{\scriptsize{}206 \textpm{} 3}{\scriptsize{} } & {\scriptsize{}231 \textpm{} 3 } & {\scriptsize{}1290 \textpm{} 27}\tabularnewline
{\scriptsize{}{}5 } & {\scriptsize{}241 \textpm{} 11 } & {\scriptsize{}222 \textpm{} 8 } & {\scriptsize{}205 \textpm{} 2 } & {\scriptsize{}206 \textpm{} 3 } & \textbf{\scriptsize{}200 \textpm{} 2}{\scriptsize{} } & {\scriptsize{}1279 \textpm{} 12}\tabularnewline
{\scriptsize{}{}6 } & {\scriptsize{}257 \textpm{} 8 } & {\scriptsize{}240 \textpm{} 8 } & {\scriptsize{}203 \textpm{} 2 } & {\scriptsize{}204 \textpm{} 2 } & \textbf{\scriptsize{}169 \textpm{} 2}{\scriptsize{} } & {\scriptsize{}1328 \textpm{} 23}\tabularnewline
{\scriptsize{}{}7 } & {\scriptsize{}288 \textpm{} 31 } & {\scriptsize{}276 \textpm{} 32 } & \textbf{\scriptsize{}205 \textpm{} 2}{\scriptsize{} } & \textbf{\scriptsize{}206 \textpm{} 3}{\scriptsize{} } & {\scriptsize{}261 \textpm{} 34 } & {\scriptsize{}1332 \textpm{} 38}\tabularnewline
\bottomrule
\end{tabular}{\scriptsize{}{}{}{}}{\scriptsize\par}

{\scriptsize{}}{\scriptsize\par}}{\scriptsize\par}

{\scriptsize{}}\subfloat[{{\small{}Exact CP (4). $(d_{1},d_{2})=(12,6)$. $p_{1}=7$. $\bm{y}[:,i]=\tilde{\bm{X}}[:,0,1]^{2}-\bm{X}[:,0,0]$({{\small{}{}$\tilde{\bm{X}}$
is rank-4 CP reconstructed $\bm{X}$}}{}{}).}}]{{\scriptsize{}{}{}{}\centering }%
\begin{tabular}{ccccccc}
\toprule 
{\scriptsize{}{}{}{}Rank } & {\scriptsize{}{}{}{}rrr } & {\scriptsize{}{}{}{}rrrBayes } & {\scriptsize{}{}{}{}TTentrywise\_CP } & {\scriptsize{}{}{}{}TTentrywise\_Tucker } & {\scriptsize{}{}{}{}TTlowrank\_CP } & {\scriptsize{}{}{}{}TTlowrank\_Tucker}\tabularnewline
\midrule 
{\scriptsize{}{}{}{}2 } & {\scriptsize{}342 \textpm{} 69 } & {\scriptsize{}326 \textpm{} 65 } & {\scriptsize{}455 \textpm{} 36 } & {\scriptsize{}464 \textpm{} 35 } & {\scriptsize{}392 \textpm{} 43 } & \textbf{\scriptsize{}204 \textpm{} 37}\tabularnewline
{\scriptsize{}{}{}{}3 } & {\scriptsize{}433 \textpm{} 91 } & {\scriptsize{}404 \textpm{} 86 } & {\scriptsize{}456 \textpm{} 36 } & {\scriptsize{}465 \textpm{} 34 } & {\scriptsize{}395 \textpm{} 42 } & \textbf{\scriptsize{}204 \textpm{} 37}\tabularnewline
{\scriptsize{}{}{}{}4 } & {\scriptsize{}531 \textpm{} 106 } & {\scriptsize{}490 \textpm{} 97 } & {\scriptsize{}456 \textpm{} 36 } & {\scriptsize{}466 \textpm{} 34 } & {\scriptsize{}398 \textpm{} 42 } & \textbf{\scriptsize{}205 \textpm{} 38}\tabularnewline
{\scriptsize{}{}{}{}5 } & {\scriptsize{}609 \textpm{} 115 } & {\scriptsize{}553 \textpm{} 107 } & {\scriptsize{}456 \textpm{} 36 } & {\scriptsize{}468 \textpm{} 34 } & {\scriptsize{}400 \textpm{} 42 } & \textbf{\scriptsize{}204 \textpm{} 37}\tabularnewline
{\scriptsize{}{}{}{}6 } & {\scriptsize{}666 \textpm{} 127 } & {\scriptsize{}596 \textpm{} 117 } & {\scriptsize{}456 \textpm{} 36 } & {\scriptsize{}468 \textpm{} 34 } & {\scriptsize{}403 \textpm{} 42 } & \textbf{\scriptsize{}206 \textpm{} 38}\tabularnewline
{\scriptsize{}{}{}{}7 } & {\scriptsize{}666 \textpm{} 127 } & {\scriptsize{}596 \textpm{} 117 } & {\scriptsize{}456 \textpm{} 36 } & {\scriptsize{}468 \textpm{} 34 } & {\scriptsize{}406 \textpm{} 42 } & \textbf{\scriptsize{}205 \textpm{} 38}\tabularnewline
\bottomrule
\end{tabular}{\scriptsize{}{}{}{}}{\scriptsize\par}

{\scriptsize{}}{\scriptsize\par}}{\scriptsize\par}

{\scriptsize{}}\subfloat[{{\small{}Exact Tucker (4, 4, 4). $(d_{1},d_{2})=(12,6)$. $p_{1}=7$.
$\bm{y}[:,i]=\tilde{\bm{X}}[:,0,1]^{2}-\bm{X}[:,0,0]$ ({{\small{}{}$\tilde{\bm{X}}$
is rank-(4,4,4) or rank-(4,4) Tucker reconstructed $\bm{X}$}}{}{}).}}]{{\scriptsize{}{}{}{}\centering }%
\begin{tabular}{ccccccc}
\toprule 
{\scriptsize{}{}{}{}Rank } & {\scriptsize{}{}{}{}rrr } & {\scriptsize{}{}{}{}rrrBayes } & {\scriptsize{}{}{}{}TTentrywise\_CP } & {\scriptsize{}{}{}{}TTentrywise\_Tucker } & {\scriptsize{}{}{}{}TTlowrank\_CP } & {\scriptsize{}{}{}{}TTlowrank\_Tucker}\tabularnewline
\midrule 
{\scriptsize{}{}{}{}2 } & {\scriptsize{}376 \textpm{} 114 } & {\scriptsize{}335 \textpm{} 98 } & {\scriptsize{}452 \textpm{} 42 } & {\scriptsize{}463 \textpm{} 41 } & {\scriptsize{}395 \textpm{} 47 } & \textbf{\scriptsize{}208 \textpm{} 57}\tabularnewline
{\scriptsize{}{}{}{}3 } & {\scriptsize{}480 \textpm{} 134 } & {\scriptsize{}427 \textpm{} 122 } & {\scriptsize{}451 \textpm{} 41 } & {\scriptsize{}464 \textpm{} 42 } & {\scriptsize{}395 \textpm{} 46 } & \textbf{\scriptsize{}208 \textpm{} 57}\tabularnewline
{\scriptsize{}{}{}{}4 } & {\scriptsize{}537 \textpm{} 153 } & {\scriptsize{}491 \textpm{} 141 } & {\scriptsize{}451 \textpm{} 41 } & {\scriptsize{}464 \textpm{} 42 } & {\scriptsize{}397 \textpm{} 46 } & \textbf{\scriptsize{}209 \textpm{} 57}\tabularnewline
{\scriptsize{}{}{}{}5 } & {\scriptsize{}609 \textpm{} 177 } & {\scriptsize{}545 \textpm{} 153 } & {\scriptsize{}451 \textpm{} 41 } & {\scriptsize{}464 \textpm{} 42 } & {\scriptsize{}400 \textpm{} 46 } & \textbf{\scriptsize{}206 \textpm{} 56}\tabularnewline
{\scriptsize{}{}{}{}6 } & {\scriptsize{}668 \textpm{} 189 } & {\scriptsize{}591 \textpm{} 166 } & {\scriptsize{}451 \textpm{} 41 } & {\scriptsize{}464 \textpm{} 42 } & {\scriptsize{}403 \textpm{} 46 } & \textbf{\scriptsize{}206 \textpm{} 56}\tabularnewline
{\scriptsize{}{}{}{}7 } & {\scriptsize{}668 \textpm{} 189 } & {\scriptsize{}591 \textpm{} 166 } & {\scriptsize{}451 \textpm{} 41 } & {\scriptsize{}464 \textpm{} 42 } & {\scriptsize{}406 \textpm{} 46 } & \textbf{\scriptsize{}209 \textpm{} 57}\tabularnewline
\bottomrule
\end{tabular}{\scriptsize{}{}{}{}}{\scriptsize\par}

{\scriptsize{}}{\scriptsize\par}}\caption{\label{fig:Tensor-output-regression-compari} Out-of-sample RPE (multiplied
by 1000 to highlight decimal differences) of various tensor-output
regression models trained on synthetic data. 
Bold font indicates the smallest RPE of a given rank and signal; we
use $\pm$ to represent the standard error over 10 repeats.}
\end{table}

\section{Regression coefficients on an EEG dataset}
\label{sec:EEGcoeff}

\begin{figure}[H]
\centering
\includegraphics[width=0.22\textwidth]{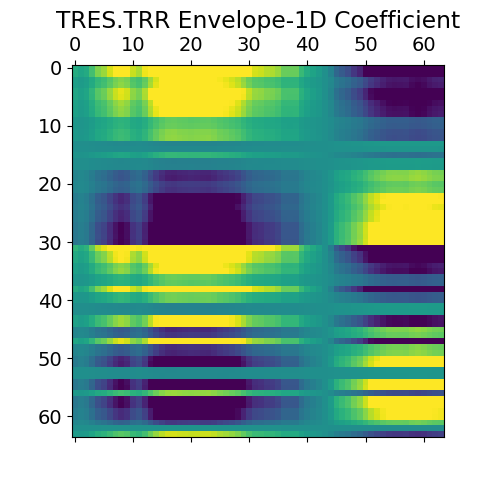}
\includegraphics[width=0.22\textwidth]{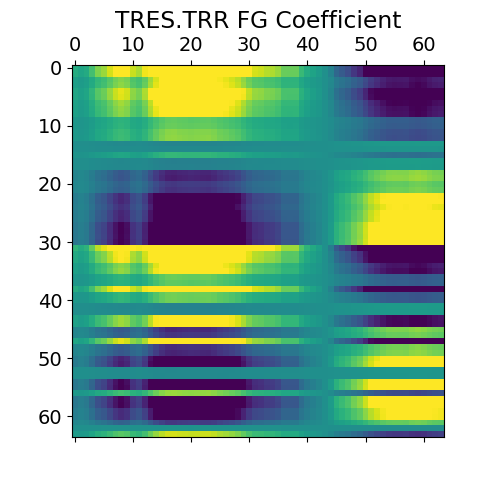}
\includegraphics[width=0.22\textwidth]{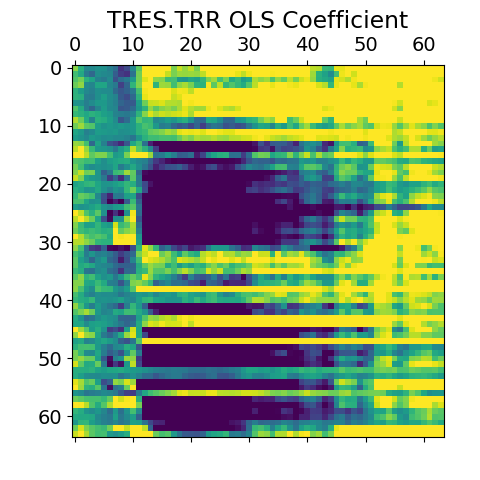}
\includegraphics[width=0.22\textwidth]{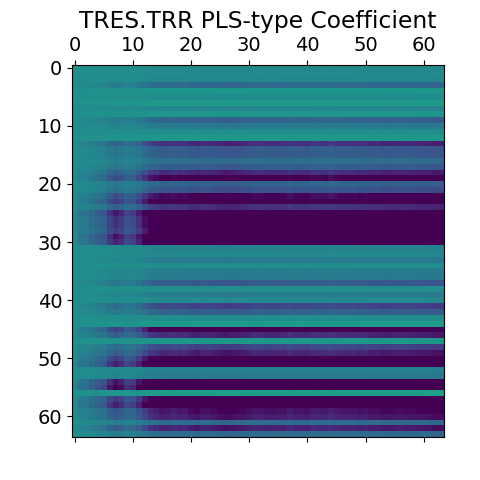}

\includegraphics[width=0.22\textwidth]{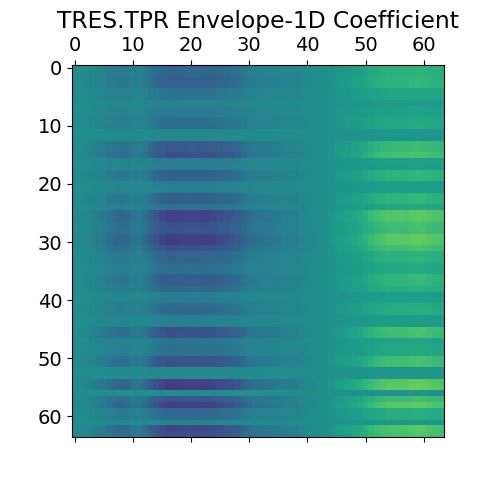}
\includegraphics[width=0.22\textwidth]{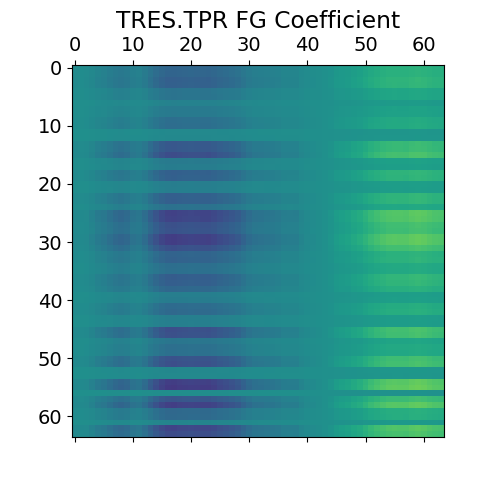}
\includegraphics[width=0.22\textwidth]{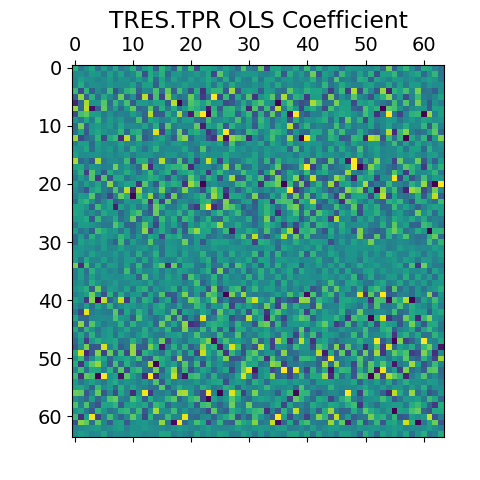}
\includegraphics[width=0.22\textwidth]{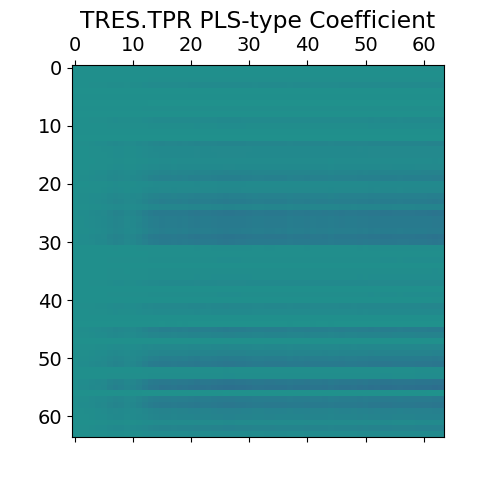}

\includegraphics[width=0.20\textwidth]{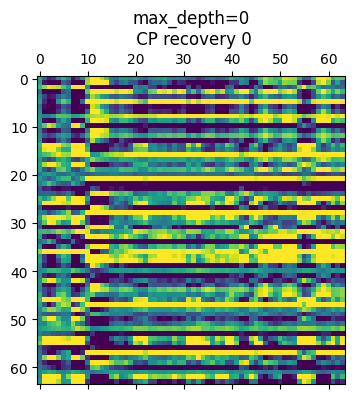}
\hspace{.2cm}
\includegraphics[width=0.20\textwidth]{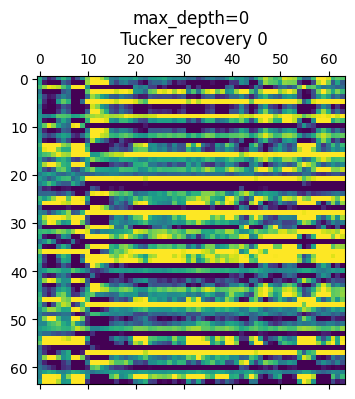}

\includegraphics[width=0.95\textwidth]{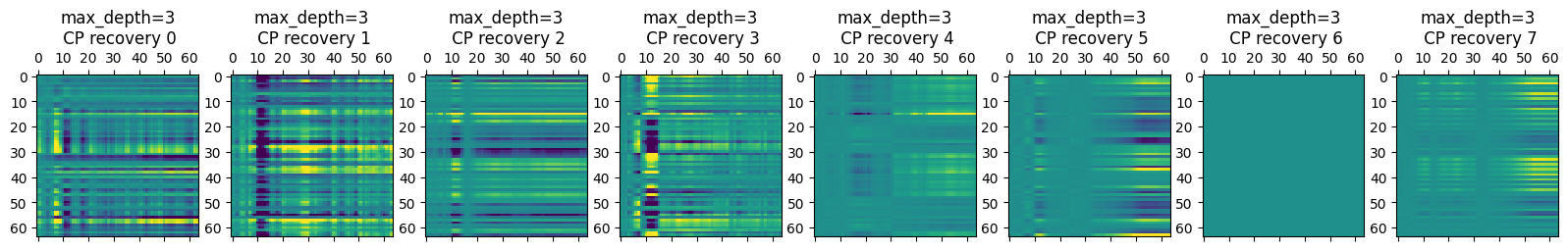}\\
\includegraphics[width=0.95\textwidth]{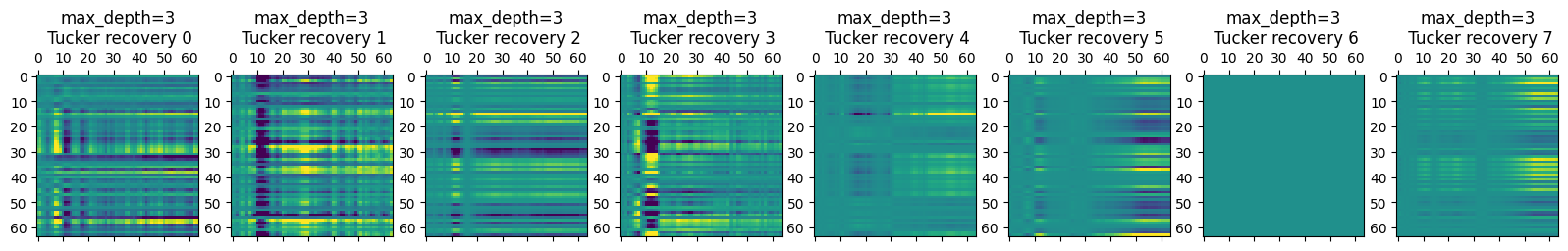}

\caption{\label{fig:Comparison-EEG}Tensor-input decision tree regression
coefficients on an EEG dataset in \citet{zeng2021tres}. 
First row: {TRES.TRR(1D/FG/OLS/PLS) tensor-response models (following the example in $\mathtt{TRES}$) 
Second row: TRES.TPR(1D/FG/OLS/PLS) tensor-input models.}
Third row: coefficients from CP ($R=5$), Tucker ($R=5$),
Fourth row: coefficients
from a tensor-input tree model with depth 3 and leaf node
CP models $m_{j}$ ($R=5$). Fifth row: coefficients from a 
tensor-input tree model with depth 3 and leaf node Tucker models
$m_{j}$ ($R=5$).}
\end{figure}

\bibliographystyleSM{chicago}
\bibliographySM{SHAPEFUN}

\end{document}